\documentclass[final,5p,times,twocolumn,authoryear,fleqn]{elsarticle}

\usepackage{hyperref}
\hypersetup{
    colorlinks=true,
    linkcolor=black,
    citecolor=black,
    urlcolor=black
}
\usepackage{etoolbox}
\usepackage{xurl}
\usepackage{amssymb}
\usepackage{amsthm}
\usepackage{amsmath}
\usepackage{multirow}
\usepackage{textcomp}
\usepackage{subcaption}
\usepackage{verbatim}
\usepackage{siunitx}
\usepackage{epstopdf}
\usepackage[table,xcdraw]{xcolor}
\usepackage{natbib}

\providecommand{\doi}[1]{\href{https://doi.org/#1}{https://doi.org/#1}}
\def\doi#1{\href{https://doi.org/#1}{https://doi.org/#1}}
\AtBeginEnvironment{thebibliography}{\raggedright\sloppy\setlength{\emergencystretch}{3em}}
\usepackage{algorithm}
\usepackage{algorithmic}

\usepackage{enumitem}

\usepackage{booktabs}

\newtheorem{theorem}{Theorem}
\newtheorem{lemma}{Lemma}

\newtheorem{definition}{Definition}

\DeclareMathOperator*{\argmin}{arg\,min}

\journal{Expert Systems with Applications}

\begin{document}

\begin{frontmatter}

\title{Robust and sparse support vector machine via hybrid truncated loss for supervised classification} 

%

\author[label1]{Yuliang Yang}
\ead{Yangyuliang@bjfu.edu.cn}
\author[label2,label1]{Chen Chen}
\ead{ccyolo2022@bjfu.edu.cn}
\author[label1]{Yuxiang Liu}
\ead{liuyuxiang@bjfu.edu.cn}
\author[label1]{Huiru Wang\corref{cor1}}
\ead{whr2019@bjfu.edu.cn}

\cortext[cor1]{Corresponding author.}
\address[label1]{School of Science, Beijing Forestry University, No.35 Qinghua East Road, 100083 Haidian, Beijing, China}

\address[label2]{Translational Cancer Research Center, Peking University First Hospital, No. 8 Xishiku Street, 100034 Xicheng, Beijing, China}


\begin{abstract}
The support vector machine (SVM) is a widely used classifier, but choosing an appropriate loss function remains difficult. Convex losses such as the hinge loss and least-squares loss are sensitive to outliers, while bounded non-convex losses often lead to high computational cost. To address this, we propose a hybrid truncated 
loss function ($L_{\mathrm{ht}}$) that is both sparse and bounded, and build the $L_{\mathrm{ht}}$-SVM model for single-view classification. We introduce the P-stationary point and use it to establish the first-order necessary and sufficient optimality conditions. Based on these conditions, we design an alternating direction method of multipliers with a working-set strategy that reduces computational cost and achieves global convergence. 
We further extend $L_{\mathrm{ht}}$-SVM to multi-view learning by adding structural information and view weights, resulting in Mv$L_{\mathrm{ht}}$-SVM, which follows both the consensus and complementarity principles. Experiments on synthetic, real-world, and image datasets show that $L_{\mathrm{ht}}$-SVM achieves higher 
accuracy with fewer support vectors and better noise robustness than five single-view methods, while Mv$L_{\mathrm{ht}}$-SVM outperforms six multi-view methods in accuracy, precision, recall, and F1-score.
\end{abstract}



\begin{keyword}
Hybrid truncated loss function \sep
P-stationary point \sep
Alternating direction method of multipliers \sep
Multi-view learning \sep
Working set strategy
\end{keyword}

\end{frontmatter}


\section{Introduction}
\label{sec:introduction}

Support vector machines (SVMs) are widely used in machine learning 
\citep{maggioni2025novel,cortes1995support}, computer vision 
\citep{senokosov2024quantum,chen2025ams}, natural language 
processing \citep{wu2023csr,alkhodhairy2025machine}, and 
biomedical analysis \citep{taha2025machine}. Given a binary 
classification dataset $\{(\mathbf{x}_i, y_i)\}_{i=1}^m$, where 
$\mathbf{x}_i \in \mathbb{R}^n$ and $y_i \in \{-1,+1\}$, the 
soft-margin SVM solves
\begin{equation}\label{eq:svm}
    \min_{\mathbf{w}, b} \; \frac{1}{2}\|\mathbf{w}\|^2 
    + C \sum_{i=1}^{m} \ell\bigl(1 - y_i(\mathbf{w}^\top \mathbf{x}_i + b)\bigr),
\end{equation}
where $C > 0$ is the penalty parameter and $\ell(\cdot)$ is the 
loss function. The choice of loss strongly affects the sparsity, 
robustness, and computational efficiency of the resulting 
classifier~\citep{akhtar2024roboss,wang2025sparse}.

Among convex losses, the hinge loss (see 
Fig.~\ref{fig:loss_functions}(a)) \citep{cortes1995support} can 
produce sparse solutions, but its penalty is unbounded for large 
margin violations, which makes it sensitive to outliers 
\citep{brahmi2024efficient,wang2023fast}. The pinball loss (see 
Fig.~\ref{fig:loss_functions}(b)) \citep{huang2013support} 
extends the hinge loss to quantile settings, but it is still 
unbounded \citep{huang2016solution}. The least-squares loss (see 
Fig.~\ref{fig:loss_functions}(c)) \citep{suykens1999least} gives 
a smooth quadratic objective, but it penalises all samples, even 
those that are already classified correctly, and therefore loses 
sparsity \citep{liu2022l2,wang2005comparison}. In short, these 
convex losses do not provide both sparsity and bounded penalty at 
the same time.

Bounded non-convex losses provide an alternative. The ramp loss 
(see Fig.~\ref{fig:loss_functions}(d)) \citep{collobert2006trading} 
clips the hinge loss at a finite threshold and thus improves both 
sparsity and robustness to outliers. The $L_{0/1}$ loss (see 
Fig.~\ref{fig:loss_functions}(e)) \citep{wang2021support} assigns 
a unit penalty to each misclassified sample and can be viewed as 
a natural counting-type loss. Other bounded losses, such as the 
sigmoid loss \citep{mason1999boosting}, rescaled hinge loss 
\citep{xu2017robust}, truncated pinball loss 
\citep{shen2017support}, and capped squared loss 
\citep{wang2024sparse}, have also been proposed to balance 
robustness and smoothness. The main difficulty is that 
non-convex losses are harder to optimise and may have many local 
minimisers \citep{akhtar2024roboss}. For example, optimisation 
under the $L_{0/1}$ loss is NP-hard in general.

A practical way to handle bounded non-convex losses was 
introduced in the $L_{0/1}$-SVM of \cite{wang2021support}, where 
an alternating direction method of multipliers (ADMM) was combined with a working-set strategy and the 
P-stationary point (PSP) was used to characterise the non-convex 
subproblem. This ADMM-plus-working-set framework was later 
extended to several other bounded-loss models, including the ramp 
fraction loss \citep{wang2024fast}, the capped squared loss 
\citep{wang2024sparse}, a sparse and robust variant 
\citep{wang2025sparse}, and the multi-view $SL_{0/1}$-SVM 
\citep{chen2025multi}.

In many real applications, data come from multiple sources or can 
be described by different feature extractors. For example, an 
image can be represented by colour histograms, texture 
descriptors, and deep features, while a web page can be 
described by both text content and hyperlink topology. Multi-view 
learning (MVL) uses these complementary representations jointly 
and often performs better than single-view methods in both theory 
and practice 
\citep{berahmand2025comprehensive,farquhar2005two,jia2025deep}.

Most MVL methods are built on two principles: \emph{consistency} 
and \emph{complementarity}. Consistency encourages agreement 
across views, whereas complementarity makes use of the 
information that is unique to each view 
\citep{berahmand2025comprehensive}. Effective multi-view 
classification models should satisfy both principles, which are 
often called the 2C principles.

Many studies have combined SVM with MVL. SVM-2K 
\citep{farquhar2005two} enforces consistency between two views 
through a KCCA-based constraint, but it does not explicitly use 
complementary information. MvTSVM \citep{xie2015multi} speeds up 
training by solving two smaller quadratic programmes, but it 
also ignores complementarity. MvSVM-2C \citep{xie2019multi} 
addresses this issue by introducing adaptive view weights so 
that both 2C principles are considered. Later models extended 
this line of work. For example, MPWTSVM \citep{xu2022multi} 
incorporates privileged information, MvLSSVC-2C 
\citep{zhang2025multi} uses hierarchical agglomerative 
clustering to enhance inter-view interaction, the deep 
multi-view LSSVM in \citep{jia2025deep} adds stacking layers, 
and Mv-vSVM \citep{he2025new} reduces the burden of parameter 
selection in SVM-2K. 

However, most existing multi-view SVM models still rely on 
convex losses and may therefore degrade under outlier 
contamination. Some recent methods have introduced bounded 
non-convex losses for better robustness, such as the LINEX loss 
in MVASY-BX \citep{tang2023robust}, the wave loss in Wave-MvSVM 
\citep{quadir2025enhancing}, and the $L_{0/1}$ loss in 
Mv$SL_{0/1}$-SVM \citep{chen2025multi}. To the best of our 
knowledge, the hybrid truncated loss has not yet been studied in 
a multi-view SVM framework.

Motivated by these observations, we propose a hybrid truncated 
($L_{\mathrm{ht}}$) loss (see Fig.~\ref{fig:loss_functions}(f)) 
that combines boundedness and sparsity in a single form. Based 
on this loss, we develop $L_{\mathrm{ht}}$-SVM for single-view 
classification together with an ADMM solver accelerated by a 
working-set strategy. We then extend the model to the multi-view 
setting and obtain Mv$L_{\mathrm{ht}}$-SVM, which satisfies both 
consistency and complementarity while maintaining robustness to 
outliers. To the best of our knowledge, this is the first study 
that introduces a hybrid truncated loss into a multi-view SVM 
model. Extensive experiments show that the proposed methods 
achieve competitive accuracy together with strong robustness.

This paper makes the following contributions.
\begin{itemize}
\item We introduce a hybrid truncated ($L_{\mathrm{ht}}$) loss 
that is bounded above and able to produce sparse solutions. Based 
on this loss, we formulate $L_{\mathrm{ht}}$-SVM for single-view 
classification.
\item We derive first-order necessary and sufficient optimality 
conditions under the P-stationary point framework. This provides 
a theoretical basis for the non-convex and non-smooth 
$L_{\mathrm{ht}}$-SVM model.
\item We develop an ADMM algorithm with a working-set strategy, 
which reduces the per-iteration cost from $O(m)$ to $O(|F_k|)$ 
and leads to substantial speedups on large-scale problems.
\item We extend the single-view model to a multi-view classifier, 
Mv$L_{\mathrm{ht}}$-SVM, by incorporating structural covariance 
information and adaptive view weighting to satisfy both 
consistency and complementarity.
\item Extensive experiments on synthetic, UCI, LIBSVM, and STL-10 
datasets show that $L_{\mathrm{ht}}$-SVM achieves competitive 
accuracy with fewer support vectors and better noise robustness, 
while Mv$L_{\mathrm{ht}}$-SVM outperforms six state-of-the-art 
multi-view classifiers in terms of accuracy, precision, recall, 
and F1-score.
\end{itemize}

The rest of this paper is organised as follows. 
Section~\ref{sec:related} reviews related loss functions and 
multi-view SVM models. Section~\ref{sec:optimality} introduces 
the $L_{\mathrm{ht}}$ loss and presents the optimality theory for 
$L_{\mathrm{ht}}$-SVM. Section~\ref{sec:algorithm} describes the 
ADMM solver with working-set acceleration. 
Section~\ref{sec:multiview} presents the multi-view extension. 
Section~\ref{sec:experiment} reports the experimental results. 
Section~\ref{sec:conclusion} concludes the paper.

\begin{figure*}[!t]
    \centering
    \includegraphics[width=0.85\textwidth]{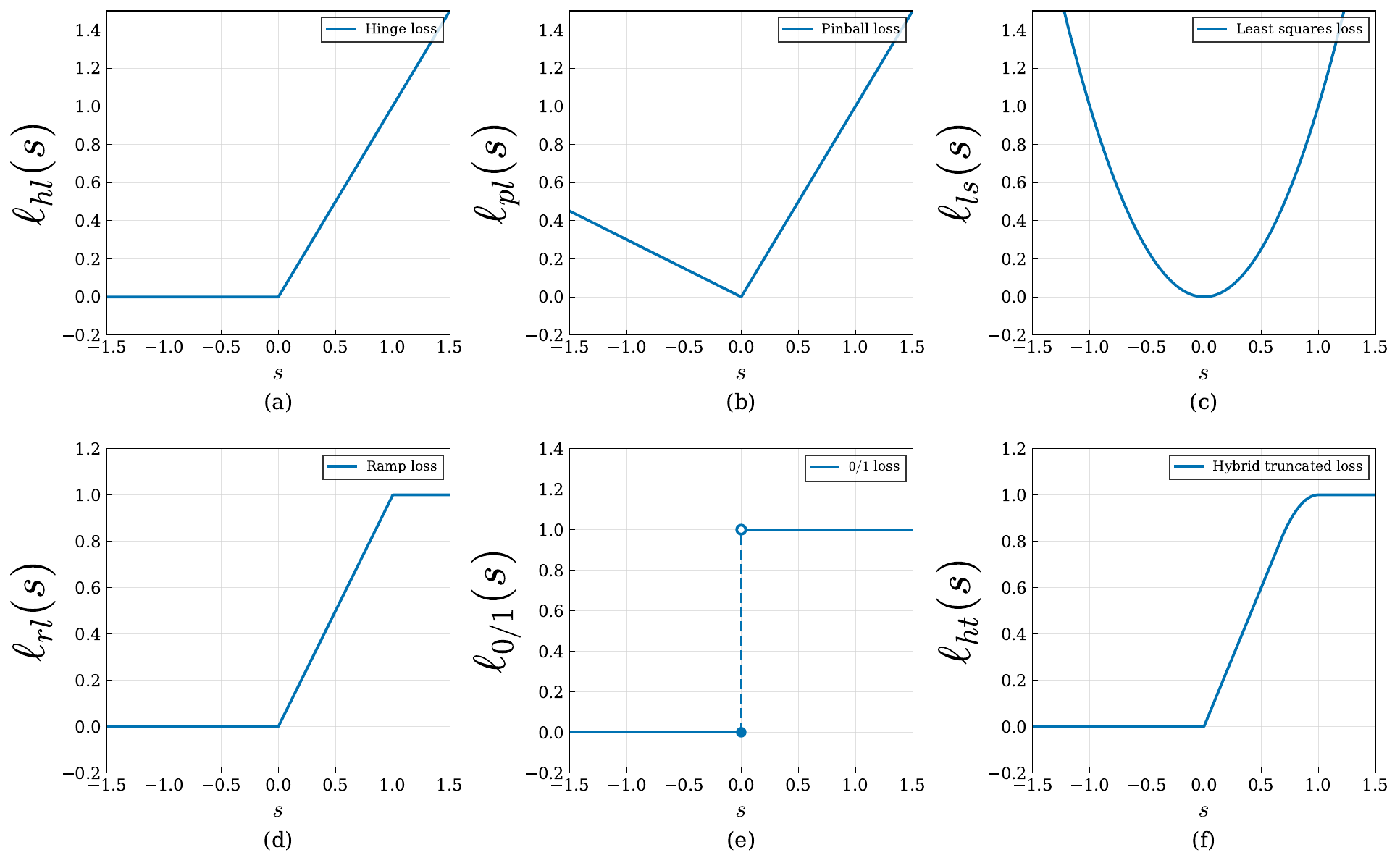}
    \caption{The figures of six loss functions. (a) Hinge loss. (b) Pinball loss. (c) Least squares loss. (d) Ramp loss. (e) $L_{0/1}$ loss. (f) Hybrid truncated loss.}
    \label{fig:loss_functions}
\end{figure*}

\section{Related work}
\label{sec:related}

We begin by fixing notation and then survey the SVM loss 
functions, optimisation algorithms, and multi-view SVM models 
that are most relevant to the present work. 
Let $S^{(V)} = \{(\mathbf{x}_i^{(1)}, \ldots,
\mathbf{x}_i^{(V)}, y_i)\}_{i=1}^{m}$ denote a $V$-view 
dataset with feature spaces 
$\mathbf{x}_i^{(v)} \in \mathcal{X}_v$ ($v \in [V]$) and 
label space $\mathcal{Y} = \{-1, +1\}$. 
Table~\ref{tab:notations} collects the principal symbols used 
throughout the paper.

\begin{table}[!t]
\centering
\caption{List of main notations.}
\label{tab:notations}
\footnotesize
\begin{tabular}{@{}ccc@{}}
\toprule
Symbol & Denotation & Scope \\
\midrule
$m$       & Number of training samples        & $\mathbb{N}$ \\
$n$       & Total dimension of the sample     & $\mathbb{N}$ \\
$n_v$ & Dimension in $v$-th view          & $\mathbb{N}$ \\
$V$       & Number of the views               & $\mathbb{N}$ \\
$\mathbf{w}$        & Weight vector of hyperplane         & $\mathbb{R}^n$ \\
$\mathbf{w}^{(v)}$  & Weight vector in $v$-th view        & $\mathbb{R}^{n_v}$ \\
$b$                  & Bias of the hyperplane              & $\mathbb{R}$ \\
$b^{(v)}$            & Bias in $v$-th view                 & $\mathbb{R}$ \\
$\theta^{(v)}$       & Weight of $v$-th view               & $(0,1)$ \\
$\mathbf{X}^{(v)}$  & Data matrix of $v$-th view          & $\mathbb{R}^{m \times n_v}$ \\
$\mathbf{y}$         & $(y_1, y_2, \ldots, y_m)$           & $\{-1,1\}^m$ \\
$\mathbf{Y}$         & $\mathrm{diag}(y_1, \ldots, y_m)$   & $\mathbb{R}^{m \times m}$ \\
$\mathbf{1}$         & $(1, \ldots, 1)^\top$               & $\mathbb{R}^m$ \\
$N$       & $[y_1\mathbf{x}_1 \!\cdots\! y_m\mathbf{x}_m]^\top$ & $\mathbb{R}^{m \times n}$ \\
$M^{(v)}$ & $[y_1\mathbf{x}_1^{(v)} \!\cdots\! y_m\mathbf{x}_m^{(v)}]^\top$ & $\mathbb{R}^{m \times n_v}$ \\
$\mathbf{t}$         & $(t_1, \ldots, t_m)^\top$           & $\mathbb{R}^m$ \\
$t_i^{(v)}$          & $1 - y_i f^{(v)}(\mathbf{x}_i^{(v)})$ & $\mathbb{R}$ \\
$\|\cdot\|_0$        & Number of non-zero elements         &  \\
$\mathbf{q}$         & Auxiliary variable                  & $\mathbb{R}^m$ \\
$\boldsymbol{\lambda}$ & Lagrangian multiplier             & $\mathbb{R}^m$ \\
$\boldsymbol{\xi}, \boldsymbol{\eta}$ & Slack variables    & $\mathbb{R}_+^m$ \\
$C, C_1, C_2, D, \epsilon$ & Nonnegative parameters       & $\mathbb{R}_+$ \\
$\alpha, \eta, \beta$ & Nonnegative parameters             & $\mathbb{R}_+$ \\
$\xi$                & ADMM penalty parameter              & $\mathbb{R}_+$ \\
$\Sigma^{(v)}$       & Structural covariance matrix        & $\mathbb{R}^{n_v \times n_v}$ \\
$l$                  & Outer loop counter                  & $\mathbb{N}$ \\
$k$                  & Inner loop counter                  & $\mathbb{N}$ \\
\bottomrule
\end{tabular}
\end{table}

\subsection{Loss functions and optimization algorithms for SVM}
\label{subsec:loss-algo-review}

A wide variety of loss functions have been investigated for 
SVM. We group them into convex and non-convex families and 
briefly discuss representative members.

\begin{itemize}
\item \textbf{Hinge loss.} Introduced in the original SVM 
\citep{cortes1995support}, the hinge loss 
(see Fig.~\ref{fig:loss_functions}(a)) is defined as
\[
\ell_{h}(s) :=
\begin{cases}
0, & s < 0,\\
s, & s \ge 0.
\end{cases}
\]
Samples satisfying $s < 0$ incur zero penalty, yielding sparse 
solutions. On the positive side, however, the loss increases 
linearly without bound, making the model susceptible to 
large-margin outliers \citep{brahmi2024efficient,wang2023fast}.

\item \textbf{Pinball loss.} A generalisation of the hinge loss 
with a tunable slope on the negative side 
\citep{huang2013support} 
(see Fig.~\ref{fig:loss_functions}(b)):
\[
\ell_{pl}(s) :=
\begin{cases}
\iota s, & s < 0,\\
s, & s \ge 0,
\end{cases}
\]
where $\iota \in [0,1]$. It is non-smooth at $s=0$ and 
imposes a linear penalty on every sample, leading to denser 
solutions and limited noise tolerance \citep{huang2016solution}.

\item \textbf{Least squares loss.} The quadratic penalty 
\citep{suykens1999least,liu2022l2} 
(see Fig.~\ref{fig:loss_functions}(c)) reads
\[
\ell_{ls}(s) := s^2, \quad s \in \mathbb{R}.
\]
Its differentiability simplifies training, but the unbounded 
quadratic growth amplifies the influence of outliers and 
eliminates sparsity \citep{wang2005comparison}.
\end{itemize}

Convex losses benefit from well-established optimisation 
theory, yet their unboundedness limits robustness under 
contaminated data. Non-convex alternatives address this 
shortcoming at the expense of a harder optimisation landscape.

\begin{itemize}
\item \textbf{Ramp loss.} The ramp loss \citep{wang2024fast} 
(see Fig.~\ref{fig:loss_functions}(d)) truncates the hinge 
loss at a finite ceiling $\overline{\chi}\ge 1$:
\[
\ell_{rl}(s) :=
\begin{cases}
\overline{\chi}, & s \ge \overline{\chi},\\
s, & s \in (0, \overline{\chi}),\\
0, & s \le 0.
\end{cases}
\]
The output is confined to $[0, \overline{\chi}]$, securing both 
sparsity and bounded influence on outliers, although non-smoothness 
at $s=0$ and $s=\overline{\chi}$ must be handled carefully.

\item \textbf{$L_{0/1}$ loss.} Proposed by 
\cite{wang2021support}, the $L_{0/1}$ loss 
(see Fig.~\ref{fig:loss_functions}(e)) is the purest counting 
loss:
\[
\ell_{0/1}(s) :=
\begin{cases}
1, & s > 0,\\
0, & s \le 0.
\end{cases}
\]
Each misclassified point contributes a unit cost irrespective 
of the margin violation magnitude, which maximally bounds the 
outlier influence. The resulting problem, however, is NP-hard, 
and practical solvers can only guarantee convergence to local 
optimisers \citep{akhtar2024roboss}.
\end{itemize}

In summary, non-convex losses achieve superior robustness 
through their bounded output range, but this advantage comes 
with a more demanding optimisation procedure.

\paragraph{Optimization algorithms}
Solving SVM models with non-convex losses requires specialised 
algorithms. Full-sample gradient-based methods---such as the 
concave-convex procedure (CCCP) \citep{wang2021twin}, 
semi-smooth Newton schemes \citep{yin2019semismooth}, and the 
conjugate gradient method \citep{suykens1999least}---are 
straightforward but incur high per-iteration cost because 
they process the entire training set. Sub-sample strategies, 
including sequential minimal optimisation (SMO) 
\citep{huang2013support}, stochastic gradient descent (SGD) 
\citep{allen2018katyusha}, and coordinate descent 
\citep{liu2021smooth}, alleviate this burden yet may exhibit 
slow convergence on non-convex objectives.

A different strategy is ADMM, which splits the original problem into
tractable subproblems solved in alternation 
\citep{feng2022support,dai2023over}. ADMM has been applied to 
SVM through, e.g., hinge-ADMM \citep{ye2011efficient}, elastic 
net ADMM \citep{wang2024fastsvm}, and ramp-ADMM 
\citep{wang2024sparse}. A common bottleneck of these 
approaches is the need to invert a matrix involving all 
training samples at every iteration. The $L_{0/1}$-SVM 
\citep{wang2021support} overcomes this limitation by coupling 
ADMM with a working-set rule that restricts each iteration to 
a small active subset, achieving global convergence to a 
P-stationary point at substantially lower cost.

\subsection{MVL models}
\label{subsec:mvl-review}

\subsubsection{SVM-2K}
SVM-2K \citep{farquhar2005two} builds upon the hinge loss and 
enforces consistency between two views via KCCA theory. For a 
two-view dataset 
$S^{(2)} = \{(\mathbf{x}_i^{(1)}, \mathbf{x}_i^{(2)}, y_i)\}_{i=1}^m$, 
the optimisation problem reads:
\begin{align}
\min_{\substack{\mathbf{w}^{(v)},
\boldsymbol{\xi}^{(v)},
\boldsymbol{\eta}}}
&\; \frac{1}{2}\bigl\|\mathbf{w}^{(1)}\bigr\|_2^2
+ \frac{1}{2}\bigl\|\mathbf{w}^{(2)}\bigr\|_2^2 \notag\\
&\; + C_1 \mathbf{1}^\top \boldsymbol{\xi}^{(1)}
+ C_2 \mathbf{1}^\top \boldsymbol{\xi}^{(2)}
+ D \mathbf{1}^\top \boldsymbol{\eta} \notag\\
\text{s.t.} \quad
&\; \bigl|\mathbf{X}^{(1)}\mathbf{w}^{(1)}
- \mathbf{X}^{(2)}\mathbf{w}^{(2)}\bigr|
\le \boldsymbol{\eta} + \epsilon \mathbf{1}, \notag\\
&\; \mathbf{Y}\mathbf{X}^{(v)}\mathbf{w}^{(v)} \ge \mathbf{1}
- \boldsymbol{\xi}^{(v)}, \notag\\
&\; \boldsymbol{\xi}^{(v)} \ge \mathbf{0}, \;
\boldsymbol{\eta} \ge \mathbf{0}, \quad v = 1, 2.
\label{eq:svm2k}
\end{align}
Here $\mathbf{Y} := \mathrm{diag}(y_1, \ldots, y_m) \in \mathbb{R}^{m \times m}$ 
is the label matrix, $\mathbf{1} := (1, \ldots, 1)^\top \in \mathbb{R}^m$, 
$C_1, C_2, D > 0$ are penalty parameters, and $\epsilon \ge 0$ 
is a tolerance constant. Because the hinge loss is unbounded, 
SVM-2K lacks robustness to outliers. Additionally, it only 
imposes a consensus constraint between the two views and does 
not capture complementary information.

\subsubsection{MvSVM-2C}
MvSVM-2C \citep{xie2019multi} extends SVM-2K to an arbitrary 
number of views while satisfying both the 2C principles. It 
retains the hinge loss but introduces a learnable weight vector 
$\boldsymbol{\theta}$ that adaptively balances view-specific 
contributions. The formulation is:
\begin{align}
\min_{\substack{\mathbf{w}^{(v)},
\boldsymbol{\xi}^{(v)},
\boldsymbol{\eta}^{(j)},
\boldsymbol{\theta}}}
&\; \frac{1}{2}\sum_{v=1}^{V} \theta^{(v)}
\bigl\|\mathbf{w}^{(v)}\bigr\|^2
+ \frac{\beta}{2}\|\boldsymbol{\theta}\|^2 \notag\\
&\; + D \sum_{v=1}^{V}\sum_{j=1}^{V}
\mathbf{1}^\top \boldsymbol{\eta}^{(j)}
+ C \sum_{v=1}^{V} \mathbf{1}^\top
\boldsymbol{\xi}^{(v)} \notag\\
\text{s.t.} \quad
&\; \bigl|\mathbf{X}^{(i)}\mathbf{w}^{(i)}
- \mathbf{X}^{(j)}\mathbf{w}^{(j)}\bigr|
\le \boldsymbol{\eta}^{(j)} + \epsilon \mathbf{1}, \notag\\
&\; \mathbf{Y}\mathbf{X}^{(v)}\mathbf{w}^{(v)} \ge \mathbf{1}
- \boldsymbol{\xi}^{(v)}, \notag\\
&\; \boldsymbol{\xi}^{(v)} \ge \mathbf{0}, \;
\boldsymbol{\eta}^{(j)} \ge \mathbf{0}, \;
\theta^{(v)} > 0, \;
\sum_{v=1}^{V} \theta^{(v)} = 1, \notag\\
&\; i, j, v \in [V], \;
1 \le i \le j \le V.
\label{eq:mvsvm2c}
\end{align}
Here $V$ is the total number of views, $C, D > 0$ are penalty 
parameters, $\epsilon \ge 0$ is a tolerance constant, and 
$\beta > 0$ regularises the weight distribution 
$\boldsymbol{\theta}$. The $\binom{V}{2}$ pairwise constraints 
encode consensus, while the adaptive weights exploit 
complementary information across views. A key limitation of 
MvSVM-2C is its reliance on the unbounded hinge loss, which 
restricts robustness under outlier or label-noise corruption.

\section{Optimality conditions of \texorpdfstring{$L_{\mathrm{ht}}$}{Lht}-SVM}
\label{sec:optimality}

\subsection{Hybrid-truncated loss function}
\label{subsec:ht-loss}

We now construct a loss that retains the sparsity and boundedness of 
the ramp loss while gaining a smooth transition region:
\begin{equation}
\ell_{\mathrm{ht}}(s) :=
\left\{
\begin{array}{ll}
1, & s \ge 1,\\
(-9s^2 + 18s - 4)/5, & s \in (2/3, 1),\\
6s/5, & s \in (0, 2/3],\\
0, & s \le 0.
\end{array}
\right.
\label{eq:ht-loss}
\end{equation}
We call \eqref{eq:ht-loss} the hybrid-truncated ($L_{\mathrm{ht}}$) loss 
(see Fig.~\ref{fig:loss_functions}(f)). The polynomial coefficients are 
determined by imposing $C^1$ continuity at the junctions $s = 2/3$ and 
$s = 1$ and by requiring that the regular subdifferential and the 
proximal operator both possess closed-form expressions. The function 
is continuous on $\mathbb{R}$ and differentiable everywhere except at 
$s = 0$. Two properties are noteworthy:
\begin{enumerate}
  \item[(i)] Sparsity: For samples satisfying $s \le 0$, the $L_{\mathrm{ht}}$ loss 
  returns zero, thereby encouraging sparse solutions.
  \item[(ii)] Robustness: For samples with $s \ge 1$, the $L_{\mathrm{ht}}$ loss is 
  capped at the constant value $1$, preventing any single outlier from 
  dominating the objective.
\end{enumerate}

Substituting $\ell(\cdot)$ in \eqref{eq:svm} with $\ell_{\mathrm{ht}}(\cdot)$ 
yields a new SVM formulation:
\begin{equation}
  \min_{\mathbf{w} \in \mathbb{R}^n, b \in \mathbb{R}} \ 
  \frac{1}{2}\|\mathbf{w}\|^2 + C \sum_{i=1}^m 
  \ell_{\mathrm{ht}}\big(1 - y_i(\langle\mathbf{w}, \mathbf{x}_i\rangle + b)\big).
  \label{eq:ht-svm-primal}
\end{equation}
Introducing an auxiliary variable $\mathbf{q} \in \mathbb{R}^m$, the above 
problem can be equivalently reformulated as:
\begin{equation}
  \min_{\mathbf{w} \in \mathbb{R}^n, b \in \mathbb{R}, \mathbf{q} \in \mathbb{R}^m} \ 
  \frac{1}{2}\|\mathbf{w}\|^2 + C L_{\mathrm{ht}}(\mathbf{q})
  \quad \text{s.t.} \quad \mathbf{q} + N\mathbf{w} + b\mathbf{y} = \mathbf{1},
  \label{eq:ht-svm-constrained}
\end{equation}
where $L_{\mathrm{ht}}(\mathbf{q}) := \sum_{i=1}^m \ell_{\mathrm{ht}}(q_i)$, 
$\mathbf{1} := (1, \ldots, 1)^\top \in \mathbb{R}^m$, 
$\mathbf{y} := (y_1, \ldots, y_m)^\top \in \mathbb{R}^m$, and 
$N := [y_1\mathbf{x}_1 \cdots y_m\mathbf{x}_m]^\top \in \mathbb{R}^{m \times n}$.

We collectively term \eqref{eq:ht-svm-primal} and 
\eqref{eq:ht-svm-constrained} the $L_{\mathrm{ht}}$-SVM. The 
subsequent subsections derive the regular subdifferential and 
proximal operator of $\ell_{\mathrm{ht}}$ and establish the 
optimality theory.

\subsection{\texorpdfstring{$L_{\mathrm{ht}}$}{Lht} regular subdifferential}
\label{subsec:subdiff}

In this subsection, we present the regular subdifferential of the 
hybrid-truncated loss function.

\begin{lemma}
\label{lem:subdiff}
For $\mathbf{q} = (q_1, q_2, \cdots, q_m)^\top \in \mathbb{R}^m$, the 
regular subdifferential of $L_{\mathrm{ht}}$ is given by
\begin{align}
&\partial L_{\mathrm{ht}}(\mathbf{q}) := 
(\partial \ell_{\mathrm{ht}}(q_1), \cdots, \partial \ell_{\mathrm{ht}}(q_m))^\top 
\in \mathbb{R}^m,
\label{eq:subdiff-vector}\\
&\partial \ell_{\mathrm{ht}}(q_i) :=
\left\{
\begin{array}{ll}
0, & q_i \ge 1,\\
(18 - 18q_i)/5, & q_i \in (2/3, 1),\\
6/5, & q_i \in (0, 2/3],\\
{[0, 6/5]}, & q_i = 0,\\
0, & q_i < 0,
\end{array}
\right.
\quad i \in [m].
\label{eq:subdiff-scalar}
\end{align}
\end{lemma}

\noindent\textbf{Proof.} The proof follows directly from the definition 
of regular subdifferential and is omitted here. \hfill$\square$

\subsection{\texorpdfstring{$\ell_{\mathrm{ht}}$}{lht} proximal operator}
\label{subsec:prox}

The concept of proximal operator for the hybrid-truncated loss is given below.

\begin{definition}
\label{def:prox-1d}
Given $z \in \mathbb{R}$, $\nu, C > 0$, the proximal operator of 
$\ell_{\mathrm{ht}}$ is defined by
\begin{equation}
\mathrm{prox}_{\nu C\ell_{\mathrm{ht}}}(z) = 
\arg\min_{s \in \mathbb{R}} \nu C\ell_{\mathrm{ht}}(s) + \frac{1}{2}(s - z)^2.
\label{eq:prox-def}
\end{equation}
\end{definition}

The following lemma provides the closed-form expression of the 
$\ell_{\mathrm{ht}}$ proximal operator.

\begin{lemma}
\label{lem:prox-1d}
For $z \in \mathbb{R}$, $\nu, C > 0$, we have the following results.

\noindent (i) When $\nu C \in (0, 5/18)$, the $\ell_{\mathrm{ht}}$ proximal 
operator has the following exact formula:
\begin{equation}
\mathrm{prox}_{\nu C\ell_{\mathrm{ht}}}(z) =
\left\{
\begin{array}{ll}
z, & z > 1,\\
\frac{5z - 18\nu C}{5 - 18\nu C}
, & z \in (\frac{2}{3} + 6\nu C/5, 1],\\
z - 6\nu C/5, & z \in (6\nu C/5, 2/3 + 6\nu C/5],\\
0, & z \in [0, 6\nu C/5],\\
z, & z < 0.
\end{array}
\right.
\label{eq:prox-case1}
\end{equation}

\noindent (ii) When $\nu C \in [5/18, 25/18)$, the $\ell_{\mathrm{ht}}$ 
proximal operator has the following exact formula:
\begin{equation}
\mathrm{prox}_{\nu C\ell_{\mathrm{ht}}}(z) =
\left\{
\begin{array}{ll}
z, & z > 5/6 + 3\nu C/5,\\
z - 6\nu C/5, & z \in (6\nu C/5, 5/6 + 3\nu C/5],\\
0, & z \in [0, 6\nu C/5],\\
z, & z < 0.
\end{array}
\right.
\label{eq:prox-case2}
\end{equation}

\noindent (iii) For $\nu C \ge 25/18$, the $\ell_{\mathrm{ht}}$ proximal 
operator has the following exact formula:
\begin{equation}
\mathrm{prox}_{\nu C\ell_{\mathrm{ht}}}(z) =
\left\{
\begin{array}{ll}
z, & z > \sqrt{2\nu C},\\
z \text{ or } 0, & z = \sqrt{2\nu C},\\
0, & z \in [0, \sqrt{2\nu C}),\\
z, & z < 0.
\end{array}
\right.
\label{eq:prox-case3}
\end{equation}
\end{lemma}

\noindent\textbf{Proof.} See \ref{app:prox-proof}. \hfill$\square$

\begin{definition}
\label{def:prox-md}
Given $\nu, C > 0$, the proximal operator of $L_{\mathrm{ht}}$ 
at $\mathbf{z} = (z_1, z_2, \ldots, z_m)^\top \in \mathbb{R}^m$ is defined by
\begin{equation}
\mathrm{prox}_{\nu C L_{\mathrm{ht}}}(\mathbf{z}) := 
\begin{pmatrix}
\mathrm{prox}_{\nu C\ell_{\mathrm{ht}}}(z_1)\\
\vdots\\
\mathrm{prox}_{\nu C\ell_{\mathrm{ht}}}(z_m)
\end{pmatrix}
\in \mathbb{R}^m,
\label{eq:prox-vector}
\end{equation}
where $\mathrm{prox}_{\nu C\ell_{\mathrm{ht}}}(z_i)$ is given in 
Lemma~\ref{lem:prox-1d}.
\end{definition}

In this paper, we will utilize Eqs.~\eqref{eq:subdiff-vector}--\eqref{eq:subdiff-scalar} 
and \eqref{eq:prox-vector} to establish the optimality theory of $L_{\mathrm{ht}}$-SVM. 
Additionally, Eq.~\eqref{eq:prox-vector} will be used to develop an efficient 
algorithm for solving $L_{\mathrm{ht}}$-SVM.

\subsection{First-order optimality conditions}
\label{subsec:optimality}

In this part, we turn our attention to establish the first-order optimality 
conditions of \eqref{eq:ht-svm-constrained}, which plays an important role in 
developing a fast algorithm. For this reason, we begin with introducing the 
definition of proximal stationary point (P-stationary point) of 
\eqref{eq:ht-svm-constrained} according to $L_{\mathrm{ht}}$ proximal operator.

\begin{definition}
\label{def:p-stationary}
The $(\mathbf{w}^*; b^*; \mathbf{q}^*)$ is called a P-stationary point of 
\eqref{eq:ht-svm-constrained} if there is $\nu, C > 0$ and 
$\boldsymbol{\psi}^* \in \mathbb{R}^m$ such that
\begin{equation}
\begin{cases}
\mathbf{w}^* + N^\top \boldsymbol{\psi}^* = \mathbf{0},\\
\langle \mathbf{y}, \boldsymbol{\psi}^* \rangle = 0,\\
\mathbf{q}^* + N\mathbf{w}^* + b^*\mathbf{y} = \mathbf{1},\\
\mathrm{prox}_{\nu C L_{\mathrm{ht}}}(\mathbf{q}^* - \nu\boldsymbol{\psi}^*) \ni \mathbf{q}^*.
\end{cases}
\label{eq:p-stationary}
\end{equation}
\end{definition}

Given $C > 0$ and $\mathbf{q}^* \in \mathbb{R}^m$, denote
\begin{align*}
&\mathcal{A}^* := \{k \in [m] : q_k^* > 1\}, \quad 
\mathcal{B}^* := \{k \in [m] : q_k^* \in [\tfrac{2}{3}, 1]\},\\
&\mathcal{C}^* := \{k \in [m] : q_k^* \in (0, \tfrac{2}{3})\}, \quad 
\mathcal{D}^* := \{k \in [m] : q_k^* = 0\},\\
&\mathcal{E}^* := \{k \in [m] : q_k^* < 0\}, \quad 
\mathcal{P}^* := \{k \in [m] : q_k^* \in (0, \tfrac{5}{6})\},\\
&\mathcal{Q}^* := \{k \in [m] : q_k^* = \tfrac{5}{6}\}, \quad 
\mathcal{R}^* := \{k \in [m] : q_k^* > \tfrac{5}{6}\},
\end{align*}
\begin{align}
&\nu_1^* := 
\begin{cases}
\frac{5}{18C}, & \mathcal{B}^* \neq \emptyset,\\
+\infty, & \mathcal{B}^* = \emptyset,
\end{cases}
\label{eq:nu1}\\
&\nu_2^* := 
\begin{cases}
\min_{k \in \mathcal{R}^*} \frac{5(6q_k^* - 5)}{18C}, & \mathcal{R}^* \neq \emptyset,\\
+\infty, & \mathcal{R}^* = \emptyset,
\end{cases}
\label{eq:nu2}\\
&\nu_3^* := 
\begin{cases}
\min_{k \in \mathcal{P}^*} \frac{5(5 - 6q_k^*)}{18C}, & \mathcal{P}^* \neq \emptyset,\\
+\infty, & \mathcal{P}^* = \emptyset,
\end{cases}
\label{eq:nu3}\\
&\nu_4^* := 
\begin{cases}
\frac{25}{18C}, & \mathcal{D}^* \neq \emptyset,\\
+\infty, & \mathcal{D}^* = \emptyset,
\end{cases}
\label{eq:nu4}\\
&\nu_5^* := 
\begin{cases}
+\infty, & \mathbf{q}^* \le 0,\\
\min\{\frac{(q_k^*)^2}{2C} : q_k^* > 0\}, & \text{otherwise}.
\end{cases}
\label{eq:nu5}
\end{align}

By using the above notations, we present the first-order necessary and 
sufficient condition of \eqref{eq:ht-svm-constrained} as follows.

\begin{theorem}[First-order necessary condition]
\label{thm:necessary}
Let $(\mathbf{w}^*; b^*; \mathbf{q}^*)$ be a local minimizer of 
\eqref{eq:ht-svm-constrained}. Then it is a P-stationary point 
with $0 < \nu \le \nu^* := \min\{\nu_1^*, \nu_2^*, \nu_3^*, \nu_4^*, \nu_5^*\}$.
\end{theorem}

\noindent\textbf{Proof.} The proof is presented in \ref{app:necessary}. 
\hfill$\square$

\begin{theorem}[First-order sufficient condition]
\label{thm:sufficient}
Suppose that $(\mathbf{w}^*; b^*; \mathbf{q}^*)$ with $\boldsymbol{\psi}^*$ is 
a P-stationary point of \eqref{eq:ht-svm-constrained}. Then it is a local minimizer of \eqref{eq:ht-svm-constrained}.
\end{theorem}

\noindent\textbf{Proof.} The proof is presented in \ref{app:sufficient}. 
\hfill$\square$

Based on the above two theorems, it is evident that our newly constructed P-stationary 
point must be a local minimizer of \eqref{eq:ht-svm-constrained}, which indicates 
that we can apply the P-stationary point to characterize the support vectors of 
$L_{\mathrm{ht}}$-SVM and employ it as a halting condition of 
our established method in the next part.

\section{Fast algorithm}
\label{sec:algorithm}

We now present the fast solver for $L_{\mathrm{ht}}$-SVM. We first 
characterise its support vectors and then describe an ADMM 
algorithm paired with a working-set selection rule (ADMM-WS).

\subsection{\texorpdfstring{$L_{\mathrm{ht}}$}{Lht} support vectors}
\label{subsec:sv}

The separating hyperplane of $L_{\mathrm{ht}}$-SVM is 
fully determined by its support vectors; all remaining 
training samples can be discarded. Reducing the 
support-vector count is therefore essential for scalability. 
Using the optimality conditions derived in 
Section~\ref{sec:optimality}, we formally define the 
$L_{\mathrm{ht}}$ support vectors below.

\begin{theorem}[$L_{\mathrm{ht}}$ support vectors for $\nu C \in (0, 5/18)$]
\label{thm:sv-case1}
For $\nu C \in (0, 5/18)$, let $(\mathbf{w}^*; b^*; \mathbf{q}^*)$ with 
$\boldsymbol{\psi}^*$ be a P-stationary point of \eqref{eq:ht-svm-constrained}. 
Then, the $\mathbf{w}^*$ satisfies
\begin{equation}
\mathbf{w}^* = -\sum_{i \in \mathbb{G}^*} \psi_i^* y_i \mathbf{x}_i 
\text{ and } \psi_i^* = 0, i \in \overline{\mathbb{G}}^*,
\label{eq:sv-case1}
\end{equation}
\begin{align*}
&\mathbb{G}_1^* := \{i \in [m] : p_i^* < 0\}, \quad 
\mathbb{G}_2^* := \{i \in [m] : p_i^* \in [0, \widetilde{\beta}]\},\\
&\mathbb{G}_3^* := \{i \in [m] : p_i^* \in (\widetilde{\beta}, \widehat{\beta})\},\\
&\mathbb{G}_4^* := \{i \in [m] : p_i^* \in [\widehat{\beta}, 1)\}, \quad 
\mathbb{G}_5^* := \{i \in [m] : p_i^* \ge 1\},
\end{align*}
where $\mathbf{p}^* := \mathbf{q}^* - \nu\boldsymbol{\psi}^*$, 
$\widetilde{\beta} := 6\nu C/5$, $\widehat{\beta} := 2/3 + 6\nu C/5$, 
$\mathbb{G}^* := \mathbb{G}_2^* \cup \mathbb{G}_3^* \cup \mathbb{G}_4^*$ 
and $\overline{\mathbb{G}}^* := [m] \backslash \mathbb{G}^*$. The training 
vectors $\{\mathbf{x}_i \in \mathbb{R}^n : i \in \mathbb{G}^*\}$ are called the 
$L_{\mathrm{ht}}$ support vectors and meet
\begin{equation}
\begin{cases}
y_i(\langle \mathbf{w}^*, \mathbf{x}_i \rangle + b^*) = 1, & i \in \mathbb{G}_2^*,\\
y_i(\langle \mathbf{w}^*, \mathbf{x}_i \rangle + b^*) \in (1/3, 1), & i \in \mathbb{G}_3^*,\\
y_i(\langle \mathbf{w}^*, \mathbf{x}_i \rangle + b^*) \in (0, 1/3], & i \in \mathbb{G}_4^*.
\end{cases}
\label{eq:sv-cond1}
\end{equation}
\end{theorem}

\noindent\textbf{Proof.} The proof is presented in \ref{app:sv-case1}. 
\hfill$\square$

Hence, we can deduce that the $L_{\mathrm{ht}}$-SVM has sparsity since the 
training vectors with $y_i(\langle \mathbf{w}^*, \mathbf{x}_i \rangle + b^*) < 0$ 
are not support vectors and the $L_{\mathrm{ht}}$-SVM enjoys robustness to outliers 
since the training vectors with $y_i(\langle \mathbf{w}^*, \mathbf{x}_i \rangle + b^*) > 1$ 
are not support vectors.

\begin{theorem}[$L_{\mathrm{ht}}$ support vectors for $\nu C \in [5/18, 25/18)$]
\label{thm:sv-case2}
For $\nu C \in [5/18, 25/18)$, let $(\mathbf{w}^*; b^*; \mathbf{q}^*)$ with 
$\boldsymbol{\psi}^*$ be a P-stationary point of \eqref{eq:ht-svm-constrained}. 
Then, the $\mathbf{w}^*$ satisfies
\begin{equation}
\mathbf{w}^* = -\sum_{i \in \mathbb{I}^*} \psi_i^* y_i \mathbf{x}_i 
\text{ and } \psi_i^* = 0, i \in \overline{\mathbb{I}}^*,
\label{eq:sv-case2}
\end{equation}
\begin{align*}
&\mathbb{I}_1^* := \{i \in [m] : p_i^* < 0\}, \quad 
\mathbb{I}_2^* := \{i \in [m] : p_i^* \in [0, \widetilde{\beta}]\},\\
&\mathbb{I}_3^* := \{i \in [m] : p_i^* \in (\widetilde{\beta}, \widehat{\beta}) 
\text{ or } p_i^* = \widehat{\beta}, \psi_i^* \ne 0\},\\
&\mathbb{I}_4^* := \{i \in [m] : p_i^* > \widehat{\beta} 
\text{ or } p_i^* = \widehat{\beta}, \psi_i^* = 0\},
\end{align*}
where $\mathbf{p}^* := \mathbf{q}^* - \nu\boldsymbol{\psi}^*$, 
$\widetilde{\beta} := 6\nu C/5$, $\widehat{\beta} := 5/6 + 3\nu C/5$, 
$\mathbb{I}^* := \mathbb{I}_2^* \cup \mathbb{I}_3^*$ 
and $\overline{\mathbb{I}}^* := [m] \backslash \mathbb{I}^*$. The training 
vectors $\{\mathbf{x}_i \in \mathbb{R}^n : i \in \mathbb{I}^*\}$ are called the 
$L_{\mathrm{ht}}$ support vectors and meet
\begin{equation}
\begin{cases}
y_i(\langle \mathbf{w}^*, \mathbf{x}_i \rangle + b^*) = 1, & i \in \mathbb{I}_2^*,\\
y_i(\langle \mathbf{w}^*, \mathbf{x}_i \rangle + b^*) \in [1/6 + 3\nu C/5, 1), & i \in \mathbb{I}_3^*.
\end{cases}
\label{eq:sv-cond2}
\end{equation}
\end{theorem}

The proof is similar to that of Theorem~\ref{thm:sv-case1} and thus is omitted.
We can deduce that the $L_{\mathrm{ht}}$-SVM has sparsity since the 
training vectors with $y_i(\langle \mathbf{w}^*, \mathbf{x}_i \rangle + b^*) < 1/6 + 3\nu C/5$ 
are not support vectors and the $L_{\mathrm{ht}}$-SVM enjoys robustness to outliers 
since the training vectors with $y_i(\langle \mathbf{w}^*, \mathbf{x}_i \rangle + b^*) > 1$ 
are not support vectors.

\begin{theorem}[$L_{\mathrm{ht}}$ support vectors for $\nu C \ge 25/18$]
\label{thm:sv-case3}
For $\nu C \ge 25/18$, suppose that $(\mathbf{w}^*; b^*; \mathbf{q}^*)$ with 
$\boldsymbol{\psi}^*$ is a P-stationary point of \eqref{eq:ht-svm-constrained}. 
Then, the $\mathbf{w}^*$ satisfies
\begin{equation}
\mathbf{w}^* = -\sum_{i \in \mathbb{K}^*} \psi_i^* y_i \mathbf{x}_i 
\text{ and } \psi_i^* = 0, i \in \overline{\mathbb{K}}^*,
\label{eq:sv-case3}
\end{equation}
\begin{align*}
&\mathbb{K}_1^* := \{i \in [m] : p_i^* < 0\}, \quad 
\mathbb{K}_2^* := \{i \in [m] : p_i^* \in [0, \sqrt{2\nu C})\},\\
&\mathbb{K}_3^* := \{i \in [m] : p_i^* = \sqrt{2\nu C}, \psi_i^* \ne 0\},\\
&\mathbb{K}_4^* := \{i \in [m] : p_i^* > \sqrt{2\nu C} 
\text{ or } p_i^* = \sqrt{2\nu C}, \psi_i^* = 0\},
\end{align*}
where $\mathbf{p}^* := \mathbf{q}^* - \nu\boldsymbol{\psi}^*$, 
$\mathbb{K}^* := \mathbb{K}_2^* \cup \mathbb{K}_3^*$ 
and $\overline{\mathbb{K}}^* := [m] \backslash \mathbb{K}^*$. The training 
vectors $\{\mathbf{x}_i \in \mathbb{R}^n : i \in \mathbb{K}^*\}$ are called the 
$L_{\mathrm{ht}}$ support vectors and meet
\begin{equation}
y_i(\langle \mathbf{w}^*, \mathbf{x}_i \rangle + b^*) = 1, \quad i \in \mathbb{K}^*.
\label{eq:sv-cond3}
\end{equation}
\end{theorem}

The proof is similar to that of Theorem~\ref{thm:sv-case1} and thus is omitted. 
We can deduce that the $L_{\mathrm{ht}}$-SVM has sparsity since the 
training vectors with $y_i(\langle \mathbf{w}^*, \mathbf{x}_i \rangle + b^*) < 1$ 
are not support vectors and the $L_{\mathrm{ht}}$-SVM enjoys robustness to outliers 
since the training vectors with $y_i(\langle \mathbf{w}^*, \mathbf{x}_i \rangle + b^*) > 1$ 
are not support vectors.
Based on Theorems~\ref{thm:sv-case1}--\ref{thm:sv-case3}, we deduce that the 
$L_{\mathrm{ht}}$ support vectors are a small portion of the training set and 
the $L_{\mathrm{ht}}$-SVM possesses sparsity and robustness to outliers, which 
plays an important role in constructing a new working set in the next part. Motivated 
by this, we will establish an ADMM with
working set for dealing with $L_{\mathrm{ht}}$-SVM \eqref{eq:ht-svm-constrained}.

\subsection{\texorpdfstring{$L_{\mathrm{ht}}$}{Lht}-ADMM framework}
\label{subsec:admm}

In this part, we focus on designing an efficient ADMM algorithm to tackle the 
$L_{\mathrm{ht}}$-SVM problem \eqref{eq:ht-svm-constrained}. To handle the 
equality constraint in \eqref{eq:ht-svm-constrained}, we introduce the 
augmented Lagrangian function as
\begin{align}
\mathcal{L}_\xi(\mathbf{w}; b; \mathbf{q}; \boldsymbol{\psi}) 
&= \frac{1}{2}\|\mathbf{w}\|^2 + C L_{\mathrm{ht}}(\mathbf{q}) 
+ \langle \boldsymbol{\psi}, \mathbf{q} - \mathbf{1} + N\mathbf{w} + b\mathbf{y} \rangle \notag\\
&\; + \frac{\xi}{2}\|\mathbf{q} - \mathbf{1} + N\mathbf{w} + b\mathbf{y}\|^2,
\label{eq:aug-lagrangian}
\end{align}
where $\boldsymbol{\psi} \in \mathbb{R}^m$ represents the Lagrangian multiplier 
and $\xi > 0$ is the penalty parameter. Denote $\nu := 1/\xi$. Starting from 
an initial point $(\mathbf{w}^0; b^0; \mathbf{q}^0; \boldsymbol{\psi}^0)$, 
the ADMM iteration proceeds as
\begin{align}
\mathbf{q}^{k+1} &= \argmin_{\mathbf{q} \in \mathbb{R}^m} 
\mathcal{L}_\xi(\mathbf{w}^k, b^k, \mathbf{q}, \boldsymbol{\psi}^k), \notag\\
\mathbf{w}^{k+1} &= \argmin_{\mathbf{w} \in \mathbb{R}^n} 
\frac{\xi}{2}\|\mathbf{w} - \mathbf{w}^k\|_{Z_k}^2 
+ \mathcal{L}_\xi(\mathbf{w}, b^k, \mathbf{q}^{k+1}, \boldsymbol{\psi}^k), \notag\\
b^{k+1} &= \argmin_{b \in \mathbb{R}} 
\mathcal{L}_\xi(\mathbf{w}^{k+1}, b, \mathbf{q}^{k+1}, \boldsymbol{\psi}^k), \notag\\
\boldsymbol{\psi}^{k+1} &= \boldsymbol{\psi}^k 
+ \tau\xi(\mathbf{q}^{k+1} - \mathbf{1} + N\mathbf{w}^{k+1} + b^{k+1}\mathbf{y}).
\label{eq:admm-iteration}
\end{align}
where $Z_k \in \mathbb{R}^{n \times n}$ is a symmetric matrix and 
$\tau \in (0, (1+\sqrt{5})/2)$ denotes the dual step-size parameter. Here 
$\|\mathbf{w} - \mathbf{w}^k\|_{Z_k}^2 := \langle \mathbf{w} - \mathbf{w}^k, 
Z_k(\mathbf{w} - \mathbf{w}^k) \rangle$.

For the construction of $Z_k$, we first establish a working set $F_k$ based on the 
$L_{\mathrm{ht}}$ support vectors. Let 
$\mathbf{p}^k := \mathbf{1} - N\mathbf{w}^k - b^k\mathbf{y} - \nu\boldsymbol{\psi}^k$ and define
\begin{align*}
&T_k^1 := \{i \in [m] : p_i^k \in [0, \widetilde{\beta}]\}, \quad
T_k^2 := \{i \in [m] : p_i^k \in (\widetilde{\beta}, \widehat{\beta})\},\\
&J_k^1 := \{i \in [m] : p_i^k \in (\widetilde{\beta}, \widehat{\beta}) 
\text{ or } p_i^k = \widehat{\beta}, \psi_i^k \ne 0\},\\
&I_k := \{i \in [m] : p_i^k \in [0, \sqrt{2\nu C}) 
\text{ or } p_i^k = \sqrt{2\nu C}, \psi_i^k \ne 0\},
\end{align*}
where $\widetilde{\beta} := 6\nu C/5$ and 
\[
\widehat{\beta} := 
\begin{cases}
2/3 + 6\nu C/5, & \nu C \in (0, 5/18),\\
5/6 + 3\nu C/5, & \nu C \in [5/18, 25/18),\\
\sqrt{2\nu C}, & \nu C \ge 25/18.
\end{cases}
\]
Based on these index sets, 
the working set $F_k$ at iteration $k$ is constructed as
\begin{equation}
F_k := 
\begin{cases}
T_k^1 \cup T_k^2 \cup T_k^3, & \nu C \in (0, \frac{5}{18}),\\[3pt]
J_k^1 \cup J_k^2, & \nu C \in [\frac{5}{18}, \frac{25}{18}),\\[3pt]
I_k, & \nu C \ge \frac{25}{18},
\end{cases}
\label{eq:working-set}
\end{equation}
with $T_k^3 := \{i \in [m] : p_i^k \in [\widehat{\beta}, 1)\}$ and 
$J_k^2 := \{i \in [m] : p_i^k \in [0, \widetilde{\beta}]\}$.
Denote $\overline{F}_k := [m] \backslash F_k$ as the complement of $F_k$. 
The matrix $Z_k$ is then chosen as
\begin{equation}
Z_k = -N_{\overline{F}_k}^\top N_{\overline{F}_k},
\label{eq:Z-matrix}
\end{equation}
where $N_{\overline{F}_k} \in \mathbb{R}^{|\overline{F}_k| \times n}$ is the 
sub-matrix formed by rows of $N$ corresponding to indices in $\overline{F}_k$, 
and $|\overline{F}_k|$ represents its cardinality. This selection of $Z_k$ 
and $F_k$ significantly reduces the computational burden by excluding 
non-support-vector samples. We now derive the closed-form solution for each 
sub-problem in \eqref{eq:admm-iteration}.

\textbf{(i) Updating $\mathbf{q}^{k+1}$.} The $\mathbf{q}$-subproblem in 
\eqref{eq:admm-iteration} can be reformulated as
\begin{align}
\mathbf{q}^{k+1} &= \argmin_{\mathbf{q} \in \mathbb{R}^m} C L_{\mathrm{ht}}(\mathbf{q}) 
+ \langle \boldsymbol{\psi}^k, \mathbf{q} \rangle 
+ \frac{\xi}{2}\|\mathbf{q} - \mathbf{1} + N\mathbf{w}^k + b^k\mathbf{y}\|^2 \notag\\
&= \argmin_{\mathbf{q} \in \mathbb{R}^m} C L_{\mathrm{ht}}(\mathbf{q}) 
+ \frac{\xi}{2}\|\mathbf{q} - \mathbf{p}^k\|^2 = \mathrm{prox}_{\nu C L_{\mathrm{ht}}}(\mathbf{p}^k),
\label{eq:q-update}
\end{align}
where $\mathbf{p}^k = \mathbf{1} - N\mathbf{w}^k - b^k\mathbf{y} - \nu\boldsymbol{\psi}^k$ 
as defined above. Combining this with \eqref{eq:prox-case1}--\eqref{eq:prox-case3} 
yields the closed-form solution via Lemma~\ref{lem:prox-1d}.

\textbf{(ii) Updating $\mathbf{w}^{k+1}$.} The $\mathbf{w}$-subproblem takes the form
\begin{align}
\mathbf{w}^{k+1} &= \argmin_{\mathbf{w} \in \mathbb{R}^n} \frac{1}{2}\|\mathbf{w}\|^2 
+ \frac{\xi}{2}\|\mathbf{w} - \mathbf{w}^k\|_{Z_k}^2 \notag\\
&\; + \langle \boldsymbol{\psi}^k, N\mathbf{w} \rangle 
+ \frac{\xi}{2}\|\mathbf{q}^{k+1} - \mathbf{1} + N\mathbf{w} + b^k\mathbf{y}\|^2.
\label{eq:w-subproblem}
\end{align}
Since this is a convex quadratic problem with differentiable objective, the 
optimality condition leads to the equation
\begin{equation}
\mathbf{0} = \mathbf{w} - \xi N_{\overline{F}_k}^\top N_{\overline{F}_k}(\mathbf{w} - \mathbf{w}^k) 
+ N^\top \boldsymbol{\psi}^k + \xi N^\top(\mathbf{q}^{k+1} - \mathbf{1} + N\mathbf{w} + b^k\mathbf{y}).
\label{eq:w-equation}
\end{equation}
This can be equivalently rewritten as the linear system
\begin{equation}
(I + \xi N_{F_k}^\top N_{F_k})\mathbf{w} = \xi N_{F_k}^\top \boldsymbol{\chi}_{F_k}^k,
\label{eq:w-linear-system}
\end{equation}
where $\boldsymbol{\chi}^k := -(\mathbf{q}^{k+1} + b^k\mathbf{y} - \mathbf{1} 
+ \boldsymbol{\psi}^k/\xi)$. The derivation of \eqref{eq:w-linear-system} from 
\eqref{eq:w-equation} relies on the identity
\[
N_{F_k}^\top N_{F_k} = N^\top N - N_{\overline{F}_k}^\top N_{\overline{F}_k}.
\]
Notably, samples indexed by $\overline{F}_k$ do not appear in 
\eqref{eq:w-linear-system}, meaning that the working set strategy allows us 
to skip computations involving $\{\mathbf{x}_j, j \in \overline{F}_k\}$. 
This leads to substantial speedup when $|F_k|$ is small relative to $m$. 
The linear system \eqref{eq:w-linear-system} is solved differently depending 
on the relationship between $n$ and $|F_k|$. When $n \le |F_k|$, we directly 
solve \eqref{eq:w-linear-system} via
\begin{equation}
\mathbf{w}^{k+1} = (I + \xi N_{F_k}^\top N_{F_k})^{-1} \xi N_{F_k}^\top \boldsymbol{\chi}_{F_k}^k.
\label{eq:w-case1}
\end{equation}
When $n > |F_k|$, we apply the Sherman--Morrison--Woodbury formula to obtain
\[
(I + \xi N_{F_k}^\top N_{F_k})^{-1} = I - \xi N_{F_k}^\top (I + \xi N_{F_k} N_{F_k}^\top)^{-1} N_{F_k},
\]
which gives the update
\begin{equation}
\mathbf{w}^{k+1} = \xi N_{F_k}^\top (I + \xi N_{F_k} N_{F_k}^\top)^{-1} \boldsymbol{\chi}_{F_k}^k.
\label{eq:w-case2}
\end{equation}

\textbf{(iii) Updating $b^{k+1}$.} For the $b$-subproblem, we minimize
\[
b^{k+1} = \argmin_{b \in \mathbb{R}} \langle \boldsymbol{\psi}^k, b\mathbf{y} \rangle 
+ \frac{\xi}{2}\|\mathbf{q}^{k+1} - \mathbf{1} + N\mathbf{w}^{k+1} + b\mathbf{y}\|^2.
\]
This univariate quadratic minimization admits the closed-form solution
\begin{equation}
b^{k+1} = \langle \mathbf{y}, \mathbf{r}^k \rangle / \|\mathbf{y}\|^2 
= \langle \mathbf{y}, \mathbf{r}^k \rangle / m,
\label{eq:b-update}
\end{equation}
where $\mathbf{r}^k := -N\mathbf{w}^{k+1} + \mathbf{1} - \mathbf{q}^{k+1} 
- \boldsymbol{\psi}^k/\xi$.

\textbf{(iv) Updating $\boldsymbol{\psi}^{k+1}$.} The multiplier update from 
\eqref{eq:admm-iteration} becomes
\begin{equation}
\boldsymbol{\psi}_{F_k}^{k+1} = \boldsymbol{\psi}_{F_k}^k 
+ \tau\xi \boldsymbol{\Lambda}_{F_k}^{k+1}, \quad 
\boldsymbol{\psi}_{\overline{F}_k}^{k+1} = \mathbf{0},
\label{eq:psi-update}
\end{equation}
where $\boldsymbol{\Lambda}^{k+1} := \mathbf{q}^{k+1} - \mathbf{1} + N\mathbf{w}^{k+1} 
+ b^{k+1}\mathbf{y}$. The assignment $\boldsymbol{\psi}_{\overline{F}_k}^{k+1} = \mathbf{0}$ 
follows from Theorems~\ref{thm:sv-case1}--\ref{thm:sv-case3}.

The complete procedure is summarized in Algorithm~\ref{alg:admm}.

\begin{algorithm}[t]
\caption{$L_{\mathrm{ht}}$-ADMM for handling problem \eqref{eq:ht-svm-constrained}}
\label{alg:admm}
\begin{algorithmic}[1]
\STATE \textbf{Input:} Data $\{(\mathbf{x}_i, y_i)\}_{i=1}^m$, parameters $C, \xi > 0$, maximum iterations $K$.
\STATE Initialize $(\mathbf{w}^0; b^0; \mathbf{q}^0; \boldsymbol{\psi}^0)$. Set $\nu := 1/\xi$, $k := 0$.
\WHILE{$k \le K$ and the stop condition does not hold}
\STATE Compute $F_k$ as in \eqref{eq:working-set}.
\STATE Compute $\mathbf{q}^{k+1}$ by \eqref{eq:q-update}.
\STATE Compute $\mathbf{w}^{k+1}$ by \eqref{eq:w-case1} if $n \le |F_k|$ and 
\eqref{eq:w-case2} otherwise.
\STATE Compute $b^{k+1}$ by \eqref{eq:b-update}.
\STATE Calculate $\boldsymbol{\psi}^{k+1}$ by \eqref{eq:psi-update}.
\STATE Set $k = k + 1$.
\ENDWHILE
\STATE \textbf{return} the solution $(\mathbf{w}^k, b^k)$ to 
\eqref{eq:ht-svm-constrained}.
\end{algorithmic}
\end{algorithm}

Following Theorem~\ref{thm:sufficient}, the P-stationary point serves as the 
termination criterion for $L_{\mathrm{ht}}$-ADMM. Specifically, the algorithm 
terminates when the current iterate $(\mathbf{w}^k; b^k; \mathbf{q}^k; \boldsymbol{\psi}^k)$ 
approximately satisfies \eqref{eq:p-stationary}, i.e.,
\[
\max\{\epsilon_1^k, \epsilon_2^k, \epsilon_3^k, \epsilon_4^k\} < \epsilon,
\]
where $\epsilon > 0$ is the prescribed tolerance, and
\begin{align*}
&\epsilon_1^k := \frac{\|\mathbf{w}^k + N_{F_k}^\top \boldsymbol{\psi}_{F_k}^k\|}
{1 + \|\mathbf{w}^k\|}, \quad 
\epsilon_2^k := \frac{|\langle \mathbf{y}_{F_k}, \boldsymbol{\psi}_{F_k}^k \rangle|}
{1 + |F_k|},\\
&\epsilon_3^k := \frac{\|\mathbf{q}^k - \mathbf{1} + N\mathbf{w}^k + b^k\mathbf{y}\|}
{\sqrt{m}},\\
&\epsilon_4^k := \frac{\|\mathbf{q}^k - \mathrm{prox}_{\nu C L_{\mathrm{ht}}}
(\mathbf{p}^k)\|}{1 + \|\mathbf{q}^k\|}.
\end{align*}

\subsection{Global convergence}
\label{subsec:convergence}

We now establish the convergence guarantee for $L_{\mathrm{ht}}$-ADMM.

\begin{theorem}
\label{thm:convergence}
If the sequence $\{(\mathbf{w}^k; b^k; \mathbf{q}^k; \boldsymbol{\psi}^k)\}$ 
generated by $L_{\mathrm{ht}}$-ADMM converges to 
$(\mathbf{w}^*; b^*; \mathbf{q}^*; \boldsymbol{\psi}^*)$, then 
$(\mathbf{w}^*; b^*; \mathbf{q}^*)$ is a P-stationary point of 
\eqref{eq:ht-svm-constrained} and consequently a local minimizer of 
\eqref{eq:ht-svm-constrained}.
\end{theorem}

\noindent\textbf{Proof.} The proof is presented in \ref{app:convergence}. 
\hfill$\square$

\subsection{Computational complexity analysis}
\label{subsec:complexity}

We now examine the per-iteration cost of $L_{\mathrm{ht}}$-ADMM. 
Constructing the working set $F_k$ requires $\mathcal{O}(m)$ operations. 
For the $\mathbf{q}$-update, the dominant cost is the matrix-vector product 
$N\mathbf{w}^k$, requiring $\mathcal{O}(mn)$ operations. 
The $\mathbf{w}$-update is dominated by computing $N\mathbf{w}^{k+1}$, 
costing $\mathcal{O}(mn)$. 
The $b$-update reuses the product $N\mathbf{w}^{k+1}$, adding $\mathcal{O}(m)$ cost. 
The linear system solve for $\mathbf{w}^{k+1}$ (matrix inversion) costs 
$\mathcal{O}(\min\{n^2, |F_k|^2\} \max\{n, |F_k|\})$.
The Gram matrix $N_{F_k} N_{F_k}^\top$ requires $\mathcal{O}(n|F_k|^2)$ operations, 
and its inverse $(I + \xi N_{F_k} N_{F_k}^\top)^{-1}$ costs $\mathcal{O}(|F_k|^\omega)$ 
where $\omega \in (2, 3)$ denotes the matrix multiplication exponent. Summing up, the per-iteration complexity of $\mathbf{w}^{k+1}$ is
\[
\mathcal{O}(\min\{n^2, |F_k|^2\} \max\{n, |F_k|\}).
\]
The overall per-iteration complexity is therefore
\[
\mathcal{O}\big(mn + \min\{n^2, |F_k|^2\} \max\{n, |F_k|\}\big).
\]
When the working set is small, i.e., $\max\{|F_k|, n\} \ll m$, the algorithm achieves 
significant computational savings compared to methods that process all samples.

\section{Application in multi-view learning}
\label{sec:multiview}

This section generalises $L_{\mathrm{ht}}$-SVM to the multi-view 
setting. We first describe the multi-view task formulation, then 
construct the structural information matrices, and finally detail 
the optimisation procedure.

\subsection{Multi-view learning task}
\label{subsec:mv-task}

Consider a multi-view training set 
$\mathcal{D} = \bigl\{(\mathbf{x}_i^{(1)},\ldots,\mathbf{x}_i^{(V)},y_i)
\bigr\}_{i=1}^{m}$, where 
$\mathbf{x}_i^{(v)} \in \mathbb{R}^{n_v}$ is the feature 
representation of sample $i$ in view $v$ 
($v \in [V] := \{1,\ldots,V\}$) and 
$y_i \in \{-1,+1\}$ is the label. Each view may encode distinct 
yet complementary discriminative cues.

For each view $v$, we define an affine classifier
\[
f^{(v)}(\mathbf{x}^{(v)}) = \bigl(\mathbf{w}^{(v)}\bigr)^{\!\top}
\mathbf{x}^{(v)} + b^{(v)},
\quad \mathbf{w}^{(v)} \in \mathbb{R}^{n_v},\ b^{(v)} \in \mathbb{R}.
\]
A weight coefficient $\theta^{(v)} \ge 0$ is assigned to each view 
to reflect its contribution to the final decision. The view weights 
satisfy the simplex constraint
\begin{equation}
\theta^{(v)} \ge 0\ (v \in [V]),
\quad
\sum_{v=1}^{V} \theta^{(v)} = 1,
\label{eq:weight-constraint}
\end{equation}
which ensures a convex combination across views. The prediction rule is
\begin{equation}
f(\mathbf{x}) = \mathrm{sign}\!\left(
\sum_{v=1}^{V} \theta^{(v)}
\Bigl[\bigl(\mathbf{w}^{(v)}\bigr)^{\!\top} \mathbf{x}^{(v)} + b^{(v)}\Bigr]
\right).
\label{eq:mv-decision}
\end{equation}
The weight $\theta^{(v)}$ adaptively adjusts the importance of view $v$ 
during training.

\subsection{\texorpdfstring{Mv$L_{\mathrm{ht}}$}{MvLht}-SVM model formulation}
\label{subsec:mv-model}

Extending the single-view $L_{\mathrm{ht}}$-SVM of 
Section~\ref{sec:optimality}, we formulate the multi-view 
$L_{\mathrm{ht}}$-SVM (Mv$L_{\mathrm{ht}}$-SVM). 
Fig.~\ref{fig:mv-flowchart} depicts the overall framework. 
The resulting optimisation problem is
\begin{align}
\min_{\substack{\mathbf{w}^{(v)} \in \mathbb{R}^{n_v}, b^{(v)} \in \mathbb{R},\\
\theta^{(v)} \ge 0,\ v \in [V]}}
&\; \frac{1}{2}\sum_{v=1}^{V} \theta^{(v)} \bigl\|\mathbf{w}^{(v)}\bigr\|^2
+ \frac{\alpha}{2}\|\boldsymbol{\theta}\|^2
+ \frac{\eta}{2}\sum_{v=1}^{V} 
\bigl(\mathbf{w}^{(v)}\bigr)^{\!\top} \Sigma^{(v)} \mathbf{w}^{(v)} \notag\\
&\; + C\sum_{v=1}^{V}\sum_{i=1}^{m} 
\ell_{\mathrm{ht}}\Bigl(1 - y_i f^{(v)}\bigl(\mathbf{x}_i^{(v)}\bigr)\Bigr) \notag\\
\text{s.t.} \quad
&\; \sum_{v=1}^{V} \theta^{(v)} = 1,
\label{eq:mvht-svm}
\end{align}
where $C > 0$ is the penalty parameter for the $L_{\mathrm{ht}}$-loss, $\alpha > 0$ 
regularizes the view weights, $\eta > 0$ controls the structural 
regularization strength, and $\Sigma^{(v)} \in \mathbb{R}^{n_v \times n_v}$ 
is a positive semi-definite covariance matrix capturing local structure 
in view $v$ (see Section~\ref{subsec:structural-info}).

The objective in \eqref{eq:mvht-svm} consists of four terms: 
(i) $\frac{1}{2}\sum_v \theta^{(v)}\|\mathbf{w}^{(v)}\|^2$ for weighted 
margin regularization across views; 
(ii) $\frac{\alpha}{2}\|\boldsymbol{\theta}\|^2$ as the weight complexity 
penalty to avoid degenerate solutions; 
(iii) $C\sum_v\sum_i \ell_{\mathrm{ht}}(\cdot)$ for the $L_{\mathrm{ht}}$-loss over all views, 
retaining sparsity and robustness; 
(iv) $\frac{\eta}{2}\sum_v (\mathbf{w}^{(v)})^\top \Sigma^{(v)} \mathbf{w}^{(v)}$ 
for structural information regularization.

Introducing auxiliary variables $\mathbf{t}^{(v)} \in \mathbb{R}^m$ 
for each view, problem \eqref{eq:mvht-svm} can be rewritten as
\begin{align}
\min_{\substack{\mathbf{w}^{(v)}, b^{(v)}, \mathbf{t}^{(v)},\\
\theta^{(v)} \ge 0,\ v \in [V]}}
&\; \frac{1}{2}\sum_{v=1}^{V} \theta^{(v)} \bigl\|\mathbf{w}^{(v)}\bigr\|^2
+ \frac{\alpha}{2}\|\boldsymbol{\theta}\|^2
+ \frac{\eta}{2}\sum_{v=1}^{V} 
\bigl(\mathbf{w}^{(v)}\bigr)^{\!\top} \Sigma^{(v)} \mathbf{w}^{(v)} \notag\\
&\; + C\sum_{v=1}^{V} L_{\mathrm{ht}}\bigl(\mathbf{t}^{(v)}\bigr) \notag\\
\text{s.t.} \quad
&\; \mathbf{t}^{(v)} + M^{(v)}\mathbf{w}^{(v)} + b^{(v)}\mathbf{y} = \mathbf{1},\quad
\sum_{v=1}^{V} \theta^{(v)} = 1, \quad v \in [V],
\label{eq:mvht-svm-constrained}
\end{align}
where $L_{\mathrm{ht}}(\mathbf{t}^{(v)}) := \sum_{i=1}^m 
\ell_{\mathrm{ht}}(t_i^{(v)})$ and 
$M^{(v)} := [y_1\mathbf{x}_1^{(v)} \cdots y_m\mathbf{x}_m^{(v)}]^\top 
\in \mathbb{R}^{m \times n_v}$. The linear constraints separate the 
non-smooth loss from the regularization terms, which facilitates the 
ADMM-based solver in Section~\ref{subsec:mv-algorithm}.

\begin{figure*}[!t]
\centering
\includegraphics[width=0.95\textwidth]{"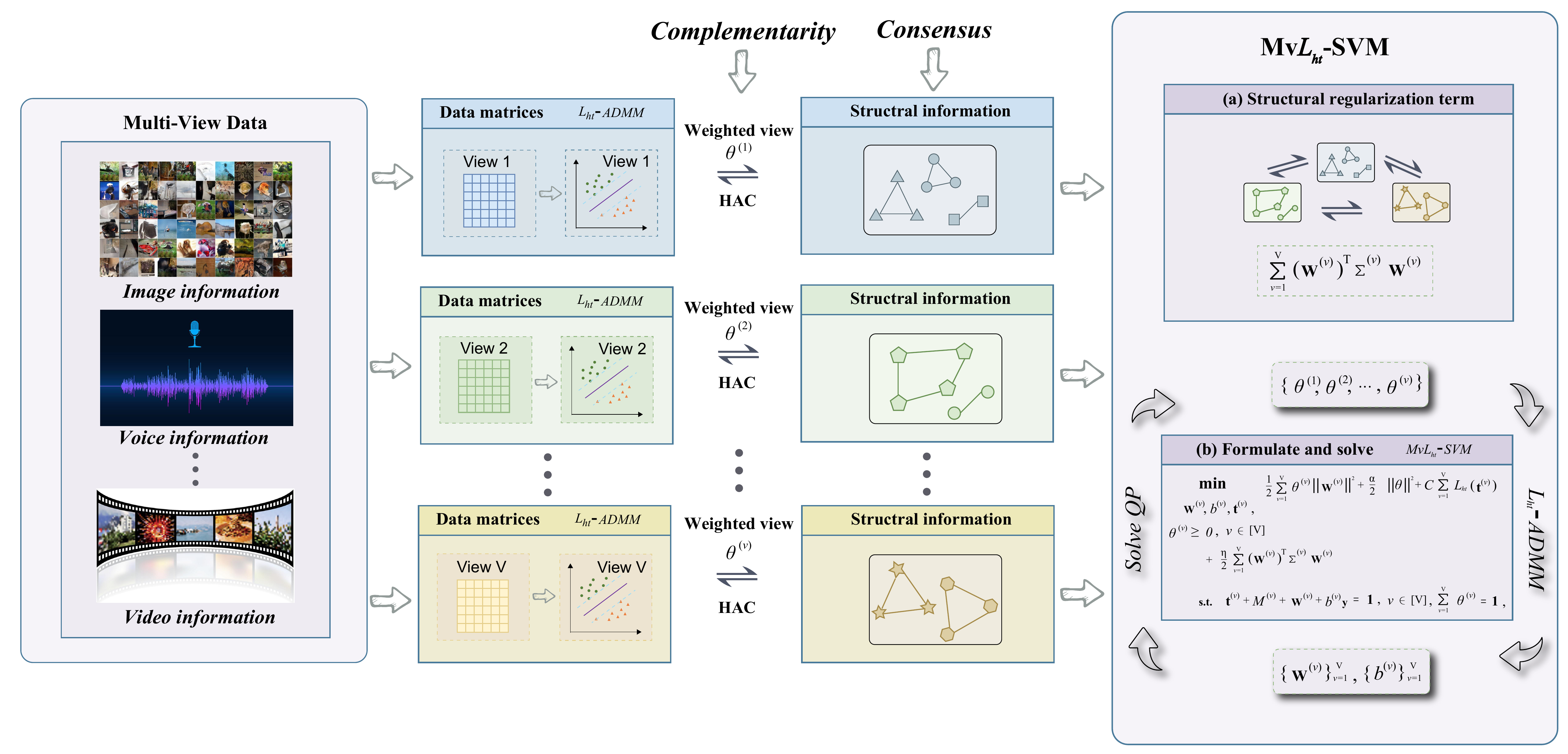"}
\caption{Framework of the proposed Mv$L_{\mathrm{ht}}$-SVM. Multi-view data (e.g., 
image, voice, video) are processed through view-specific $L_{\mathrm{ht}}$-ADMM 
solvers to obtain classifiers $(\mathbf{w}^{(v)}, b^{(v)})$. Structural 
information is constructed via HAC and exchanged across views to capture 
both complementarity and consensus principles. The optimization alternates 
between solving QP for view weights $\boldsymbol{\theta}$ and running 
$L_{\mathrm{ht}}$-ADMM for classifier parameters until convergence.}
\label{fig:mv-flowchart}
\end{figure*}

\subsection{Structural information and construction of \texorpdfstring{$\Sigma^{(v)}$}{Sigma(v)}}
\label{subsec:structural-info}

The quadratic term 
$(\mathbf{w}^{(v)})^\top \Sigma^{(v)} \mathbf{w}^{(v)}$ in 
\eqref{eq:mvht-svm} incorporates local geometric structure, 
drawing on the idea of manifold regularisation 
\citep{belkin2006manifold} and its multi-view extensions 
\citep{liu2025multiview}.

We construct $\Sigma^{(v)}$ with hierarchical agglomerative 
clustering (HAC) \citep{ward1963hierarchical}. Within each 
view $u$, positive and negative samples are clustered 
separately; the number of clusters $K_u$ is determined by 
the L-method. Denoting the index set of cluster $j$ in 
view $u$ by $\mathcal{C}_j^{(u)}$, we form the within-cluster 
covariance of view $v$'s features as
\begin{equation}
\Sigma_j^{(v,u)} = \frac{1}{|\mathcal{C}_j^{(u)}|}
\sum_{i \in \mathcal{C}_j^{(u)}} 
\bigl(\mathbf{x}_i^{(v)} - \boldsymbol{\mu}_j^{(v,u)}\bigr)
\bigl(\mathbf{x}_i^{(v)} - \boldsymbol{\mu}_j^{(v,u)}\bigr)^{\!\top},
\label{eq:local-cov}
\end{equation}
where $\boldsymbol{\mu}_j^{(v,u)} = |\mathcal{C}_j^{(u)}|^{-1} 
\sum_{i \in \mathcal{C}_j^{(u)}} \mathbf{x}_i^{(v)}$ is the centroid.

The aggregate structural matrix for view $v$ is obtained by 
summing over all views and clusters, weighted by the view 
coefficients:
\begin{equation}
\Sigma^{(v)} = \sum_{u=1}^{V} \theta^{(u)} \sum_{j=1}^{K_u} \Sigma_j^{(v,u)}.
\label{eq:global-cov}
\end{equation}
Because every $\Sigma_j^{(v,u)}$ is positive semi-definite, so 
is $\Sigma^{(v)}$. The resulting regulariser discourages large 
projections along high-variance directions within clusters, 
steering the classifier towards locally smooth decision 
boundaries.

\subsection{Optimization algorithm}
\label{subsec:mv-algorithm}

We solve \eqref{eq:mvht-svm-constrained} via a two-level 
alternating scheme. In the \emph{inner loop}, the view weights 
$\boldsymbol{\theta}$ are held fixed and every view-specific 
classifier $(\mathbf{w}^{(v)}, b^{(v)})$ is updated independently. 
In the \emph{outer loop}, the classifiers are frozen and 
$\boldsymbol{\theta}$ is re-estimated. The two loops alternate 
until convergence.

After fixing $\boldsymbol{\theta}$, 
the sub-problem for each view becomes independent and shares the same 
structure as the single-view $L_{\mathrm{ht}}$-SVM. For view $v$, let 
$(\mathbf{w}^{(v)*}, b^{(v)*}, \mathbf{t}^{(v)*})$ with 
$\boldsymbol{\lambda}^{(v)*}$ be a P-stationary point. Define 
$\mathbf{p}^{(v)*} := \mathbf{t}^{(v)*} - \nu\boldsymbol{\lambda}^{(v)*}$ 
with $\nu := 1/\xi$. According to Theorems~\ref{thm:sv-case1}--\ref{thm:sv-case3}, 
the support vector index set for view $v$ is
\begin{equation}
\mathcal{I}^{(v)*} := 
\begin{cases}
\mathcal{G}^{(v)*}_2 \cup \mathcal{G}^{(v)*}_3 \cup \mathcal{G}^{(v)*}_4, 
& \nu C \in (0, 5/18),\\[3pt]
\mathcal{I}^{(v)*}_2 \cup \mathcal{I}^{(v)*}_3, 
& \nu C \in [5/18, 25/18),\\[3pt]
\mathcal{K}^{(v)*}_2 \cup \mathcal{K}^{(v)*}_3, & \nu C \ge 25/18,
\end{cases}
\label{eq:mv-sv-index}
\end{equation}
where the index sets $\mathcal{G}^{(v)*}_j$, $\mathcal{I}^{(v)*}_j$, 
$\mathcal{K}^{(v)*}_j$ are defined analogously to the single-view case 
(see Theorems~\ref{thm:sv-case1}--\ref{thm:sv-case3}) with view-specific variables. The classifier satisfies
\begin{equation}
\mathbf{w}^{(v)*} = -\bigl(\theta^{(v)}I + \eta\Sigma^{(v)}\bigr)^{-1} 
\sum_{i \in \mathcal{I}^{(v)*}} \lambda_i^{(v)*} y_i \mathbf{x}_i^{(v)}.
\label{eq:mv-w-sv}
\end{equation}
Compared with the single-view case \eqref{eq:sv-case1}, the regularization 
matrix changes from $I$ to $\theta^{(v)}I + \eta\Sigma^{(v)}$, while the 
sparsity property is preserved.

\subsubsection{Update classifiers via \texorpdfstring{$L_{\mathrm{ht}}$}{Lht}-ADMM-WS}

For view $v$, the augmented Lagrangian function is defined as
\begin{align}
&\mathcal{L}_\xi^{(v)}(\mathbf{w}^{(v)}, b^{(v)}, \mathbf{t}^{(v)}, 
\boldsymbol{\lambda}^{(v)}) \notag\\
={}&\frac{\theta^{(v)}}{2}\|\mathbf{w}^{(v)}\|^2 
+ C L_{\mathrm{ht}}(\mathbf{t}^{(v)})
+ \langle\boldsymbol{\lambda}^{(v)}, \mathbf{t}^{(v)} - \mathbf{1} 
+ M^{(v)}\mathbf{w}^{(v)} + b^{(v)}\mathbf{y}\rangle \notag\\
&+ \frac{\xi}{2}\|\mathbf{t}^{(v)} - \mathbf{1} 
+ M^{(v)}\mathbf{w}^{(v)} + b^{(v)}\mathbf{y}\|^2
+ \frac{\eta}{2}(\mathbf{w}^{(v)})^\top\Sigma^{(v)}\mathbf{w}^{(v)},
\label{eq:mv-aug-lagrangian}
\end{align}
where $\boldsymbol{\lambda}^{(v)} \in \mathbb{R}^m$ is the Lagrangian 
multiplier and $\xi > 0$ is the penalty parameter. Starting from 
an initial point, the ADMM iteration proceeds as
\begin{align}
\mathbf{t}^{(v)^{k+1}} &= \argmin_{\mathbf{t}^{(v)}} 
\mathcal{L}_\xi^{(v)}(\mathbf{w}^{(v)^k}, b^{(v)^k}, \mathbf{t}^{(v)}, 
\boldsymbol{\lambda}^{(v)^k}), \notag\\
\mathbf{w}^{(v)^{k+1}} &= \argmin_{\mathbf{w}^{(v)} \in \mathbb{R}^{n_v}} 
\mathcal{L}_\xi^{(v)}(\mathbf{w}^{(v)}, b^{(v)^k}, \mathbf{t}^{(v)^{k+1}}, 
\boldsymbol{\lambda}^{(v)^k}) + \frac{\xi}{2}\big\|\mathbf{w}^{(v)} 
- \mathbf{w}^{(v)^k}\big\|_{Z_k^{(v)}}^2, \notag\\
b^{(v)^{k+1}} &= \argmin_{b^{(v)} \in \mathbb{R}} 
\mathcal{L}_\xi^{(v)}(\mathbf{w}^{(v)^{k+1}}, b^{(v)}, 
\mathbf{t}^{(v)^{k+1}}, \boldsymbol{\lambda}^{(v)^k}), \notag\\
\boldsymbol{\lambda}^{(v)^{k+1}} &= \boldsymbol{\lambda}^{(v)^k} 
+ \tau\xi\bigl(\mathbf{t}^{(v)^{k+1}} - \mathbf{1} + M^{(v)}\mathbf{w}^{(v)^{k+1}} 
+ b^{(v)^{k+1}}\mathbf{y}\bigr).
\label{eq:mv-admm-iteration}
\end{align}
where $Z_k^{(v)} \in \mathbb{R}^{n_v \times n_v}$ is a symmetric matrix 
and $\tau \in (0, (1+\sqrt{5})/2)$ is the dual step-size.

Following the single-view case, we construct the working set 
$F_k^{(v)}$ for each view. Let 
$\mathbf{p}^{(v)^k} := \mathbf{1} - M^{(v)}\mathbf{w}^{(v)^k} - b^{(v)^k}\mathbf{y} 
- \nu\boldsymbol{\lambda}^{(v)^k}$. 
The working set $F_k^{(v)}$ is constructed according to \eqref{eq:working-set} 
with view-specific variables. Denote 
$\overline{F}_k^{(v)} := [m] \backslash F_k^{(v)}$ and set
\begin{equation}
Z_k^{(v)} = -M_{\overline{F}_k^{(v)}}^{(v)\top} M_{\overline{F}_k^{(v)}}^{(v)}.
\label{eq:mv-Z-matrix}
\end{equation}
This choice excludes non-support-vector samples from computation, 
leading to significant speedup when $|F_k^{(v)}| \ll m$.

\textbf{(i) Updating $\mathbf{t}^{(v)^{k+1}}$.} The $\mathbf{t}^{(v)}$-subproblem 
reduces to
\begin{equation}
\mathbf{t}^{(v)^{k+1}} = \argmin_{\mathbf{t}^{(v)}} C L_{\mathrm{ht}}(\mathbf{t}^{(v)}) 
+ \frac{\xi}{2}\|\mathbf{t}^{(v)} - \mathbf{p}^{(v)^k}\|^2 
= \mathrm{prox}_{\nu C L_{\mathrm{ht}}}(\mathbf{p}^{(v)^k}),
\label{eq:mv-t-update}
\end{equation}
where $\mathbf{p}^{(v)^k} = \mathbf{1} - M^{(v)}\mathbf{w}^{(v)^k} - b^{(v)^k}\mathbf{y} 
- \nu\boldsymbol{\lambda}^{(v)^k}$ as defined above. By Lemma~\ref{lem:prox-1d}, 
the closed-form solution is obtained component-wise.

\textbf{(ii) Updating $\mathbf{w}^{(v)^{k+1}}$.} Setting the gradient of 
the $\mathbf{w}^{(v)}$-subproblem to zero yields the linear system
\begin{equation}
\bigl(\theta^{(v)}I + \eta\Sigma^{(v)} + \xi M_{F_k^{(v)}}^{(v)\top} 
M_{F_k^{(v)}}^{(v)}\bigr)\mathbf{w}^{(v)} = \xi M_{F_k^{(v)}}^{(v)\top}
\boldsymbol{\chi}_{F_k^{(v)}}^{(v)^k},
\label{eq:mv-w-linear-system}
\end{equation}
where $\boldsymbol{\chi}^{(v)^k} := -(\mathbf{t}^{(v)^{k+1}} + b^{(v)^k}\mathbf{y} 
- \mathbf{1} + \boldsymbol{\lambda}^{(v)^k}/\xi)$. The derivation uses the identity
\[
M_{F_k^{(v)}}^{(v)\top} M_{F_k^{(v)}}^{(v)} = M^{(v)\top} M^{(v)} 
- M_{\overline{F}_k^{(v)}}^{(v)\top} M_{\overline{F}_k^{(v)}}^{(v)}.
\]
Depending on the relationship between $n_v$ and $|F_k^{(v)}|$, we have two cases. 
When $n_v \le |F_k^{(v)}|$, we solve \eqref{eq:mv-w-linear-system} directly:
\begin{equation}
\mathbf{w}^{(v)^{k+1}} = \bigl(\theta^{(v)}I + \eta\Sigma^{(v)} 
+ \xi M_{F_k^{(v)}}^{(v)\top} M_{F_k^{(v)}}^{(v)}\bigr)^{-1} 
\xi M_{F_k^{(v)}}^{(v)\top}\boldsymbol{\chi}_{F_k^{(v)}}^{(v)^k}.
\label{eq:mv-w-case1}
\end{equation}
When $n_v > |F_k^{(v)}|$, applying the Sherman--Morrison--Woodbury formula gives
\begin{align}
\mathbf{w}^{(v)^{k+1}} &= \xi(\theta^{(v)}I + \eta\Sigma^{(v)})^{-1}
M_{F_k^{(v)}}^{(v)\top} \notag\\
&\; \times \bigl(I + \xi M_{F_k^{(v)}}^{(v)}(\theta^{(v)}I 
+ \eta\Sigma^{(v)})^{-1}M_{F_k^{(v)}}^{(v)\top}\bigr)^{-1}
\boldsymbol{\chi}_{F_k^{(v)}}^{(v)^k}.
\label{eq:mv-w-case2}
\end{align}

\textbf{(iii) Updating $b^{(v)^{k+1}}$.} The $b^{(v)}$-subproblem admits 
the closed-form solution
\begin{equation}
b^{(v)^{k+1}} = \langle\mathbf{y}, \mathbf{r}^{(v)^k}\rangle / m,
\label{eq:mv-b-update}
\end{equation}
where $\mathbf{r}^{(v)^k} := -M^{(v)}\mathbf{w}^{(v)^{k+1}} + \mathbf{1} 
- \mathbf{t}^{(v)^{k+1}} - \boldsymbol{\lambda}^{(v)^k}/\xi$.

\textbf{(iv) Updating $\boldsymbol{\lambda}^{(v)^{k+1}}$.} The multiplier 
update follows from \eqref{eq:mv-admm-iteration}:
\begin{equation}
\boldsymbol{\lambda}_{F_k^{(v)}}^{(v)^{k+1}} = \boldsymbol{\lambda}_{F_k^{(v)}}^{(v)^k} 
+ \tau\xi\boldsymbol{\Lambda}_{F_k^{(v)}}^{(v)^{k+1}}, \quad 
\boldsymbol{\lambda}_{\overline{F}_k^{(v)}}^{(v)^{k+1}} = \mathbf{0},
\label{eq:mv-lambda-update}
\end{equation}
where $\boldsymbol{\Lambda}^{(v)^{k+1}} := \mathbf{t}^{(v)^{k+1}} - \mathbf{1} 
+ M^{(v)}\mathbf{w}^{(v)^{k+1}} + b^{(v)^{k+1}}\mathbf{y}$.

For the inner ADMM loop, we terminate when
\[
\max\{\epsilon_1^{(v)^k}, \epsilon_2^{(v)^k}, \epsilon_3^{(v)^k}, 
\epsilon_4^{(v)^k}\} < \epsilon,
\]
where the error measures are defined analogously to those of Algorithm~\ref{alg:admm}.

\subsubsection{Update view weights \texorpdfstring{$\boldsymbol{\theta}$}{theta}}

After fixing all classifiers 
$\{\mathbf{w}^{(v)}, b^{(v)}\}_{v=1}^V$, the $\boldsymbol{\theta}$ 
optimization becomes a small-scale convex quadratic program:
\begin{equation}
\min_{\boldsymbol{\theta} \in \mathbb{R}^V} 
\frac{1}{2}\boldsymbol{\theta}^\top\boldsymbol{\pi} 
+ \frac{\eta}{2}\boldsymbol{\theta}^\top\boldsymbol{\rho} 
+ \frac{\alpha}{2}\|\boldsymbol{\theta}\|_2^2, \quad 
\text{s.t.} \sum_{v=1}^V \theta^{(v)} = 1, \; \theta^{(v)} \ge 0,
\label{eq:mv-theta-qp}
\end{equation}
where $\boldsymbol{\pi} = [\|\mathbf{w}^{(1)}\|^2, \ldots, \|\mathbf{w}^{(V)}\|^2]^\top$ 
and $\boldsymbol{\rho} = (\rho_1, \ldots, \rho_V)^\top$ with 
$\rho_v = \sum_{u=1}^V (\mathbf{w}^{(u)})^\top \Sigma^{(u,v)} \mathbf{w}^{(u)}$, 
where $\Sigma^{(u,v)} := \sum_{j=1}^{K_v} \Sigma_j^{(u,v)}$ and 
$\Sigma_j^{(u,v)}$ is defined analogously to \eqref{eq:local-cov} 
by computing the covariance of view $u$ features over cluster $j$ 
from view $v$. This can be solved efficiently by standard QP solvers.

For the outer loop, we stop when 
$\|\boldsymbol{\theta}^{(l+1)} - \boldsymbol{\theta}^{(l)}\| < \epsilon$. 
The complete procedure is summarized in Algorithm~\ref{alg:mv-admm}.

\begin{algorithm}[t]
\caption{Mv$L_{\mathrm{ht}}$-SVM: Alternating ADMM for solving \eqref{eq:mvht-svm-constrained}}
\label{alg:mv-admm}
\begin{algorithmic}[1]
\STATE \textbf{Input:} Multi-view data $\{\mathbf{X}^{(v)}\}_{v=1}^V$, labels $\mathbf{y}$, 
parameters $C, \eta, \alpha, \xi$, maximum outer iterations $l_{\max}$, maximum inner iterations $K$.
\STATE Initialize $\theta^{(v)} = 1/V$ for $v \in [V]$, $l = 0$, $\Delta\theta = +\infty$.
\WHILE{$\Delta\theta \ge \epsilon$ and $l \le l_{\max}$}
\FOR{$v = 1$ \textbf{to} $V$}
\STATE Compute $\Sigma^{(v)}$ by \eqref{eq:global-cov}.
\STATE Initialize $\mathbf{w}^{(v)^0}, b^{(v)^0}, \mathbf{t}^{(v)^0}, 
\boldsymbol{\lambda}^{(v)^0}$, $k = 0$.
\WHILE{$\max\{\epsilon_1^{(v)^k}, \epsilon_2^{(v)^k}, \epsilon_3^{(v)^k}, 
\epsilon_4^{(v)^k}\} \ge \epsilon$ and $k \le K$}
\STATE Compute $F_k^{(v)}$ as in \eqref{eq:working-set}.
\STATE Compute $\mathbf{t}^{(v)^{k+1}}$ by \eqref{eq:mv-t-update}.
\STATE Compute $\mathbf{w}^{(v)^{k+1}}$ by \eqref{eq:mv-w-case1} if 
$n_v \le |F_k^{(v)}|$, or by \eqref{eq:mv-w-case2} otherwise.
\STATE Compute $b^{(v)^{k+1}}$ by \eqref{eq:mv-b-update}.
\STATE Compute $\boldsymbol{\lambda}^{(v)^{k+1}}$ by \eqref{eq:mv-lambda-update}.
\STATE $k \leftarrow k + 1$.
\ENDWHILE
\STATE \textbf{return} $(\mathbf{w}^{(v)*}, b^{(v)*}) \leftarrow 
(\mathbf{w}^{(v)^k}, b^{(v)^k})$.
\ENDFOR
\STATE Solve QP \eqref{eq:mv-theta-qp} to obtain $\boldsymbol{\theta}^{\mathrm{new}}$.
\STATE $\Delta\theta \leftarrow \|\boldsymbol{\theta}^{\mathrm{new}} - \boldsymbol{\theta}\|$; 
$\boldsymbol{\theta} \leftarrow \boldsymbol{\theta}^{\mathrm{new}}$; $l \leftarrow l + 1$.
\ENDWHILE
\STATE \textbf{Output:} $\boldsymbol{\theta}^*$, 
$\{\mathbf{w}^{(v)*}, b^{(v)*}\}_{v=1}^V$.
\end{algorithmic}
\end{algorithm}

\subsection{Computational complexity analysis}
\label{subsec:mv-complexity}

The computational cost of Algorithm~\ref{alg:mv-admm} consists of two parts: 
the inner ADMM loop for each view and the outer loop for updating 
$\boldsymbol{\theta}$.

For the inner ADMM in view $v$, let $|F_k^{(v)}|$ denote the working set size 
at iteration $k$. The main costs per iteration are as follows. 
Computing $\mathbf{t}^{(v)^{k+1}}$ via the proximal operator requires 
$\mathcal{O}(m)$ operations. 
When $n_v \le |F_k^{(v)}|$, solving \eqref{eq:mv-w-case1} costs 
$\mathcal{O}(n_v^2 |F_k^{(v)}| + n_v^3)$ for forming and inverting the 
$(n_v \times n_v)$ matrix. 
When $n_v > |F_k^{(v)}|$, solving \eqref{eq:mv-w-case2} costs 
$\mathcal{O}(n_v |F_k^{(v)}|^2 + |F_k^{(v)}|^3)$ for inverting the 
$(|F_k^{(v)}| \times |F_k^{(v)}|)$ matrix. 
Computing $b^{(v)^{k+1}}$ and $\boldsymbol{\lambda}^{(v)^{k+1}}$ requires 
$\mathcal{O}(m)$ and $\mathcal{O}(|F_k^{(v)}|)$ operations, respectively.

Due to the sparsity of $L_{\mathrm{ht}}$ support vectors, $|F_k^{(v)}| \ll m$ typically 
holds after a few iterations. Suppose the inner loop converges in $K$ 
iterations on average. The per-view cost is 
$\mathcal{O}(K(n_v^2 |F_k^{(v)}| + n_v^3))$ when $n_v \le |F_k^{(v)}|$, 
or $\mathcal{O}(K(n_v |F_k^{(v)}|^2 + |F_k^{(v)}|^3))$ otherwise.

For the view weight update, solving the QP \eqref{eq:mv-theta-qp} involves 
a $(V \times V)$ quadratic program, which costs $\mathcal{O}(V^3)$. Since 
$V$ is typically small (e.g., $V \le 10$), this step is negligible.

Let $L$ be the number of outer iterations. The total complexity of 
Algorithm~\ref{alg:mv-admm} is
\[
\mathcal{O}\Bigl(L \cdot \sum_{v=1}^V K \cdot \max\{n_v^2 |F_k^{(v)}|, 
n_v |F_k^{(v)}|^2\}\Bigr).
\]
Compared with methods that process all $m$ samples, the working set 
strategy yields significant speedup when the support vector ratio is low.

\section{Numerical experiments}
\label{sec:experiment}

\subsection{Single-view experiments}
\label{subsec:sv-exp}

\subsubsection{Experimental settings}
\label{subsubsec:sv-settings}

This subsection benchmarks the proposed $L_{\mathrm{ht}}$-SVM 
against five representative single-view SVM solvers in terms of 
classification accuracy and training efficiency. The platform 
is a Windows~11 workstation (Intel Core i9-13980HX, 32\,GB RAM) 
running Python~3.11.

\paragraph{Datasets}
We assemble 15 real-world datasets drawn from the 
LIBSVM\footnote{\url{https://www.csie.ntu.edu.tw/~cjlin/libsvmtools/datasets/}} 
and 
UCI\footnote{\url{https://archive.ics.uci.edu/ml/index.php}} 
repositories (Table~\ref{tab:sv-datasets}). All multi-class 
labels are binarised by mapping non-``$+1$'' classes to $-1$, 
and each feature is normalised to $[-1,1]$. 
For UCI datasets, results are averaged over 5-fold 
cross-validation; for LIBSVM datasets, the official train/test 
partitions are adopted directly.

For the synthetic experiments, we sample two-dimensional data 
from a pair of Gaussian distributions. Specifically, $m$ 
positive samples are generated from 
$\mathcal{N}(\boldsymbol{\theta}_1, \Phi_1)$ with 
$\boldsymbol{\theta}_1 = (0.5, -3)^\top$, and $m$ negative 
samples from $\mathcal{N}(\boldsymbol{\theta}_2, \Phi_2)$ 
with $\boldsymbol{\theta}_2 = (-0.5, 3)^\top$; both classes 
share the covariance matrix 
$\Phi_1 = \Phi_2 = \bigl(\begin{smallmatrix} 0.2 & 0 \\ 0 & 3 \end{smallmatrix}\bigr)$. 
The pooled dataset is then evenly split into training and 
testing subsets.

\begin{table}[!t]
\centering
\caption{Details of the 15 real datasets.}
\footnotesize
\label{tab:sv-datasets}
\begin{tabular}{ccccc}
\toprule
Datasets & Source & $m$ & $m_t$ & $n$ \\
\midrule
australian    & UCI      & 690    & --     & 14  \\
breast-cancer & UCI      & 683    & --     & 9   \\
mushroom      & UCI      & 8124   & --     & 22  \\
wine-red      & UCI      & 1599   & --     & 11  \\
wine-white    & UCI      & 4898   & --     & 11  \\
a1a           & LIBSVM$^*$ & 1605   & 30956  & 123 \\
a2a           & LIBSVM$^*$ & 2265   & 30296  & 123 \\
a4a           & LIBSVM$^*$ & 4781   & 27780  & 123 \\
a6a           & LIBSVM$^*$ & 11220  & 21341  & 123 \\
a7a           & LIBSVM$^*$ & 16100  & 16461  & 123 \\
a8a           & LIBSVM$^*$ & 22696  & 9865   & 123 \\
a9a           & LIBSVM$^*$ & 32561  & 16281  & 123 \\
w2a           & LIBSVM$^*$ & 3470   & 46279  & 300 \\
w5a           & LIBSVM$^*$ & 9888   & 39861  & 300 \\
w6a           & LIBSVM$^*$ & 17188  & 32561  & 300 \\
\bottomrule
\end{tabular}\\[4pt]
{\raggedright\scriptsize $m$: number of training samples; $m_t$: number of testing samples; $n$: number of features.\\
$^*$\,LIBSVM datasets use the predefined train/test split.\par}
\end{table}

\paragraph{Comparison methods}

(1)~Hinge-SVM (hinge loss): the standard SVM with the hinge 
soft-margin loss, implemented by LIBSVM \citep{cortes1995support}.
(2)~LS-SVM (least squares loss): SVM with the least squares 
soft-margin loss \citep{suykens1999least}.
(3)~Pin-SVM (pinball loss): SVM with the pinball soft-margin loss 
\citep{huang2013support}.
(4)~Ramp-SVM (ramp loss): SVM with the ramp soft-margin loss, 
solved by the concave-convex procedure \citep{collobert2006trading}.
(5)~$L_{0/1}$-SVM ($L_{0/1}$ loss): SVM with the $L_{0/1}$ 
soft-margin loss, solved by ADMM \citep{wang2021support}.

\paragraph{Experimental setup and evaluation metrics}
Hyper-parameters of every method are selected by grid search 
over $\{2^i \mid i = -4, -3, \ldots, 4\}$ for all penalty 
terms. The proposed $L_{\mathrm{ht}}$-ADMM starts from 
$\mathbf{w}^0 = \mathbf{1}/100$, $b^0 = 0$, with tolerance 
$\epsilon = 10^{-3}$ and iteration budget $K = 10^3$; the 
P-stationary point condition (Theorem~\ref{thm:sufficient}) 
serves as the stopping rule. To guarantee a fair comparison, 
all baseline solvers are tuned over the same parameter grid.

Three metrics are reported alongside training time: 
(i)~\texttt{ACC}, the test-set classification accuracy; 
(ii)~\texttt{TIME}, the mean wall-clock training time; 
(iii)~\texttt{NSV}, the number of support vectors retained. 
A superior solver should exhibit high \texttt{ACC} together 
with low \texttt{TIME} and \texttt{NSV}.

\subsubsection{Results on real datasets}
\label{subsubsec:sv-real}

The average results of \texttt{ACC}, \texttt{NSV}, and 
\texttt{TIME} of six solvers are recorded in 
Tables~\ref{tab:uci-acc}--\ref{tab:libsvm-time}, where the best 
results are highlighted in bold.

On the UCI datasets, $L_{\mathrm{ht}}$-SVM ranks first in 
average \texttt{ACC}. For example, on breast-cancer it achieves 
97.22\%, whereas all other solvers remain below 97.10\%. It 
also attains the smallest average \texttt{NSV}, indicating that 
the working-set strategy successfully confines the active set 
and thus lowers computational cost. In terms of \texttt{TIME}, 
both $L_{\mathrm{ht}}$-SVM and $L_{0/1}$-SVM are the fastest 
solvers, with a clear margin over the remaining four baselines.

The LIBSVM benchmarks range from 1\,605 to 32\,561 training 
samples. As shown in 
Tables~\ref{tab:libsvm-acc}--\ref{tab:libsvm-time}, 
$L_{\mathrm{ht}}$-SVM and $L_{0/1}$-SVM remain the top 
performers in most cases, and even when they are not the best, 
their results are close to the optimum. Their training-time 
advantage is particularly pronounced on larger datasets, 
confirming the scalability of the working-set ADMM framework.

\begin{table*}[!t]
\centering
\caption{Accuracy rates (\%) with standard deviations for the UCI datasets.}
\label{tab:uci-acc}
\footnotesize
\begin{tabular}{ccccccc}
\toprule
Dataset & Hinge-SVM & LS-SVM & Pin-SVM & Ramp-SVM & $L_{0/1}$-SVM & $L_{\mathrm{ht}}$-SVM \\
\midrule
australian    & 85.51$\pm$3.14 & 86.09$\pm$3.25 & 85.51$\pm$3.14 & 85.51$\pm$3.14 & 85.80$\pm$3.32 & \textbf{86.09$\pm$2.35} \\
breast-cancer & 96.92$\pm$0.97 & 96.05$\pm$1.50 & 96.49$\pm$1.25 & 96.39$\pm$1.03 & 97.07$\pm$1.22 & \textbf{97.22$\pm$0.86} \\
mushroom      & \textbf{100.00$\pm$0.00} & \textbf{100.00$\pm$0.00} & \textbf{100.00$\pm$0.00} & \textbf{100.00$\pm$0.00} & \textbf{100.00$\pm$0.00} & \textbf{100.00$\pm$0.00} \\
wine-red      & 73.98$\pm$1.34 & 74.55$\pm$1.58 & 74.05$\pm$1.44 & 74.16$\pm$1.93 & \textbf{74.73$\pm$1.80} & 74.61$\pm$1.54 \\
wine-white    & 75.11$\pm$0.44 & 75.05$\pm$0.38 & 75.03$\pm$0.72 & 75.05$\pm$0.48 & 75.31$\pm$0.27 & \textbf{75.38$\pm$0.53} \\
\midrule
Avg.~Acc.     & 86.30 & 86.35 & 86.22 & 86.22 & 86.58 & \textbf{86.66} \\
Avg.~Rank     & 4.10 & 3.70 & 4.70 & 4.40 & 2.30 & \textbf{1.80} \\
\bottomrule
\end{tabular}
\end{table*}

\begin{table*}[!t]
\centering
\caption{Number of support vectors (\texttt{NSV}) for the UCI datasets.}
\label{tab:uci-nsv}
\footnotesize
\begin{tabular}{ccccccc}
\toprule
Dataset & Hinge-SVM & LS-SVM & Pin-SVM & Ramp-SVM & $L_{0/1}$-SVM & $L_{\mathrm{ht}}$-SVM \\
\midrule
australian    & 207  & 690   & 384  & 194  & 56   & \textbf{44}  \\
breast-cancer & 259  & 683   & 546  & 175  & \textbf{16}   & 25  \\
mushroom      & 1008 & 8124  & 3555 & 824  & 395  & \textbf{347} \\
wine-red      & 770  & 1599  & 1279 & 616  & 177  & \textbf{153} \\
wine-white    & 2282 & 4898  & 3918 & 1697 & \textbf{136}  & 188 \\
\midrule
Avg.~NSV      & 905  & 3199  & 1936 & 701  & 156  & \textbf{151} \\
Avg.~Rank     & 4.00 & 6.00  & 5.00 & 3.00 & 1.60 & \textbf{1.40} \\
\bottomrule
\end{tabular}
\end{table*}

\begin{table*}[!t]
\centering
\caption{Running time (s) for the UCI datasets.}
\label{tab:uci-time}
\footnotesize
\begin{tabular}{ccccccc}
\toprule
Dataset & Hinge-SVM & LS-SVM & Pin-SVM & Ramp-SVM & $L_{0/1}$-SVM & $L_{\mathrm{ht}}$-SVM \\
\midrule
australian    & 0.024  & 0.017  & 0.015  & 0.613   & 0.006  & \textbf{0.005}  \\
breast-cancer & 0.047  & 0.029  & 0.198  & 7.871   & 0.009  & \textbf{0.004}  \\
mushroom      & 8.266  & 11.387 & 3.692  & 205.892 & \textbf{0.570}  & 0.579  \\
wine-red      & 0.729  & 1.063  & 0.607  & 5.845   & 0.126  & \textbf{0.101}  \\
wine-white    & 0.693  & 1.027  & 0.704  & 28.087  & \textbf{0.139}  & 0.199  \\
\midrule
Avg.~Time     & 1.952  & 2.705  & 1.043  & 49.662  & \textbf{0.172}  & 0.176  \\
Avg.~Rank     & 4.00   & 4.40   & 3.60   & 6.00    & 1.60   & \textbf{1.40}   \\
\bottomrule
\end{tabular}
\end{table*}

\begin{table*}[!t]
\centering
\caption{Accuracy rates (\%) for the LIBSVM datasets.}
\label{tab:libsvm-acc}
\footnotesize
\begin{tabular}{ccccccc}
\toprule
Dataset & Hinge-SVM & LS-SVM & Pin-SVM & Ramp-SVM & $L_{0/1}$-SVM & $L_{\mathrm{ht}}$-SVM \\
\midrule
a1a & 83.74 & 83.99 & 83.51 & 83.60 & 83.89 & \textbf{84.08} \\
a2a & 84.33 & 84.37 & 84.30 & 84.37 & 84.51 & \textbf{84.68} \\
a4a & 84.45 & 84.51 & 84.34 & 84.35 & \textbf{84.63} & 84.54 \\
a6a & 84.74 & 84.58 & 84.24 & 84.41 & 84.82 & \textbf{84.93} \\
a7a & 84.53 & 84.72 & 84.39 & 84.27 & 84.74 & \textbf{84.81} \\
a8a & 85.20 & 85.17 & 84.94 & 84.74 & 85.23 & \textbf{85.40} \\
a9a & 85.03 & 84.55 & 84.42 & 84.39 & \textbf{85.05} & 84.96 \\
w2a & 98.07 & 98.17 & 98.14 & 98.10 & 98.03 & \textbf{98.21} \\
w5a & 98.50 & 98.28 & \textbf{98.51} & 98.32 & \textbf{98.51} & \textbf{98.51} \\
w6a & 98.59 & 98.58 & 98.41 & 97.71 & 98.62 & \textbf{98.67} \\
\midrule
Avg.~Acc.  & 88.72 & 88.69 & 88.52 & 88.43 & 88.80 & \textbf{88.88} \\
Avg.~Rank  & 3.70  & 3.55  & 4.90  & 5.15  & 2.30  & \textbf{1.40}  \\
\bottomrule
\end{tabular}
\end{table*}

\begin{table*}[!t]
\centering
\caption{Number of support vectors (\texttt{NSV}) for the LIBSVM datasets.}
\label{tab:libsvm-nsv}
\footnotesize
\begin{tabular}{ccccccc}
\toprule
Dataset & Hinge-SVM & LS-SVM & Pin-SVM & Ramp-SVM & $L_{0/1}$-SVM & $L_{\mathrm{ht}}$-SVM \\
\midrule
a1a & 632   & 1605  & 1605  & 837   & \textbf{74}    & 110   \\
a2a & 904   & 2265  & 2265  & 697   & \textbf{96}    & 120   \\
a4a & 1762  & 4781  & 4781  & 855   & 227   & \textbf{213}  \\
a6a & 4038  & 11204 & 8388  & 1151  & 448   & \textbf{397}  \\
a7a & 5824  & 16052 & 13626 & 4816  & 1217  & \textbf{918}  \\
a8a & 8315  & 22694 & 14231 & 4112  & 1313  & \textbf{625}  \\
a9a & 11664 & 32552 & 14281 & 632   & \textbf{411}   & 549   \\
w2a & 256   & 3470  & 1386  & 384   & 120   & \textbf{115}  \\
w5a & 472   & 9869  & 2740  & 422   & \textbf{169}   & 208   \\
w6a & 882   & 17186 & 3401  & 723   & 380   & \textbf{230}  \\
\midrule
Avg.~NSV   & 3475  & 12168 & 6670  & 1463  & 446   & \textbf{348}  \\
Avg.~Rank  & 3.80  & 6.00  & 5.00  & 3.20  & 1.60  & \textbf{1.40} \\
\bottomrule
\end{tabular}
\end{table*}

\begin{table*}[!t]
\centering
\caption{Running time (s) for the LIBSVM datasets.}
\label{tab:libsvm-time}
\footnotesize
\begin{tabular}{ccccccc}
\toprule
Dataset & Hinge-SVM & LS-SVM & Pin-SVM & Ramp-SVM & $L_{0/1}$-SVM & $L_{\mathrm{ht}}$-SVM \\
\midrule
a1a & 2.17   & 0.44   & 0.20   & 12.82   & 0.19   & \textbf{0.11}   \\
a2a & 5.78   & 1.32   & \textbf{0.30}   & 89.61   & 0.41   & 0.38   \\
a4a & 1.19   & 2.12   & 3.92   & 388.21  & 1.60   & \textbf{0.61}   \\
a6a & 12.24  & 14.98  & 7.13   & 499.66  & 4.47   & \textbf{3.16}   \\
a7a & 27.21  & 27.17  & 7.28   & 839.18  & 5.61   & \textbf{5.43}   \\
a8a & 44.95  & 75.84  & 12.35  & 1230.23 & 5.24   & \textbf{4.78}   \\
a9a & 75.13  & 220.78 & 22.77  & 1713.84 & \textbf{5.40}   & 7.44   \\
w2a & 2.34   & 1.59   & 4.33   & 218.15  & \textbf{0.86}   & 0.96   \\
w5a & 14.32  & 41.61  & 6.68   & 388.95  & \textbf{1.30}   & 1.43   \\
w6a & 31.66  & 40.07  & 9.10   & 911.92  & 7.08   & \textbf{2.53}   \\
\midrule
Avg.~Time  & 21.70  & 42.59  & 7.41   & 629.26  & 3.22   & \textbf{2.68}   \\
Avg.~Rank  & 4.10   & 4.40   & 3.20   & 6.00    & 1.90   & \textbf{1.40}   \\
\bottomrule
\end{tabular}
\end{table*}

\subsubsection{Results on synthetic datasets}
\label{subsubsec:sv-synthetic}

We conduct an experiment to address the synthetic data with 
$n = 2$ and $m \in \{5000, 10000, 15000, 20000, 25000, 30000\}$. 
The average results and standard deviations are calculated based 
on data gathered from 50 independent runs, as shown in 
Tables~\ref{tab:syn-acc}--\ref{tab:syn-time}. Regarding the 
\texttt{ACC}, we observe that all six solvers are able to 
successfully classify the testing data. In this case, 
$L_{\mathrm{ht}}$-SVM exhibits the highest accuracy in 
classification and the lowest standard deviations, indicating its 
effectiveness. In addition, when considering \texttt{NSV}, it 
becomes clear that the LS-SVM and Pin-SVM utilize entire samples 
as \texttt{NSV}. In contrast, other solvers tend to operate with 
a smaller \texttt{NSV}. Notably, $L_{\mathrm{ht}}$-SVM exhibits 
a significantly smaller \texttt{NSV} compared to other solvers. 
This further confirms the effectiveness of the working set 
strategy in reducing the number of support vectors. Regarding 
the \texttt{TIME}, $L_{\mathrm{ht}}$-SVM and $L_{0/1}$-SVM 
achieve comparable and the fastest execution times, significantly 
outperforming the other four solvers. For instance, when 
$m = 30000$, $L_{\mathrm{ht}}$-SVM only takes 0.031~s while 
Ramp-SVM needs 2886.044~s. To summarize, the 
superiority of $L_{\mathrm{ht}}$-SVM becomes more evident with 
larger $m$. This discovery demonstrates the efficiency and 
scalability of our proposed algorithm for large-scale datasets.

\begin{table*}[!t]
\centering
\caption{Accuracy rates (\%) with standard deviations on synthetic 
datasets ($n = 2$), where the SD represents the standard deviation.}
\label{tab:syn-acc}
\footnotesize
\begin{tabular}{ccccccc}
\toprule
$m$ & Hinge-SVM & LS-SVM & Pin-SVM & Ramp-SVM & $L_{0/1}$-SVM & $L_{\mathrm{ht}}$-SVM \\
\midrule
5000  & 97.96$\pm$0.30 & 97.99$\pm$0.30 & 97.89$\pm$1.92 & 97.81$\pm$0.50 & 98.01$\pm$0.30 & \textbf{98.04$\pm$0.30} \\
10000 & 98.01$\pm$0.17 & 98.05$\pm$0.16 & 98.01$\pm$0.39 & 98.05$\pm$0.33 & \textbf{98.07$\pm$0.18} & \textbf{98.07$\pm$0.18} \\
15000 & 97.99$\pm$0.15 & 97.96$\pm$0.16 & 97.97$\pm$2.19 & 97.75$\pm$0.32 & 98.05$\pm$0.19 & \textbf{98.06$\pm$0.16} \\
20000 & 98.01$\pm$0.13 & 98.01$\pm$0.12 & 97.95$\pm$3.20 & 97.77$\pm$0.31 & 97.99$\pm$0.15 & \textbf{98.05$\pm$0.12} \\
25000 & 97.93$\pm$0.13 & 98.01$\pm$0.13 & 97.98$\pm$9.12 & 97.70$\pm$0.41 & 98.04$\pm$0.13 & \textbf{98.07$\pm$0.13} \\
30000 & 98.06$\pm$0.11 & 98.06$\pm$0.11 & 97.96$\pm$8.01 & 97.92$\pm$0.31 & 98.06$\pm$0.11 & \textbf{98.09$\pm$0.11} \\
\bottomrule
\end{tabular}
\end{table*}

\begin{table*}[!t]
\centering
\caption{Number of support vectors (\texttt{NSV}) $\pm$ SD on 
synthetic datasets.}
\label{tab:syn-nsv}
\footnotesize
\begin{tabular}{ccccccc}
\toprule
$m$ & Hinge-SVM & LS-SVM & Pin-SVM & Ramp-SVM & $L_{0/1}$-SVM & $L_{\mathrm{ht}}$-SVM \\
\midrule
5000  & 432$\pm$12  & 5000$\pm$0  & 2480$\pm$136  & 172$\pm$13  & 20$\pm$19  & \textbf{17$\pm$23}  \\
10000 & 660$\pm$19  & 10000$\pm$0 & 4960$\pm$199  & 441$\pm$27  & 17$\pm$11  & \textbf{12$\pm$9}   \\
15000 & 1085$\pm$28 & 15000$\pm$0 & 7305$\pm$862  & 705$\pm$18  & 72$\pm$6   & \textbf{69$\pm$6}   \\
20000 & 1975$\pm$25 & 20000$\pm$0 & 9684$\pm$969  & 1054$\pm$12 & 140$\pm$41 & \textbf{132$\pm$34} \\
25000 & 2029$\pm$33 & 25000$\pm$0 & 7909$\pm$823  & 1430$\pm$15 & \textbf{137$\pm$64} & 175$\pm$65 \\
30000 & 2755$\pm$43 & 30000$\pm$0 & 10129$\pm$93  & 1972$\pm$25 & 218$\pm$85 & \textbf{212$\pm$55} \\
\bottomrule
\end{tabular}
\end{table*}

\begin{table*}[!t]
\centering
\caption{Running time (s) $\pm$ SD on synthetic datasets.}
\label{tab:syn-time}
\footnotesize
\begin{tabular}{ccccccc}
\toprule
$m$ & Hinge-SVM & LS-SVM & Pin-SVM & Ramp-SVM & $L_{0/1}$-SVM & $L_{\mathrm{ht}}$-SVM \\
\midrule
5000  & 0.105$\pm$0.027 & 1.345$\pm$0.003 & 0.461$\pm$0.082 & 21.019$\pm$21.545  & 0.010$\pm$0.003 & \textbf{0.002$\pm$0.000} \\
10000 & 0.317$\pm$0.054 & 12.151$\pm$0.024 & 1.385$\pm$0.113 & 229.475$\pm$188.209 & 0.017$\pm$0.004 & \textbf{0.011$\pm$0.017} \\
15000 & 0.704$\pm$0.135 & 40.673$\pm$0.072 & 2.351$\pm$0.269 & 494.150$\pm$420.492 & \textbf{0.015$\pm$0.009} & 0.025$\pm$0.081 \\
20000 & 2.021$\pm$0.327 & 94.595$\pm$0.311 & 4.295$\pm$0.385 & 1215.696$\pm$819.898 & 0.031$\pm$0.009 & \textbf{0.010$\pm$0.001} \\
25000 & 2.219$\pm$0.321 & 181.278$\pm$0.370 & 4.845$\pm$0.415 & 1861.817$\pm$1316.984 & 0.032$\pm$0.005 & \textbf{0.014$\pm$0.001} \\
30000 & 2.544$\pm$0.270 & 420.441$\pm$2.672 & 5.044$\pm$0.497 & 2886.044$\pm$1533.268 & \textbf{0.012$\pm$0.003} & 0.031$\pm$0.005 \\
\bottomrule
\end{tabular}
\end{table*}

\subsubsection{Performance analysis of \texorpdfstring{$L_{\mathrm{ht}}$}{Lht}-ADMM}
\label{subsubsec:sv-performance}

We illustrate the convergence behaviour of $L_{\mathrm{ht}}$-ADMM 
on dataset w5a ($m = 9\,888$) in Fig.~\ref{fig:sv-convergence}; 
analogous patterns are observed on other datasets and omitted 
for brevity.

Fig.~\ref{fig:sv-convergence}(a) shows that the algorithm 
reaches a P-stationary point within roughly 20 iterations, 
confirming global convergence. 
Fig.~\ref{fig:sv-convergence}(b) indicates that the active set 
$F_k$ stabilises at a small cardinality as $k$ grows: the ratio 
$|F_k|/m$ drops below 5\% in several iterations, which 
validates the per-iteration complexity reduction achieved by 
the working-set rule. 
Fig.~\ref{fig:sv-convergence}(c) reveals that the training 
accuracy (\texttt{TACC}) climbs steadily during the first 17 
iterations and plateaus above 98.3\%, demonstrating both rapid 
convergence and reliable generalisation on large-scale data. 
Finally, Fig.~\ref{fig:sv-convergence}(d) depicts the per-iteration 
wall-clock time, which remains low throughout. This efficiency 
stems from the closed-form $\mathbf{w}$-subproblem solution, 
the compact working set, and the P-stationary point termination 
criterion.

\begin{figure*}[!t]
\centering
\includegraphics[width=0.85\textwidth]{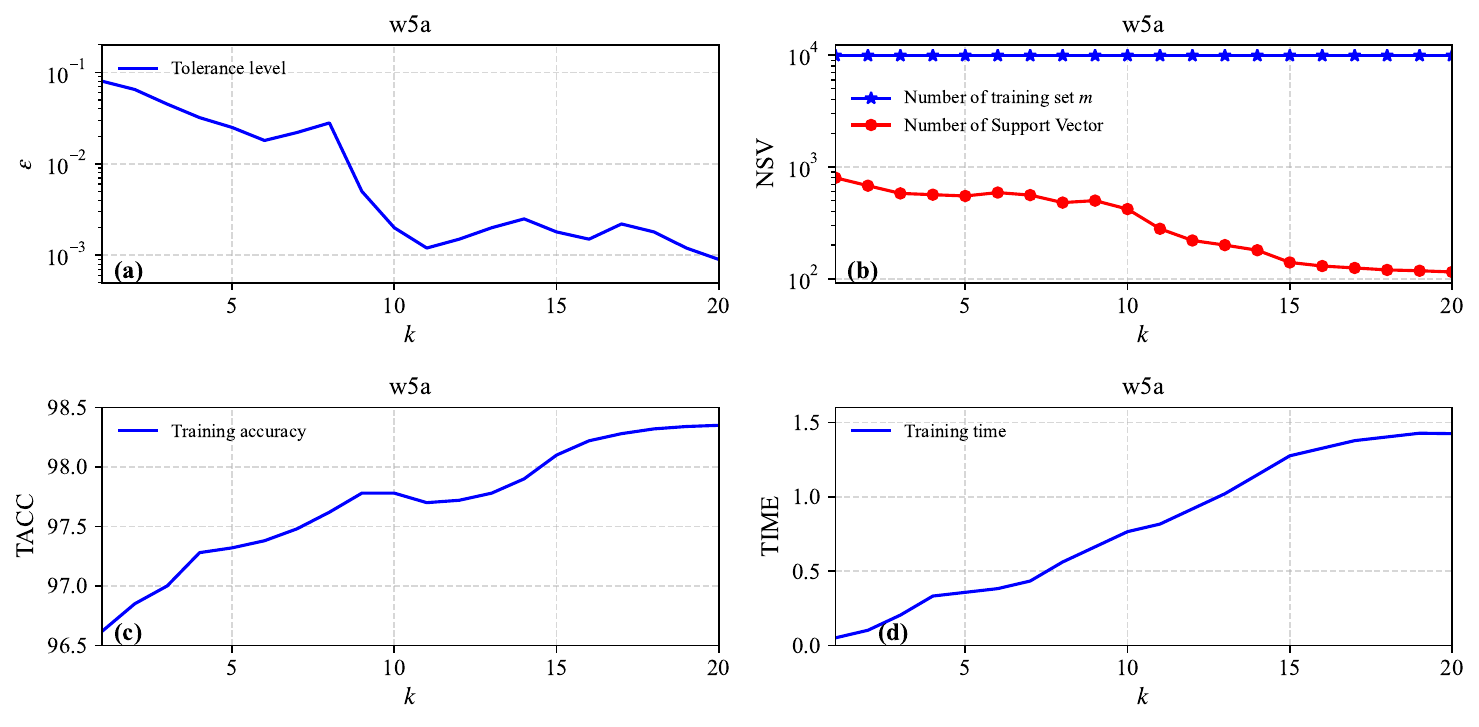}
\caption{The influence of $L_{\mathrm{ht}}$-ADMM on dataset w5a.}
\label{fig:sv-convergence}
\end{figure*}

\subsubsection{Robustness to label noise}
\label{subsubsec:sv-robustness}

We probe robustness by injecting label noise into the 
breast-cancer and australian datasets. For each class, a 
fraction $r \in \{0.00, 0.02, 0.04, 0.06, 0.08, 0.10\}$ of 
labels is randomly flipped; results are averaged over 5-fold 
cross-validation and reported in 
Tables~\ref{tab:robust-breast}--\ref{tab:robust-aust}.

As the noise rate rises, the accuracy of all solvers 
degrades, yet $L_{\mathrm{ht}}$-SVM maintains the highest 
\texttt{ACC} throughout, confirming its robustness advantage. 
In terms of \texttt{NSV}, $L_{\mathrm{ht}}$-SVM and 
$L_{0/1}$-SVM stabilise at a low support-vector count even 
under heavy corruption, whereas LS-SVM and Pin-SVM eventually 
retain nearly all training samples as support vectors. 
Hinge-SVM also exhibits a noticeable increase in \texttt{NSV}. 
These observations corroborate that the bounded 
$L_{\mathrm{ht}}$ loss effectively limits the influence of 
mislabelled samples. 
As for \texttt{TIME}, $L_{\mathrm{ht}}$-SVM and $L_{0/1}$-SVM 
consistently achieve the fastest execution times, significantly 
outperforming the other four solvers, indicating their 
computational efficiency. It is worth noting that, although 
$L_{0/1}$-SVM often yields the smallest \texttt{NSV}, the smallest 
support-vector count does not necessarily imply the best 
\texttt{ACC} or \texttt{TIME}. The running time is also affected by 
the numerical behaviour of the optimization procedure and the number 
of iterations required for convergence, while the classification 
accuracy depends on whether the retained active samples still 
preserve sufficient boundary information under label corruption. In 
contrast, $L_{\mathrm{ht}}$-SVM typically keeps a similarly sparse 
yet more informative active set, thereby achieving a better balance 
between sparsity, robustness, and computational efficiency. In 
summary, the experimental results 
demonstrate that $L_{\mathrm{ht}}$-SVM has better computational 
efficiency and robustness against outliers compared to the other 
solvers.

\begin{table*}[!t]
\centering
\caption{Robustness comparison on the breast-cancer dataset 
under different label noise ratios.}
\label{tab:robust-breast}
\footnotesize
\begin{tabular}{ccccccc}
\toprule
\multicolumn{7}{c}{\texttt{ACC} (\%)} \\
\midrule
$r$ & Hinge-SVM & LS-SVM & Pin-SVM & Ramp-SVM & $L_{0/1}$-SVM & $L_{\mathrm{ht}}$-SVM \\
\midrule
0.00 & 96.92 & 96.05 & 96.49 & 96.39 & 97.07 & \textbf{97.22} \\
0.02 & 95.17 & 94.43 & 95.02 & 95.46 & 95.46 & \textbf{95.75} \\
0.04 & 92.98 & 92.54 & 92.83 & 93.42 & \textbf{94.00} & 93.71 \\
0.06 & 91.66 & 90.78 & 91.22 & 91.51 & 92.09 & \textbf{92.24} \\
0.08 & 89.16 & 88.72 & 89.31 & 89.45 & \textbf{89.75} & \textbf{89.75} \\
0.10 & 87.26 & 87.26 & 87.26 & 87.85 & 88.00 & \textbf{88.29} \\
\midrule
\multicolumn{7}{c}{\texttt{NSV}} \\
\midrule
$r$ & Hinge-SVM & LS-SVM & Pin-SVM & Ramp-SVM & $L_{0/1}$-SVM & $L_{\mathrm{ht}}$-SVM \\
\midrule
0.00 & 259  & 683 & 546 & 175  & \textbf{16}  & 25  \\
0.02 & 276  & 683 & 546 & 196  & \textbf{17}  & 19  \\
0.04 & 207 & 683 & 546 & 196 & \textbf{17}  & 25  \\
0.06 & 258  & 683 & 546 & 182  & \textbf{17}  & 18  \\
0.08 & 234  & 683 & 546 & 124  & \textbf{17}  & \textbf{17}  \\
0.10 & 265  & 683 & 546 & 149  & 18  & \textbf{16}  \\
\midrule
\multicolumn{7}{c}{\texttt{TIME} (seconds)} \\
\midrule
$r$ & Hinge-SVM & LS-SVM & Pin-SVM & Ramp-SVM & $L_{0/1}$-SVM & $L_{\mathrm{ht}}$-SVM \\
\midrule
0.00 & 0.047 & 0.029 & 0.198 & 7.871 & 0.009 & \textbf{0.004} \\
0.02 & 0.051 & 0.034 & 0.221 & 7.293 & 0.010 & \textbf{0.009} \\
0.04 & 0.049 & 0.031 & 0.198 & 7.716 & 0.007 & \textbf{0.006} \\
0.06 & 0.049 & 0.034 & 0.210 & 7.921 & 0.010 & \textbf{0.006} \\
0.08 & 0.049 & 0.026 & 0.206 & 7.854 & 0.009 & \textbf{0.005} \\
0.10 & 0.049 & 0.031 & 0.201 & 7.270 & 0.009 & \textbf{0.008} \\
\bottomrule
\end{tabular}
\end{table*}

\begin{table*}[!t]
\centering
\caption{Robustness comparison on the australian dataset 
under different label noise ratios.}
\label{tab:robust-aust}
\footnotesize
\begin{tabular}{ccccccc}
\toprule
\multicolumn{7}{c}{\texttt{ACC} (\%)} \\
\midrule
$r$ & Hinge-SVM & LS-SVM & Pin-SVM & Ramp-SVM & $L_{0/1}$-SVM & $L_{\mathrm{ht}}$-SVM \\
\midrule
0.00 & 85.51 & \textbf{86.09} & 85.51 & 85.51 & 85.80 & \textbf{86.09} \\
0.02 & 84.78 & \textbf{85.36} & 84.78 & 84.78 & 84.35 & 84.93 \\
0.04 & 83.33 & 83.62 & 83.33 & 83.33 & 83.62 & \textbf{83.91} \\
0.06 & 82.03 & 82.09 & 82.03 & 82.03 & 81.30 & \textbf{82.17} \\
0.08 & 80.00 & 80.15 & 80.00 & 80.00 & 79.57 & \textbf{80.29} \\
0.10 & 78.55 & 78.64 & 78.55 & 78.55 & 77.68 & \textbf{78.84} \\
\midrule
\multicolumn{7}{c}{\texttt{NSV}} \\
\midrule
$r$ & Hinge-SVM & LS-SVM & Pin-SVM & Ramp-SVM & $L_{0/1}$-SVM & $L_{\mathrm{ht}}$-SVM \\
\midrule
0.00 & 207 & 690 & 384 & 194 & 56  & \textbf{44} \\
0.02 & 220 & 690 & 384 & 172 & 52 & \textbf{49} \\
0.04 & 189 & 690 & 384 & 201 & 29  & \textbf{18}  \\
0.06 & 216 & 690 & 384 & 190 & 84  & \textbf{36} \\
0.08 & 209 & 690 & 384 & 172 & \textbf{21}  & 38  \\
0.10 & 217 & 690 & 384 & 184 & 34  & \textbf{23}  \\
\midrule
\multicolumn{7}{c}{\texttt{TIME} (seconds)} \\
\midrule
$r$ & Hinge-SVM & LS-SVM & Pin-SVM & Ramp-SVM & $L_{0/1}$-SVM & $L_{\mathrm{ht}}$-SVM \\
\midrule
0.00 & 0.024 & 0.017 & 0.015 & 0.613 & 0.006 & \textbf{0.005} \\
0.02 & 0.033 & 0.024 & 0.067 & 0.557 & \textbf{0.009} & \textbf{0.009} \\
0.04 & 0.012 & 0.024 & 0.021 & 0.616 & 0.005 & \textbf{0.004} \\
0.06 & 0.029 & 0.034 & 0.016 & 0.527 & \textbf{0.007} & 0.009 \\
0.08 & 0.015 & 0.019 & 0.027 & 0.592 & 0.011 & \textbf{0.006} \\
0.10 & 0.026 & 0.023 & 0.017 & 0.609 & 0.009 & \textbf{0.005} \\
\bottomrule
\end{tabular}
\end{table*}

\subsection{Multi-view experiments}
\label{subsec:mv-exp}

\subsubsection{Experimental settings}
\label{subsubsec:mv-settings}

This subsection compares Mv$L_{\mathrm{ht}}$-SVM with six 
multi-view SVM baselines on two families of datasets.

\textbf{Datasets.} We use 10 UCI datasets and 45 binary tasks 
derived from 
STL-10\footnote{\url{https://cs.stanford.edu/~acoates/stl10/}}.

\textbf{(1) UCI datasets:} Table~\ref{tab:mv-datasets} summarises 
the 10 UCI datasets employed. Because these datasets are originally 
single-view, view~1 consists of the raw features and view~2 
comprises the leading principal components that retain 95\% of 
the variance.

\textbf{(2) STL-10:} STL-10 contains $96 \times 96$-pixel images 
from ten object categories (see Fig.~\ref{fig:stl10-samples}). We 
extract HOG and LBP descriptors as two views and randomly sample 
500 images per class. All $\binom{10}{2} = 45$ one-vs-one binary 
tasks are constructed to test generalisation across diverse class 
pairs.

\begin{table}[!htbp]
\centering
\caption{Details of 10 UCI datasets.}
\label{tab:mv-datasets}
\footnotesize
\begin{tabular}{cccc}
\toprule
Datasets & \#Data ($m$) & \#View~1 $n^{(1)}$ & \#View~2 $n^{(2)}$ \\
\midrule
Balance    & 576  & 4  & 4  \\
Bankruptcy & 250  & 6  & 5  \\
BUPA       & 345  & 6  & 5  \\
German     & 1000 & 24 & 18 \\
Haberman   & 306  & 3  & 3  \\
Heart      & 270  & 13 & 10 \\
Hepatitis  & 155  & 19 & 15 \\
Iris       & 150  & 4  & 2  \\
Vote       & 435  & 16 & 13 \\
Wine       & 130  & 13 & 1  \\
\bottomrule
\end{tabular}
\end{table}

\begin{figure}[!htbp]
\centering
\includegraphics[width=0.5\textwidth]{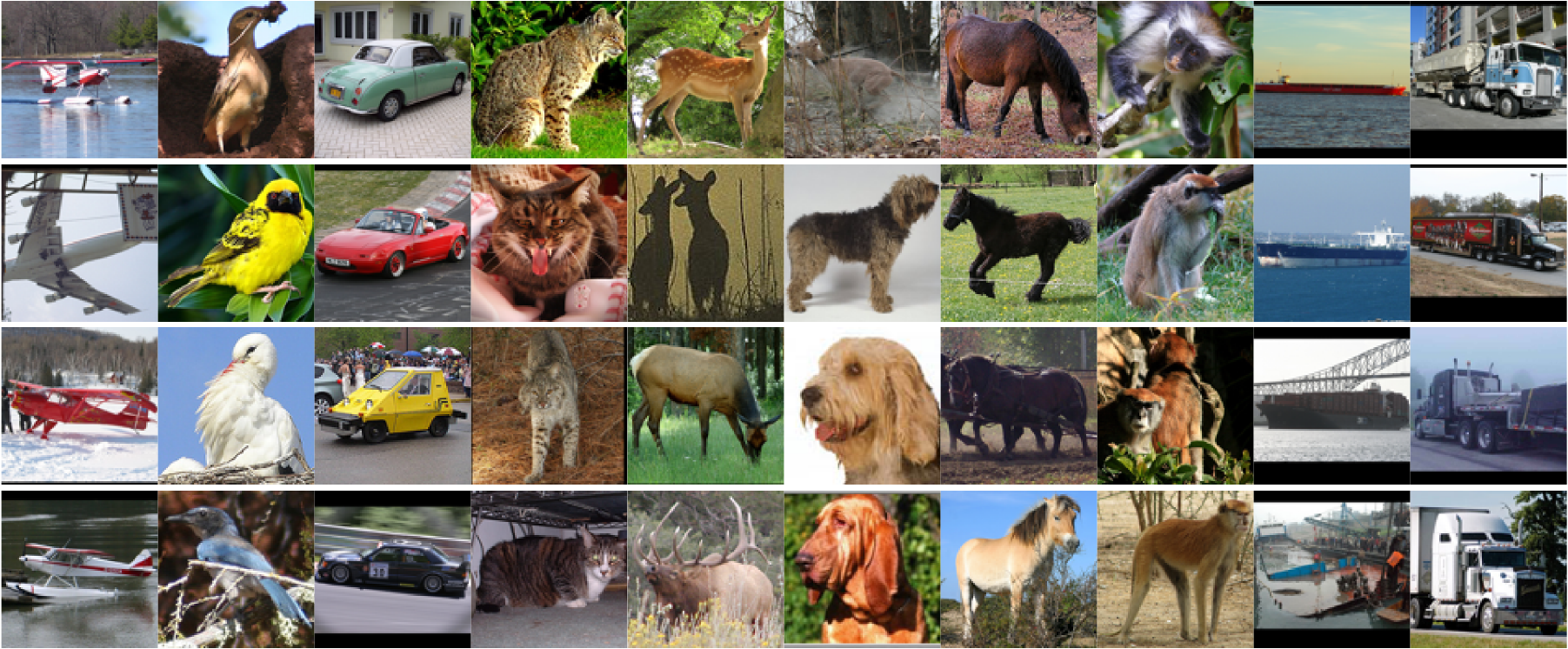}
\caption{STL-10 dataset.}
\label{fig:stl10-samples}
\end{figure}

\textbf{Benchmark methods.} Six multi-view SVM baselines are 
considered:
(1)~SVM-2K (hinge loss): enforces inter-view consistency via 
KCCA \citep{farquhar2005two}.
(2)~MvSVM-2C (hinge loss): jointly captures consensus and 
complementarity through adaptive view weighting 
\citep{xie2019multi}.
(3)~MvLSSVC-2C (least squares loss): couples multiple LS-SVM 
classifiers with a view-interaction term 
\citep{zhang2025multi}.
(4)~Wave-MvSVM (wave loss): applies a bounded wave loss for 
robustness under the 2C framework \citep{quadir2025enhancing}.
(5)~MVASY-BX (bounded LINEX loss): employs a bounded asymmetric 
loss to resist outliers \citep{tang2023robust}.
(6)~Mv$SL_{0/1}$-SVM ($L_{0/1}$ loss): adopts the 0/1 
soft-margin loss in multi-view classification 
\citep{chen2025multi}.

\textbf{Evaluation metrics.} We measure accuracy, precision, 
recall, and F1-score, all estimated via 5-fold cross-validation. 
Hyper-parameters are tuned by grid search over 
$\{2^i \mid i = -4, -2, 0, 2, 4\}$. On the larger STL-10 
tasks the search grid is narrowed to 
$\{2^i \mid i = -2, 0, 2\}$ to keep the total cost manageable.

\subsubsection{Results on UCI multi-view datasets}
\label{subsubsec:mv-uci}

Tables~\ref{tab:mv-uci-acc}--\ref{tab:mv-uci-time} show that 
Mv$L_{\mathrm{ht}}$-SVM achieves the best or tied-best accuracy 
on 8 of the 10 UCI datasets and attains the highest average rank 
across accuracy, precision, and F1-score. Although it does not 
top every dataset individually, its consistently strong 
performance across all four metrics underscores its robustness 
and generality. MvLSSVC-2C is the fastest solver on these 
small-to-medium datasets, yet 
Mv$L_{\mathrm{ht}}$-SVM (1.48\,s average) is the second fastest, 
considerably outpacing MvSVM-2C (38.95\,s), Wave-MvSVM 
(11.64\,s), MVASY-BX (5.03\,s), SVM-2K (3.03\,s), and 
Mv$SL_{0/1}$-SVM (1.82\,s). This speed advantage originates 
from the working-set ADMM that restricts each iteration to a 
compact active subset. Furthermore, methods that neglect one 
of the two 2C principles tend to exhibit weaker generalisation, 
confirming the benefit of jointly enforcing consensus and 
complementarity.

\begin{table*}[!htbp]
\centering
\caption{Accuracy rates (\%) with standard deviations for the UCI multi-view datasets.}
\label{tab:mv-uci-acc}
\footnotesize
\begin{tabular}{cccccccc}
\toprule
Dataset & SVM-2K & MvSVM-2C & MvLSSVC-2C & Wave-MvSVM & MVASY-BX & Mv$SL_{0/1}$-SVM & Mv$L_{\mathrm{ht}}$-SVM \\
\midrule
Balance    & 93.75$\pm$2.01 & 94.97$\pm$1.77 & 94.62$\pm$1.85 & 92.19$\pm$2.63 & 95.66$\pm$1.45 & 96.18$\pm$2.10 & \textbf{96.52$\pm$2.13} \\
Bankruptcy & 95.20$\pm$3.49 & 98.80$\pm$1.60 & 99.20$\pm$1.60 & \textbf{100.00$\pm$0.00} & 97.60$\pm$2.33 & \textbf{100.00$\pm$0.00} & \textbf{100.00$\pm$0.00} \\
BUPA       & 72.75$\pm$4.43 & 62.61$\pm$4.24 & 73.04$\pm$4.45 & 62.03$\pm$3.93 & 69.86$\pm$2.96 & \textbf{74.78$\pm$3.73} & 74.20$\pm$2.32 \\
German     & 70.80$\pm$0.75 & 72.20$\pm$2.42 & \textbf{75.00$\pm$1.38} & 72.80$\pm$1.54 & 70.00$\pm$0.00 & 74.50$\pm$1.55 & 74.50$\pm$1.14 \\
Haberman   & 73.53$\pm$0.48 & 73.53$\pm$0.48 & 73.20$\pm$1.26 & 73.20$\pm$0.89 & 73.53$\pm$0.48 & 74.18$\pm$0.81 & \textbf{74.19$\pm$1.09} \\
Heart      & 82.96$\pm$3.39 & 85.56$\pm$3.19 & 84.81$\pm$3.19 & 75.19$\pm$4.77 & 85.19$\pm$4.22 & 85.56$\pm$5.90 & \textbf{85.93$\pm$3.43} \\
Hepatitis  & 92.90$\pm$3.16 & 94.19$\pm$3.76 & 94.84$\pm$2.58 & 96.77$\pm$2.04 & 81.94$\pm$3.29 & 97.42$\pm$2.41 & \textbf{98.06$\pm$1.58} \\
Iris       & \textbf{100.00$\pm$0.00} & \textbf{100.00$\pm$0.00} & \textbf{100.00$\pm$0.00} & \textbf{100.00$\pm$0.00} & \textbf{100.00$\pm$0.00} & \textbf{100.00$\pm$0.00} & \textbf{100.00$\pm$0.00} \\
Vote       & 93.56$\pm$4.15 & 95.63$\pm$3.12 & 95.63$\pm$2.56 & 93.33$\pm$2.23 & 92.18$\pm$3.87 & \textbf{96.09$\pm$2.37} & \textbf{96.09$\pm$2.37} \\
Wine       & 91.54$\pm$2.88 & 93.08$\pm$2.88 & 96.92$\pm$2.88 & 90.00$\pm$3.08 & 95.38$\pm$2.88 & \textbf{100.00$\pm$0.00} & \textbf{100.00$\pm$0.00} \\
\midrule
Avg.~Acc.  & 86.70 & 87.06 & 88.73 & 85.55 & 86.13 & 89.87 & \textbf{89.95} \\
Avg.~Rank  & 5.40  & 4.40  & 3.90  & 5.35  & 5.10  & 2.10  & \textbf{1.75}  \\
\bottomrule
\end{tabular}
\end{table*}

\begin{table*}[!htbp]
\centering
\caption{Precision rates (\%) with standard deviations for the UCI multi-view datasets.}
\label{tab:mv-uci-prec}
\footnotesize
\begin{tabular}{cccccccc}
\toprule
Dataset & SVM-2K & MvSVM-2C & MvLSSVC-2C & Wave-MvSVM & MVASY-BX & Mv$SL_{0/1}$-SVM & Mv$L_{\mathrm{ht}}$-SVM \\
\midrule
Balance    & 94.22$\pm$3.48 & 95.43$\pm$4.28 & 94.00$\pm$3.33 & 91.49$\pm$2.93 & 95.75$\pm$3.78 & 95.30$\pm$2.83 & \textbf{97.20$\pm$2.35} \\
Bankruptcy & 99.00$\pm$2.00 & 98.14$\pm$2.28 & 99.05$\pm$1.90 & \textbf{100.00$\pm$0.00} & 97.14$\pm$2.34 & \textbf{100.00$\pm$0.00} & \textbf{100.00$\pm$0.00} \\
BUPA       & 69.87$\pm$7.31 & 53.54$\pm$3.74 & 70.79$\pm$7.07 & 54.41$\pm$4.09 & 70.07$\pm$5.12 & 70.01$\pm$5.51 & \textbf{70.93$\pm$3.12} \\
German     & 71.04$\pm$0.70 & 72.65$\pm$2.72 & 77.84$\pm$0.73 & 78.69$\pm$0.77 & 70.00$\pm$0.00 & \textbf{78.76$\pm$1.31} & 77.48$\pm$0.95 \\
Haberman   & 73.53$\pm$0.48 & 73.53$\pm$0.48 & 73.60$\pm$0.80 & 73.60$\pm$0.21 & 73.53$\pm$0.48 & \textbf{75.12$\pm$1.97} & 74.02$\pm$0.91 \\
Heart      & 84.07$\pm$4.34 & \textbf{86.57$\pm$4.57} & 85.58$\pm$5.20 & 80.20$\pm$6.13 & 85.24$\pm$3.36 & 85.73$\pm$7.55 & 85.59$\pm$3.25 \\
Hepatitis  & 85.48$\pm$13.31 & 91.67$\pm$10.54 & 89.64$\pm$9.48 & 95.56$\pm$8.89 & 40.00$\pm$48.99 & 94.29$\pm$7.00 & \textbf{100.00$\pm$0.00} \\
Iris       & \textbf{100.00$\pm$0.00} & \textbf{100.00$\pm$0.00} & \textbf{100.00$\pm$0.00} & \textbf{100.00$\pm$0.00} & \textbf{100.00$\pm$0.00} & \textbf{100.00$\pm$0.00} & \textbf{100.00$\pm$0.00} \\
Vote       & 91.04$\pm$6.97 & \textbf{93.89$\pm$5.56} & 92.38$\pm$4.86 & 92.79$\pm$6.06 & 89.10$\pm$7.78 & 92.47$\pm$4.85 & 93.01$\pm$5.16 \\
Wine       & 92.03$\pm$6.90 & 91.97$\pm$4.60 & 98.46$\pm$3.08 & 88.83$\pm$6.29 & 98.33$\pm$3.33 & \textbf{100.00$\pm$0.00} & \textbf{100.00$\pm$0.00} \\
\midrule
Avg.~Prec. & 86.03 & 85.74 & 88.13 & 85.56 & 81.92 & 89.17 & \textbf{89.82} \\
Avg.~Rank  & 5.40  & 4.30  & 3.95  & 4.35  & 5.20  & 2.65  & \textbf{2.15}  \\
\bottomrule
\end{tabular}
\end{table*}

\begin{table*}[!htbp]
\centering
\caption{Recall rates (\%) with standard deviations for the UCI multi-view datasets.}
\label{tab:mv-uci-rec}
\footnotesize
\begin{tabular}{cccccccc}
\toprule
Dataset & SVM-2K & MvSVM-2C & MvLSSVC-2C & Wave-MvSVM & MVASY-BX & Mv$SL_{0/1}$-SVM & Mv$L_{\mathrm{ht}}$-SVM \\
\midrule
Balance    & 93.39$\pm$3.38 & 94.77$\pm$3.52 & 95.50$\pm$2.81 & 93.04$\pm$2.71 & 95.84$\pm$3.54 & \textbf{97.22$\pm$2.09} & 95.84$\pm$2.58 \\
Bankruptcy & 89.78$\pm$7.88 & 99.05$\pm$1.90 & 99.05$\pm$1.90 & \textbf{100.00$\pm$0.00} & 97.14$\pm$3.81 & \textbf{100.00$\pm$0.00} & \textbf{100.00$\pm$0.00} \\
BUPA       & 63.45$\pm$4.14 & \textbf{92.41$\pm$6.69} & 62.07$\pm$3.78 & 60.00$\pm$10.37 & 49.66$\pm$6.01 & 71.03$\pm$7.10 & 65.52$\pm$3.08 \\
German     & 98.43$\pm$1.23 & 97.14$\pm$2.47 & 89.86$\pm$1.46 & 83.86$\pm$2.62 & \textbf{100.00$\pm$0.00} & 87.14$\pm$3.50 & 89.71$\pm$3.82 \\
Haberman   & \textbf{100.00$\pm$0.00} & \textbf{100.00$\pm$0.00} & 99.11$\pm$1.09 & 99.11$\pm$1.09 & \textbf{100.00$\pm$0.00} & 97.33$\pm$3.56 & \textbf{100.00$\pm$0.00} \\
Heart      & 86.00$\pm$5.33 & 88.00$\pm$4.00 & 88.00$\pm$2.67 & 74.00$\pm$6.46 & 88.67$\pm$4.52 & \textbf{90.00$\pm$4.71} & \textbf{90.00$\pm$5.96} \\
Hepatitis  & 84.29$\pm$10.61 & 80.48$\pm$16.45 & 87.14$\pm$12.38 & 90.48$\pm$7.82 & 11.90$\pm$16.77 & \textbf{93.81$\pm$7.62} & 90.48$\pm$7.82 \\
Iris       & \textbf{100.00$\pm$0.00} & \textbf{100.00$\pm$0.00} & \textbf{100.00$\pm$0.00} & \textbf{100.00$\pm$0.00} & \textbf{100.00$\pm$0.00} & \textbf{100.00$\pm$0.00} & \textbf{100.00$\pm$0.00} \\
Vote       & 92.82$\pm$4.91 & 95.22$\pm$4.89 & 97.01$\pm$3.81 & 90.48$\pm$6.29 & 91.62$\pm$4.06 & \textbf{98.20$\pm$2.42} & 97.61$\pm$2.93 \\
Wine       & 89.85$\pm$3.27 & 93.18$\pm$6.28 & 95.00$\pm$6.67 & 89.70$\pm$8.43 & 91.52$\pm$5.28 & \textbf{100.00$\pm$0.00} & \textbf{100.00$\pm$0.00} \\
\midrule
Avg.~Rec.  & 89.80 & 94.03 & 91.27 & 88.07 & 82.64 & \textbf{93.47} & 92.92 \\
Avg.~Rank  & 4.75  & 3.85  & 4.15  & 5.50  & 4.40  & \textbf{2.65}  & 2.70  \\
\bottomrule
\end{tabular}
\end{table*}

\begin{table*}[!htbp]
\centering
\caption{F1-score rates (\%) with standard deviations for the UCI multi-view datasets.}
\label{tab:mv-uci-f1}
\footnotesize
\begin{tabular}{cccccccc}
\toprule
Dataset & SVM-2K & MvSVM-2C & MvLSSVC-2C & Wave-MvSVM & MVASY-BX & Mv$SL_{0/1}$-SVM & Mv$L_{\mathrm{ht}}$-SVM \\
\midrule
Balance    & 93.72$\pm$2.05 & 94.97$\pm$1.75 & 94.68$\pm$1.81 & 92.25$\pm$2.60 & 95.68$\pm$1.44 & 96.23$\pm$2.07 & \textbf{96.50$\pm$2.15} \\
Bankruptcy & 93.97$\pm$4.49 & 98.58$\pm$1.90 & 99.05$\pm$1.90 & \textbf{100.00$\pm$0.00} & 97.12$\pm$2.83 & \textbf{100.00$\pm$0.00} & \textbf{100.00$\pm$0.00} \\
BUPA       & 66.29$\pm$4.38 & 67.56$\pm$2.40 & 66.03$\pm$4.75 & 56.70$\pm$6.01 & 57.91$\pm$5.25 & \textbf{70.23$\pm$4.44} & 68.10$\pm$2.86 \\
German     & 82.52$\pm$0.39 & 83.05$\pm$1.06 & \textbf{83.42$\pm$0.97} & 81.17$\pm$1.31 & 82.35$\pm$0.00 & 82.69$\pm$1.35 & 83.09$\pm$1.18 \\
Haberman   & 84.75$\pm$0.32 & 84.75$\pm$0.32 & 84.47$\pm$0.74 & 84.47$\pm$0.54 & 84.75$\pm$0.32 & 84.71$\pm$0.25 & \textbf{85.07$\pm$0.60} \\
Heart      & 84.85$\pm$3.09 & 87.15$\pm$2.70 & 86.63$\pm$2.34 & 76.76$\pm$4.61 & 86.90$\pm$3.82 & 87.54$\pm$4.47 & \textbf{87.60$\pm$3.23} \\
Hepatitis  & 83.47$\pm$5.46 & 84.18$\pm$11.14 & 87.31$\pm$6.02 & 92.33$\pm$4.15 & 17.71$\pm$23.86 & 93.79$\pm$5.53 & \textbf{94.83$\pm$4.26} \\
Iris       & \textbf{100.00$\pm$0.00} & \textbf{100.00$\pm$0.00} & \textbf{100.00$\pm$0.00} & \textbf{100.00$\pm$0.00} & \textbf{100.00$\pm$0.00} & \textbf{100.00$\pm$0.00} & \textbf{100.00$\pm$0.00} \\
Vote       & 91.82$\pm$5.22 & 94.42$\pm$3.94 & 94.54$\pm$3.22 & 91.29$\pm$2.98 & 90.17$\pm$4.78 & \textbf{95.17$\pm$2.88} & 95.15$\pm$2.88 \\
Wine       & 90.69$\pm$2.80 & 92.37$\pm$3.33 & 96.51$\pm$3.37 & 88.86$\pm$4.03 & 94.69$\pm$3.25 & \textbf{100.00$\pm$0.00} & \textbf{100.00$\pm$0.00} \\
\midrule
Avg.~F1    & 87.21 & 88.70 & 89.26 & 86.38 & 80.73 & \textbf{91.04} & 91.03 \\
Avg.~Rank  & 5.20  & 3.90  & 4.05  & 5.65  & 5.00  & 2.45  & \textbf{1.75}  \\
\bottomrule
\end{tabular}
\end{table*}

\begin{table*}[!htbp]
\centering
\caption{Running time (s) for the UCI multi-view datasets.}
\label{tab:mv-uci-time}
\footnotesize
\begin{tabular}{cccccccc}
\toprule
Dataset & SVM-2K & MvSVM-2C & MvLSSVC-2C & Wave-MvSVM & MVASY-BX & Mv$SL_{0/1}$-SVM & Mv$L_{\mathrm{ht}}$-SVM \\
\midrule
Balance    & 7.22   & 53.32   & \textbf{0.54}  & 15.54  & 3.11   & 2.79   & 2.13   \\
Bankruptcy & 0.36   & 3.65    & \textbf{0.08}  & 7.89   & 2.80   & 1.09   & 0.69   \\
BUPA       & 2.04   & 7.53    & \textbf{0.11}  & 10.06  & 3.52   & 2.36   & 1.08   \\
German     & 16.70  & 222.64  & \textbf{0.30}  & 37.18  & 18.18  & 4.00   & 3.89   \\
Haberman   & 0.54   & 2.10    & \textbf{0.05}  & 8.77   & 2.21   & 2.18   & 1.97   \\
Heart      & 0.47   & 6.26    & \textbf{0.22}  & 6.42   & 3.82   & 1.25   & 1.42   \\
Hepatitis  & 0.22   & 1.41    & \textbf{0.02}  & 7.03   & 4.09   & 0.79   & 0.50   \\
Iris       & 0.23   & 1.91    & \textbf{0.04}  & 6.04   & 1.34   & 0.48   & 0.39   \\
Vote       & 2.45   & 88.65   & \textbf{0.16}  & 11.77  & 9.54   & 2.32   & 1.79   \\
Wine       & 0.11   & 2.02    & \textbf{0.11}  & 5.71   & 1.64   & 0.90   & 0.96   \\
\midrule
Avg.~Time  & 3.03   & 38.95   & \textbf{0.16}  & 11.64  & 5.03   & 1.82   & 1.48   \\
Avg.~Rank  & 2.75   & 6.00    & \textbf{1.05}  & 6.70   & 5.10   & 3.60   & 2.80   \\
\bottomrule
\end{tabular}
\end{table*}

\subsubsection{Results on STL-10 datasets}
\label{subsubsec:mv-stl10}

Tables~\ref{tab:stl10-acc}--\ref{tab:stl10-f1} report the 
results across all 45 STL-10 binary tasks. 
Mv$L_{\mathrm{ht}}$-SVM leads with an average accuracy of 
88.60\%, followed by MvLSSVC-2C (87.88\%), MvSVM-2C (86.70\%), 
Mv$SL_{0/1}$-SVM (86.27\%), Wave-MvSVM (86.18\%), MVASY-BX 
(86.17\%), and SVM-2K (85.33\%). SVM-2K, which enforces only 
the consensus constraint, consistently trails the 2C-complete 
models, underscoring the value of complementarity. On 
particularly challenging pairs such as cat vs.\ dog and 
dog vs.\ monkey, Mv$L_{\mathrm{ht}}$-SVM retains a distinct 
accuracy edge, suggesting that the bounded $L_{\mathrm{ht}}$ 
loss is especially beneficial when class boundaries are 
ambiguous.

Mv$L_{\mathrm{ht}}$-SVM also ranks first in average precision 
(89.18\%), recall (87.84\%), and F1-score (88.29\%), 
demonstrating well-rounded performance on this large-scale image 
classification benchmark.

\begin{table*}[!htbp]
\centering
\caption{Accuracy rates (\%) with standard deviations for the STL-10 dataset.}
\label{tab:stl10-acc}
{\scriptsize
\begin{tabular}{cccccccc}
\toprule
Datasets & SVM-2K & MvSVM-2C & MvLSSVC-2C & Wave-MvSVM & MVASY-BX & Mv$SL_{0/1}$-SVM & Mv$L_{\mathrm{ht}}$-SVM \\
\midrule
air.\ vs bird.       & 84.40$\pm$2.50 & 89.80$\pm$1.40 & 89.30$\pm$1.50 & 89.20$\pm$2.29 & 90.00$\pm$2.17 & 90.80$\pm$2.46 & \textbf{91.60$\pm$1.16} \\
air.\ vs car.        & 89.20$\pm$0.60 & 91.20$\pm$1.29 & 92.50$\pm$1.38 & 92.20$\pm$1.60 & 90.40$\pm$0.58 & 94.70$\pm$0.93 & \textbf{95.30$\pm$0.68} \\
air.\ vs cat.        & 91.00$\pm$1.48 & 93.60$\pm$0.97 & \textbf{94.50$\pm$0.84} & 93.50$\pm$1.30 & 93.60$\pm$1.36 & 93.80$\pm$1.50 & 94.10$\pm$1.39 \\
air.\ vs deer.       & 93.20$\pm$1.12 & 93.70$\pm$1.36 & \textbf{96.10$\pm$0.97} & 93.70$\pm$0.75 & 95.40$\pm$0.97 & 94.80$\pm$2.66 & 95.70$\pm$0.40 \\
air.\ vs dog.        & 92.60$\pm$0.86 & 92.40$\pm$1.46 & \textbf{95.40$\pm$1.62} & 93.30$\pm$1.08 & 93.80$\pm$0.98 & 94.30$\pm$0.93 & 95.20$\pm$1.50 \\
air.\ vs horse.      & 92.90$\pm$1.11 & 94.20$\pm$0.40 & 96.20$\pm$0.93 & 94.20$\pm$2.89 & 95.00$\pm$0.63 & 95.70$\pm$1.17 & \textbf{96.60$\pm$0.73} \\
air.\ vs monkey.     & 93.40$\pm$0.58 & 97.10$\pm$1.28 & 97.10$\pm$1.56 & 95.40$\pm$1.28 & 95.70$\pm$0.93 & 95.50$\pm$1.30 & \textbf{97.60$\pm$0.49} \\
air.\ vs ship.       & \textbf{83.70$\pm$2.48} & 80.30$\pm$2.79 & 80.40$\pm$2.67 & 79.90$\pm$3.07 & 79.50$\pm$1.26 & 78.80$\pm$1.29 & 80.20$\pm$3.31 \\
air.\ vs truck.      & 89.90$\pm$1.53 & 91.50$\pm$1.00 & 91.70$\pm$1.50 & 84.70$\pm$2.16 & 90.50$\pm$0.89 & 91.10$\pm$2.27 & \textbf{92.40$\pm$2.29} \\
bird.\ vs car.       & 94.60$\pm$1.59 & 95.10$\pm$1.28 & 96.20$\pm$0.98 & 94.90$\pm$1.02 & 94.70$\pm$1.57 & 92.90$\pm$1.50 & \textbf{96.60$\pm$1.07} \\
bird.\ vs cat.       & 73.30$\pm$1.96 & 74.50$\pm$3.79 & 77.40$\pm$2.06 & 71.80$\pm$3.30 & 72.70$\pm$2.16 & 70.90$\pm$3.31 & \textbf{77.60$\pm$2.40} \\
bird.\ vs deer.      & 71.70$\pm$2.14 & 73.00$\pm$1.26 & 76.40$\pm$2.67 & 76.60$\pm$0.97 & \textbf{78.70$\pm$2.42} & 76.60$\pm$3.77 & 76.00$\pm$3.30 \\
bird.\ vs dog.       & 70.80$\pm$1.44 & 70.50$\pm$3.05 & \textbf{72.60$\pm$1.50} & 71.30$\pm$2.50 & 71.30$\pm$2.66 & 71.00$\pm$4.83 & 72.50$\pm$3.10 \\
bird.\ vs horse.     & 78.10$\pm$2.40 & 84.70$\pm$1.50 & 82.10$\pm$1.07 & 81.50$\pm$2.41 & 80.40$\pm$2.27 & 82.60$\pm$1.36 & \textbf{87.20$\pm$1.91} \\
bird.\ vs monkey.    & 71.60$\pm$0.73 & 73.60$\pm$3.07 & 76.20$\pm$2.73 & 75.80$\pm$3.33 & 75.00$\pm$3.32 & 77.90$\pm$1.98 & \textbf{79.20$\pm$3.19} \\
bird.\ vs ship.      & 91.60$\pm$1.96 & 94.00$\pm$2.47 & \textbf{95.40$\pm$1.96} & 91.60$\pm$0.73 & 91.10$\pm$1.56 & 92.80$\pm$1.25 & 93.20$\pm$1.50 \\
bird.\ vs truck.     & 94.80$\pm$1.21 & \textbf{96.20$\pm$0.98} & 93.60$\pm$1.39 & 94.50$\pm$2.70 & 94.20$\pm$1.21 & 92.20$\pm$1.50 & 95.30$\pm$1.17 \\
car.\ vs cat.        & 93.80$\pm$1.72 & 95.40$\pm$1.46 & 95.20$\pm$1.33 & 94.20$\pm$1.47 & 95.20$\pm$1.03 & 94.30$\pm$1.50 & \textbf{95.70$\pm$1.63} \\
car.\ vs deer.       & 92.50$\pm$1.05 & 96.00$\pm$0.89 & 96.70$\pm$0.60 & 95.80$\pm$1.47 & 93.40$\pm$0.86 & 96.70$\pm$0.87 & \textbf{96.80$\pm$1.29} \\
car.\ vs dog.        & 84.20$\pm$2.32 & 94.60$\pm$1.16 & 93.10$\pm$1.53 & 94.60$\pm$1.71 & 92.80$\pm$1.91 & \textbf{95.10$\pm$1.11} & 94.80$\pm$1.86 \\
car.\ vs horse.      & 93.40$\pm$1.02 & 95.40$\pm$1.83 & 96.60$\pm$0.66 & 94.40$\pm$1.36 & 93.50$\pm$1.34 & 93.80$\pm$1.21 & \textbf{97.70$\pm$0.98} \\
car.\ vs monkey.     & 96.00$\pm$1.22 & 96.30$\pm$1.33 & 95.70$\pm$1.21 & 95.90$\pm$1.16 & 96.00$\pm$0.55 & 96.30$\pm$1.40 & \textbf{96.70$\pm$0.93} \\
car.\ vs ship.       & \textbf{89.10$\pm$2.24} & 88.20$\pm$1.36 & 86.60$\pm$3.12 & 87.70$\pm$2.36 & 88.30$\pm$1.91 & 86.70$\pm$3.23 & 88.40$\pm$1.98 \\
car.\ vs truck.      & 85.30$\pm$1.54 & 87.00$\pm$2.43 & 88.30$\pm$1.63 & 84.70$\pm$1.60 & 84.90$\pm$2.62 & 87.80$\pm$2.18 & \textbf{88.60$\pm$2.42} \\
cat.\ vs deer.       & 74.70$\pm$3.20 & 74.50$\pm$2.76 & \textbf{77.40$\pm$2.13} & 72.50$\pm$3.36 & 73.70$\pm$3.20 & 74.60$\pm$4.14 & 76.90$\pm$2.20 \\
cat.\ vs dog.        & 60.10$\pm$4.85 & 62.60$\pm$3.48 & 68.60$\pm$1.28 & 64.30$\pm$1.89 & 66.00$\pm$2.86 & 60.80$\pm$2.77 & \textbf{68.90$\pm$2.33} \\
cat.\ vs horse.      & 79.60$\pm$2.06 & 77.30$\pm$3.14 & 82.90$\pm$1.88 & 79.90$\pm$2.87 & 80.80$\pm$1.47 & 82.60$\pm$2.92 & \textbf{83.30$\pm$1.86} \\
cat.\ vs monkey.     & 72.30$\pm$3.80 & 73.00$\pm$3.18 & 76.60$\pm$3.51 & 73.90$\pm$3.56 & 73.10$\pm$3.60 & 76.00$\pm$4.00 & \textbf{78.30$\pm$4.63} \\
cat.\ vs ship.       & 92.80$\pm$1.44 & 93.80$\pm$2.50 & \textbf{94.20$\pm$1.03} & 92.60$\pm$1.39 & 91.90$\pm$1.77 & 91.50$\pm$1.38 & 93.10$\pm$3.40 \\
cat.\ vs truck.      & 92.90$\pm$1.36 & 94.50$\pm$0.89 & 93.50$\pm$1.84 & 94.30$\pm$0.81 & 92.80$\pm$1.50 & 91.10$\pm$2.03 & \textbf{95.60$\pm$1.93} \\
deer.\ vs dog.       & 74.50$\pm$1.00 & 74.70$\pm$2.14 & 76.00$\pm$2.90 & 72.70$\pm$2.04 & 75.70$\pm$1.72 & 75.10$\pm$2.13 & \textbf{76.60$\pm$2.27} \\
deer.\ vs horse.     & 79.90$\pm$2.60 & 80.10$\pm$2.58 & \textbf{84.20$\pm$2.66} & 77.00$\pm$1.90 & 81.70$\pm$5.69 & 82.10$\pm$2.13 & 83.90$\pm$1.83 \\
deer.\ vs monkey.    & 77.50$\pm$2.35 & 77.60$\pm$3.51 & 82.00$\pm$2.12 & 81.20$\pm$2.23 & 77.20$\pm$1.89 & 80.70$\pm$2.16 & \textbf{82.40$\pm$3.04} \\
deer.\ vs ship.      & 94.90$\pm$1.39 & 96.60$\pm$1.53 & 97.20$\pm$1.60 & 94.70$\pm$1.36 & 94.20$\pm$1.44 & 93.70$\pm$2.77 & \textbf{97.70$\pm$1.75} \\
deer.\ vs truck.     & 95.80$\pm$1.17 & 96.70$\pm$1.29 & 97.60$\pm$0.73 & 94.60$\pm$1.16 & 94.70$\pm$1.44 & 93.10$\pm$1.88 & \textbf{97.90$\pm$0.66} \\
dog.\ vs horse.      & 71.60$\pm$2.40 & 69.50$\pm$2.07 & 73.50$\pm$3.11 & 73.20$\pm$2.25 & 72.30$\pm$3.50 & 72.70$\pm$2.84 & \textbf{73.80$\pm$3.52} \\
dog.\ vs monkey.     & 66.20$\pm$1.86 & 67.20$\pm$3.75 & 69.60$\pm$3.41 & 64.80$\pm$4.23 & 66.60$\pm$1.91 & 62.70$\pm$1.57 & \textbf{70.00$\pm$2.85} \\
dog.\ vs ship.       & \textbf{96.10$\pm$1.50} & 94.00$\pm$1.14 & 93.00$\pm$1.14 & 94.00$\pm$1.22 & 92.20$\pm$1.08 & 91.20$\pm$0.87 & 95.00$\pm$1.38 \\
dog.\ vs truck.      & 93.20$\pm$0.68 & 94.50$\pm$1.05 & 94.60$\pm$1.59 & 94.30$\pm$2.14 & 92.70$\pm$1.75 & \textbf{96.60$\pm$1.39} & 93.90$\pm$1.28 \\
horse.\ vs monkey.   & 74.30$\pm$3.17 & 77.70$\pm$3.01 & \textbf{81.90$\pm$3.34} & 78.50$\pm$2.35 & 76.60$\pm$2.96 & 72.50$\pm$2.39 & \textbf{81.90$\pm$3.47} \\
horse.\ vs ship.     & 93.10$\pm$1.36 & 94.10$\pm$1.39 & \textbf{95.70$\pm$1.36} & 93.90$\pm$1.39 & 92.20$\pm$0.68 & 92.20$\pm$0.81 & 94.60$\pm$2.01 \\
horse.\ vs truck.    & 91.90$\pm$2.24 & 93.60$\pm$0.80 & 91.50$\pm$1.67 & 92.90$\pm$3.01 & 91.20$\pm$1.21 & 91.10$\pm$2.01 & \textbf{96.10$\pm$1.59} \\
monkey.\ vs ship.    & 95.40$\pm$1.98 & 95.60$\pm$1.53 & \textbf{96.90$\pm$1.32} & 95.40$\pm$1.16 & 95.00$\pm$0.84 & 93.10$\pm$1.02 & 96.40$\pm$0.80 \\
monkey.\ vs truck.   & 95.30$\pm$1.29 & 95.60$\pm$1.24 & 94.80$\pm$1.54 & \textbf{95.90$\pm$0.58} & 94.50$\pm$1.10 & 94.30$\pm$2.01 & 94.80$\pm$0.68 \\
ship.\ vs truck.     & 76.50$\pm$2.49 & 79.80$\pm$1.17 & 77.70$\pm$1.78 & 76.20$\pm$3.14 & 76.40$\pm$1.32 & 77.50$\pm$1.58 & \textbf{80.80$\pm$1.54} \\
\midrule
Avg.~Acc.            & 85.33 & 86.70 & 87.88 & 86.18 & 86.17 & 86.27 & \textbf{88.60} \\
Avg.~Rank            & 5.42  & 3.93  & 2.69  & 4.73  & 4.98  & 4.58  & \textbf{1.67}  \\
\bottomrule
\end{tabular}
}
\end{table*}

\begin{table*}[!htbp]
\centering
\caption{Precision rates (\%) with standard deviations for the STL-10 dataset.}
\label{tab:stl10-prec}
{\scriptsize
\begin{tabular}{cccccccc}
\toprule
Datasets & SVM-2K & MvSVM-2C & MvLSSVC-2C & Wave-MvSVM & MVASY-BX & Mv$SL_{0/1}$-SVM & Mv$L_{\mathrm{ht}}$-SVM \\
\midrule
air.\ vs bird. & 87.26$\pm$2.74 & \textbf{92.93$\pm$2.20} & 90.49$\pm$2.89 & 85.58$\pm$2.69 & 92.28$\pm$3.34 & 91.07$\pm$3.42 & 92.52$\pm$2.40 \\
air.\ vs car. & \textbf{98.28$\pm$0.98} & 91.41$\pm$2.07 & 95.89$\pm$0.85 & 94.14$\pm$1.45 & 96.99$\pm$1.15 & 95.91$\pm$1.40 & 95.58$\pm$0.99 \\
air.\ vs cat. & 94.21$\pm$1.61 & 97.68$\pm$2.29 & 97.26$\pm$1.36 & 94.20$\pm$2.50 & \textbf{98.11$\pm$2.29} & 96.85$\pm$2.52 & 94.63$\pm$2.20 \\
air.\ vs deer. & 96.18$\pm$1.33 & 95.69$\pm$2.12 & 98.12$\pm$0.79 & 95.11$\pm$1.63 & \textbf{98.93$\pm$0.66} & 96.85$\pm$2.24 & 94.94$\pm$1.50 \\
air.\ vs dog. & 94.97$\pm$1.24 & 95.33$\pm$2.08 & 96.38$\pm$2.49 & 92.00$\pm$2.67 & \textbf{97.44$\pm$1.44} & 95.69$\pm$1.33 & 96.37$\pm$2.10 \\
air.\ vs horse. & 96.82$\pm$2.03 & 97.26$\pm$1.33 & 97.95$\pm$1.55 & 94.77$\pm$2.85 & \textbf{98.73$\pm$1.00} & 95.97$\pm$0.66 & 95.34$\pm$1.09 \\
air.\ vs monkey. & 98.02$\pm$0.80 & 98.15$\pm$0.42 & 97.98$\pm$1.81 & 96.34$\pm$1.36 & \textbf{98.51$\pm$0.52} & 95.25$\pm$1.44 & 98.39$\pm$0.77 \\
air.\ vs ship. & 86.08$\pm$2.50 & \textbf{88.73$\pm$3.26} & 86.74$\pm$3.18 & 81.84$\pm$2.70 & 84.70$\pm$3.88 & 78.97$\pm$1.31 & 84.35$\pm$1.59 \\
air.\ vs truck. & 92.50$\pm$2.90 & 91.94$\pm$3.14 & \textbf{93.09$\pm$2.96} & 89.31$\pm$3.73 & 92.28$\pm$3.25 & 90.76$\pm$4.31 & 90.23$\pm$3.15 \\
bird.\ vs car. & 93.73$\pm$1.30 & 94.36$\pm$2.07 & 95.00$\pm$1.92 & \textbf{97.50$\pm$1.68} & 92.46$\pm$2.20 & 90.58$\pm$2.02 & 95.71$\pm$1.53 \\
bird.\ vs cat. & 83.80$\pm$3.60 & 79.70$\pm$5.06 & 81.68$\pm$3.61 & 74.33$\pm$2.45 & 80.79$\pm$1.92 & 71.97$\pm$4.55 & \textbf{84.19$\pm$3.29} \\
bird.\ vs deer. & \textbf{95.04$\pm$1.64} & 90.24$\pm$2.65 & 86.71$\pm$2.98 & 88.08$\pm$3.35 & 89.20$\pm$3.75 & 76.96$\pm$5.53 & 81.32$\pm$8.18 \\
bird.\ vs dog. & \textbf{90.43$\pm$3.88} & 82.33$\pm$3.90 & 78.16$\pm$1.56 & 74.20$\pm$1.41 & 78.05$\pm$4.37 & 76.06$\pm$8.21 & 82.18$\pm$2.56 \\
bird.\ vs horse. & 89.76$\pm$1.20 & 84.53$\pm$2.87 & 88.90$\pm$0.93 & 86.85$\pm$1.59 & \textbf{91.96$\pm$1.43} & 85.39$\pm$5.70 & 89.99$\pm$2.34 \\
bird.\ vs monkey. & \textbf{92.41$\pm$2.57} & 85.73$\pm$4.47 & 78.15$\pm$3.18 & 74.60$\pm$5.06 & 81.84$\pm$4.11 & 80.15$\pm$3.02 & 84.37$\pm$3.58 \\
bird.\ vs ship. & 89.25$\pm$1.79 & 93.23$\pm$3.24 & 93.69$\pm$2.27 & \textbf{94.07$\pm$0.48} & 86.91$\pm$3.04 & 91.73$\pm$2.30 & 92.07$\pm$1.90 \\
bird.\ vs truck. & 93.46$\pm$1.72 & 95.49$\pm$1.24 & 91.80$\pm$1.60 & \textbf{97.22$\pm$2.20} & 92.20$\pm$1.50 & 88.99$\pm$2.14 & 92.98$\pm$0.98 \\
car.\ vs cat. & 96.19$\pm$0.56 & 97.15$\pm$2.20 & 97.69$\pm$1.53 & 95.30$\pm$1.47 & \textbf{98.93$\pm$1.14} & 95.40$\pm$2.83 & 94.25$\pm$3.83 \\
car.\ vs deer. & 97.33$\pm$0.52 & 98.34$\pm$1.06 & \textbf{98.96$\pm$0.66} & 98.12$\pm$1.21 & 96.81$\pm$1.28 & 95.63$\pm$2.67 & 94.54$\pm$2.19 \\
car.\ vs dog. & 76.23$\pm$2.60 & \textbf{97.68$\pm$1.17} & 96.77$\pm$0.55 & 95.21$\pm$2.55 & 97.58$\pm$1.24 & 96.92$\pm$1.60 & 93.96$\pm$2.39 \\
car.\ vs horse. & 96.99$\pm$1.19 & 95.99$\pm$2.45 & 98.16$\pm$1.34 & 94.04$\pm$0.92 & 96.81$\pm$1.42 & 95.43$\pm$1.06 & \textbf{99.18$\pm$0.76} \\
car.\ vs monkey. & 98.35$\pm$1.36 & 97.73$\pm$1.21 & 98.74$\pm$1.01 & 95.46$\pm$1.42 & \textbf{99.57$\pm$0.52} & 95.47$\pm$3.28 & 96.81$\pm$1.40 \\
car.\ vs ship. & 84.39$\pm$2.38 & 87.64$\pm$3.35 & 85.61$\pm$8.44 & \textbf{88.04$\pm$3.41} & 84.46$\pm$2.72 & 87.85$\pm$6.87 & 87.34$\pm$3.07 \\
car.\ vs truck. & \textbf{90.04$\pm$2.63} & 88.94$\pm$4.47 & 89.43$\pm$2.84 & 84.81$\pm$2.08 & 89.38$\pm$2.51 & 89.46$\pm$3.17 & 89.97$\pm$3.32 \\
cat.\ vs deer. & 79.34$\pm$4.50 & 77.22$\pm$3.26 & \textbf{80.95$\pm$2.92} & 73.70$\pm$3.49 & \textbf{80.95$\pm$2.32} & 79.95$\pm$5.22 & 79.82$\pm$3.26 \\
cat.\ vs dog. & 58.53$\pm$3.78 & 62.14$\pm$2.57 & 67.81$\pm$0.49 & 66.19$\pm$2.43 & 63.17$\pm$2.18 & 62.34$\pm$3.43 & \textbf{68.31$\pm$1.69} \\
cat.\ vs horse. & 76.44$\pm$2.29 & 77.22$\pm$2.78 & 83.17$\pm$2.52 & 82.11$\pm$3.05 & 80.37$\pm$3.24 & 82.02$\pm$3.29 & \textbf{83.52$\pm$1.79} \\
cat.\ vs monkey. & 72.90$\pm$3.30 & 73.66$\pm$2.02 & 77.52$\pm$2.90 & 74.62$\pm$2.27 & 75.07$\pm$2.73 & 76.77$\pm$3.32 & \textbf{79.27$\pm$4.53} \\
cat.\ vs ship. & 90.57$\pm$1.88 & 91.75$\pm$3.56 & 92.05$\pm$1.48 & 91.75$\pm$3.08 & 87.13$\pm$2.21 & \textbf{92.32$\pm$2.28} & 90.68$\pm$5.53 \\
cat.\ vs truck. & 91.41$\pm$2.49 & 93.23$\pm$1.25 & 91.84$\pm$2.97 & \textbf{97.30$\pm$2.31} & 89.24$\pm$1.92 & 93.25$\pm$4.08 & 95.81$\pm$2.43 \\
deer.\ vs dog. & 71.01$\pm$0.73 & \textbf{73.81$\pm$2.14} & 72.83$\pm$3.13 & 71.39$\pm$1.68 & 72.98$\pm$2.20 & 71.11$\pm$2.47 & 73.55$\pm$2.37 \\
deer.\ vs horse. & 75.57$\pm$2.86 & 78.72$\pm$3.58 & 82.12$\pm$3.26 & 78.43$\pm$3.52 & 81.18$\pm$8.82 & 81.56$\pm$3.22 & \textbf{82.56$\pm$3.04} \\
deer.\ vs monkey. & 74.35$\pm$2.19 & 74.86$\pm$3.97 & \textbf{78.59$\pm$2.10} & 77.90$\pm$2.40 & 74.88$\pm$2.06 & 78.12$\pm$2.48 & 78.30$\pm$3.29 \\
deer.\ vs ship. & 92.46$\pm$1.47 & 94.50$\pm$2.11 & 95.78$\pm$2.40 & 95.53$\pm$1.48 & 90.82$\pm$2.10 & 96.63$\pm$3.53 & \textbf{97.79$\pm$1.61} \\
deer.\ vs truck. & 93.92$\pm$2.07 & 95.72$\pm$1.77 & 96.31$\pm$0.93 & 95.17$\pm$1.61 & 91.58$\pm$2.93 & 95.21$\pm$3.14 & \textbf{97.81$\pm$0.95} \\
dog.\ vs horse. & 71.56$\pm$2.57 & 68.93$\pm$1.97 & 73.34$\pm$3.36 & 73.05$\pm$3.01 & \textbf{75.90$\pm$5.30} & 72.62$\pm$3.14 & 74.10$\pm$4.03 \\
dog.\ vs monkey. & 68.02$\pm$2.13 & 68.18$\pm$3.73 & 71.05$\pm$4.91 & 61.59$\pm$3.26 & 67.61$\pm$2.05 & 65.67$\pm$5.54 & \textbf{73.24$\pm$5.07} \\
dog.\ vs ship. & \textbf{96.02$\pm$1.75} & 93.02$\pm$2.07 & 91.24$\pm$1.98 & 94.58$\pm$1.73 & 87.99$\pm$1.75 & 94.65$\pm$1.69 & 93.46$\pm$1.58 \\
dog.\ vs truck. & 91.71$\pm$1.05 & 93.94$\pm$2.17 & 92.29$\pm$2.43 & 95.66$\pm$1.70 & 88.81$\pm$2.50 & \textbf{95.89$\pm$1.87} & 93.13$\pm$1.20 \\
horse.\ vs monkey. & 81.91$\pm$3.16 & 79.53$\pm$3.91 & \textbf{84.41$\pm$4.12} & 78.48$\pm$4.00 & 75.64$\pm$3.41 & 78.44$\pm$4.60 & 84.31$\pm$3.41 \\
horse.\ vs ship. & 89.78$\pm$2.43 & 93.05$\pm$2.30 & 93.75$\pm$2.28 & 93.75$\pm$2.75 & 87.85$\pm$1.25 & 91.28$\pm$1.58 & \textbf{98.11$\pm$1.67} \\
horse.\ vs truck. & 89.22$\pm$2.81 & 92.50$\pm$2.24 & 88.15$\pm$2.65 & 93.25$\pm$3.84 & 86.82$\pm$1.77 & 93.56$\pm$1.89 & \textbf{94.14$\pm$2.61} \\
monkey.\ vs ship. & 93.87$\pm$2.69 & 93.90$\pm$2.31 & 95.22$\pm$2.14 & 93.84$\pm$1.60 & 91.52$\pm$1.14 & 95.13$\pm$2.51 & \textbf{95.67$\pm$0.98} \\
monkey.\ vs truck. & 92.86$\pm$2.11 & 95.09$\pm$1.81 & 92.85$\pm$2.92 & \textbf{96.77$\pm$1.27} & 90.87$\pm$1.86 & 94.16$\pm$4.29 & 95.22$\pm$1.76 \\
ship.\ vs truck. & \textbf{86.23$\pm$3.87} & 80.45$\pm$1.53 & 79.83$\pm$3.15 & 77.78$\pm$1.78 & 79.02$\pm$1.88 & 79.50$\pm$1.43 & 82.93$\pm$3.16 \\
\midrule
Avg.~Prec. & 87.76 & 88.21 & 88.77 & 87.20 & 87.43 & 87.23 & \textbf{89.18} \\
Avg.~Rank & 4.44 & 3.87 & 3.20 & 4.73 & 4.24 & 4.42 & \textbf{3.09} \\
\bottomrule
\end{tabular}
}
\end{table*}

\begin{table*}[!htbp]
\centering
\caption{Recall rates (\%) with standard deviations for the STL-10 dataset.}
\label{tab:stl10-rec}
{\scriptsize
\begin{tabular}{cccccccc}
\toprule
Datasets & SVM-2K & MvSVM-2C & MvLSSVC-2C & Wave-MvSVM & MVASY-BX & Mv$SL_{0/1}$-SVM & Mv$L_{\mathrm{ht}}$-SVM \\
\midrule
air.\ vs bird. & 80.60$\pm$3.72 & 86.20$\pm$1.33 & 88.00$\pm$2.68 & \textbf{94.40$\pm$3.01} & 87.40$\pm$1.20 & 90.60$\pm$2.65 & 90.60$\pm$0.80 \\
air.\ vs car. & 79.80$\pm$0.75 & 91.00$\pm$1.10 & 88.80$\pm$2.14 & 90.00$\pm$2.10 & 83.40$\pm$1.02 & 93.40$\pm$0.80 & \textbf{95.00$\pm$0.63} \\
air.\ vs cat. & 87.40$\pm$2.58 & 89.40$\pm$1.85 & 91.60$\pm$1.36 & 92.80$\pm$1.94 & 89.00$\pm$3.03 & 90.60$\pm$1.20 & \textbf{93.60$\pm$3.01} \\
air.\ vs deer. & 90.00$\pm$2.28 & 91.60$\pm$2.73 & 94.00$\pm$1.41 & 92.20$\pm$2.32 & 91.80$\pm$1.94 & 92.60$\pm$3.72 & \textbf{96.60$\pm$1.85} \\
air.\ vs dog. & 90.00$\pm$2.19 & 89.20$\pm$1.60 & 94.40$\pm$1.36 & \textbf{95.00$\pm$2.10} & 90.00$\pm$2.19 & 92.80$\pm$1.17 & 94.00$\pm$2.76 \\
air.\ vs horse. & 88.80$\pm$2.93 & 91.00$\pm$1.79 & 94.40$\pm$1.02 & 93.60$\pm$4.08 & 91.20$\pm$2.04 & 95.40$\pm$1.85 & \textbf{98.00$\pm$0.63} \\
air.\ vs monkey. & 88.60$\pm$1.36 & 96.00$\pm$2.28 & 96.20$\pm$1.60 & 94.40$\pm$1.74 & 92.80$\pm$1.72 & 95.80$\pm$2.14 & \textbf{96.80$\pm$1.33} \\
air.\ vs ship. & \textbf{80.40$\pm$3.20} & 69.40$\pm$3.38 & 71.80$\pm$3.31 & 76.80$\pm$4.45 & 72.40$\pm$4.03 & 78.60$\pm$4.18 & 74.20$\pm$7.73 \\
air.\ vs truck. & 87.00$\pm$3.46 & 91.20$\pm$2.99 & 90.20$\pm$1.17 & 79.00$\pm$1.79 & 88.60$\pm$1.96 & 91.80$\pm$0.98 & \textbf{95.20$\pm$1.72} \\
bird.\ vs car. & 95.60$\pm$2.65 & 96.00$\pm$2.28 & \textbf{97.60$\pm$1.85} & 92.20$\pm$1.94 & 97.40$\pm$1.96 & 95.80$\pm$1.17 & \textbf{97.60$\pm$1.36} \\
bird.\ vs cat. & 58.00$\pm$4.43 & 66.20$\pm$6.24 & \textbf{71.00$\pm$5.18} & 66.40$\pm$5.68 & 59.60$\pm$4.76 & 69.60$\pm$8.09 & 68.20$\pm$5.42 \\
bird.\ vs deer. & 45.80$\pm$4.45 & 51.80$\pm$4.71 & 62.40$\pm$4.80 & 61.80$\pm$3.76 & 65.60$\pm$6.12 & \textbf{77.00$\pm$6.07} & 69.80$\pm$9.06 \\
bird.\ vs dog. & 46.80$\pm$4.31 & 52.20$\pm$4.96 & 62.80$\pm$3.97 & \textbf{65.20$\pm$5.56} & 59.80$\pm$6.40 & 63.20$\pm$9.66 & 57.60$\pm$8.11 \\
bird.\ vs horse. & 63.40$\pm$4.72 & \textbf{85.20$\pm$4.17} & 73.40$\pm$3.01 & 74.20$\pm$4.02 & 66.60$\pm$4.13 & 79.80$\pm$8.26 & 83.80$\pm$3.82 \\
bird.\ vs monkey. & 47.20$\pm$2.93 & 56.80$\pm$6.05 & 72.80$\pm$3.49 & \textbf{79.20$\pm$2.40} & 64.40$\pm$6.05 & 74.40$\pm$3.67 & 71.80$\pm$5.31 \\
bird.\ vs ship. & 94.60$\pm$2.65 & 95.00$\pm$3.16 & \textbf{97.40$\pm$2.65} & 88.80$\pm$1.47 & 97.00$\pm$2.00 & 94.20$\pm$2.99 & 94.60$\pm$2.58 \\
bird.\ vs truck. & 96.40$\pm$2.42 & 97.00$\pm$1.67 & 95.80$\pm$2.71 & 91.60$\pm$3.50 & 96.60$\pm$1.96 & 96.40$\pm$2.73 & \textbf{98.00$\pm$1.41} \\
car.\ vs cat. & 91.20$\pm$3.25 & 93.60$\pm$2.24 & 92.60$\pm$1.85 & 93.00$\pm$2.19 & 91.40$\pm$1.85 & 93.20$\pm$1.94 & \textbf{97.60$\pm$1.62} \\
car.\ vs deer. & 87.40$\pm$2.15 & 93.60$\pm$2.06 & 94.40$\pm$1.02 & 93.40$\pm$2.58 & 89.80$\pm$2.32 & 98.00$\pm$1.41 & \textbf{99.40$\pm$0.80} \\
car.\ vs dog. & \textbf{99.60$\pm$0.49} & 91.40$\pm$2.65 & 89.20$\pm$3.66 & 94.00$\pm$2.10 & 87.80$\pm$3.87 & 93.20$\pm$2.32 & 95.80$\pm$1.60 \\
car.\ vs horse. & 89.60$\pm$2.06 & 94.80$\pm$1.60 & 95.00$\pm$0.63 & 94.80$\pm$2.04 & 90.00$\pm$2.97 & 92.00$\pm$1.55 & \textbf{96.20$\pm$1.72} \\
car.\ vs monkey. & 93.60$\pm$2.50 & 94.80$\pm$1.60 & 92.60$\pm$2.50 & 96.40$\pm$1.36 & 92.40$\pm$1.02 & \textbf{97.40$\pm$2.50} & 96.60$\pm$1.02 \\
car.\ vs ship. & \textbf{96.00$\pm$2.19} & 89.20$\pm$1.72 & 90.00$\pm$6.23 & 87.40$\pm$1.85 & 94.00$\pm$0.63 & 86.80$\pm$10.17 & 90.00$\pm$3.16 \\
car.\ vs truck. & 79.60$\pm$4.72 & 84.80$\pm$1.72 & \textbf{87.00$\pm$2.00} & 84.60$\pm$1.85 & 79.20$\pm$3.66 & 85.80$\pm$1.47 & \textbf{87.00$\pm$2.10} \\
cat.\ vs deer. & 67.00$\pm$3.63 & 69.60$\pm$4.59 & 71.80$\pm$3.31 & 70.00$\pm$5.02 & 62.00$\pm$7.24 & 65.80$\pm$4.66 & \textbf{72.20$\pm$2.71} \\
cat.\ vs dog. & 69.40$\pm$8.50 & 64.00$\pm$7.69 & 70.80$\pm$4.07 & 58.80$\pm$4.02 & \textbf{77.00$\pm$8.22} & 55.00$\pm$2.28 & 70.40$\pm$4.76 \\
cat.\ vs horse. & \textbf{85.80$\pm$4.92} & 77.60$\pm$6.56 & 82.60$\pm$2.58 & 76.60$\pm$5.75 & 82.00$\pm$5.10 & 83.60$\pm$3.01 & 83.00$\pm$3.35 \\
cat.\ vs monkey. & 70.80$\pm$5.53 & 71.40$\pm$5.68 & 74.80$\pm$4.87 & 72.20$\pm$6.94 & 70.00$\pm$12.99 & 74.40$\pm$5.82 & \textbf{76.60$\pm$5.39} \\
cat.\ vs ship. & 95.60$\pm$1.96 & 96.40$\pm$1.85 & 96.80$\pm$1.83 & 93.80$\pm$1.83 & \textbf{98.40$\pm$2.06} & 90.60$\pm$1.85 & 96.60$\pm$2.42 \\
cat.\ vs truck. & 94.80$\pm$1.47 & 96.00$\pm$1.79 & 95.60$\pm$1.20 & 91.20$\pm$1.17 & \textbf{97.40$\pm$2.58} & 89.00$\pm$5.80 & 95.40$\pm$1.62 \\
deer.\ vs dog. & 82.80$\pm$2.04 & 76.60$\pm$2.58 & 83.20$\pm$2.64 & 75.80$\pm$4.45 & 81.80$\pm$2.64 & \textbf{84.80$\pm$2.79} & 83.20$\pm$3.37 \\
deer.\ vs horse. & \textbf{88.60$\pm$3.07} & 82.80$\pm$3.87 & 87.60$\pm$3.20 & 74.80$\pm$2.04 & 84.80$\pm$2.56 & 83.20$\pm$2.56 & 86.20$\pm$3.06 \\
deer.\ vs monkey. & 84.00$\pm$2.45 & 83.40$\pm$2.33 & 88.00$\pm$2.10 & 87.20$\pm$1.60 & 82.00$\pm$3.85 & 85.40$\pm$2.06 & \textbf{89.80$\pm$2.32} \\
deer.\ vs ship. & 97.80$\pm$2.04 & \textbf{99.00$\pm$1.26} & 98.80$\pm$0.75 & 93.80$\pm$1.94 & 98.40$\pm$0.80 & 90.80$\pm$6.08 & 97.60$\pm$2.06 \\
deer.\ vs truck. & 98.00$\pm$0.89 & 97.80$\pm$1.17 & \textbf{99.00$\pm$0.63} & 94.00$\pm$1.67 & 98.60$\pm$1.36 & 91.00$\pm$4.94 & 98.00$\pm$0.63 \\
dog.\ vs horse. & 71.80$\pm$3.82 & 71.00$\pm$2.45 & \textbf{74.00$\pm$4.29} & 73.80$\pm$2.93 & 66.00$\pm$5.33 & 73.00$\pm$3.41 & 73.40$\pm$4.84 \\
dog.\ vs monkey. & 61.20$\pm$2.86 & 64.40$\pm$4.88 & 66.80$\pm$3.76 & \textbf{78.80$\pm$6.76} & 64.60$\pm$10.37 & 57.60$\pm$12.24 & 63.80$\pm$5.11 \\
dog.\ vs ship. & 96.20$\pm$1.47 & 95.20$\pm$1.17 & 95.20$\pm$1.72 & 93.40$\pm$2.24 & \textbf{97.80$\pm$0.75} & 87.40$\pm$2.42 & 96.80$\pm$1.94 \\
dog.\ vs truck. & 95.00$\pm$0.89 & 95.20$\pm$0.40 & 97.40$\pm$0.49 & 92.80$\pm$2.93 & \textbf{97.80$\pm$0.75} & 97.40$\pm$1.36 & 94.80$\pm$1.94 \\
horse.\ vs monkey. & 62.40$\pm$6.62 & 74.80$\pm$2.40 & 78.40$\pm$3.93 & \textbf{79.00$\pm$2.37} & 78.80$\pm$5.56 & 62.60$\pm$4.03 & 78.40$\pm$4.50 \\
horse.\ vs ship. & 97.40$\pm$1.85 & 95.40$\pm$1.85 & \textbf{98.00$\pm$1.67} & 94.20$\pm$2.48 & \textbf{98.00$\pm$1.79} & 93.40$\pm$3.01 & 91.00$\pm$4.15 \\
horse.\ vs truck. & 95.40$\pm$2.42 & 95.00$\pm$1.67 & 96.00$\pm$1.10 & 92.60$\pm$3.14 & 97.20$\pm$0.75 & 88.40$\pm$5.39 & \textbf{98.40$\pm$0.80} \\
monkey.\ vs ship. & 97.20$\pm$1.33 & 97.60$\pm$1.02 & 98.80$\pm$0.75 & 97.20$\pm$0.75 & \textbf{99.20$\pm$0.75} & 91.00$\pm$3.35 & 97.20$\pm$0.75 \\
monkey.\ vs truck. & 98.20$\pm$0.40 & 96.20$\pm$1.47 & 97.20$\pm$0.40 & 95.00$\pm$1.67 & \textbf{99.00$\pm$0.63} & 94.80$\pm$3.25 & 94.40$\pm$1.85 \\
ship.\ vs truck. & 63.20$\pm$3.97 & \textbf{78.80$\pm$2.71} & 74.40$\pm$2.73 & 73.20$\pm$5.81 & 72.00$\pm$3.03 & 74.20$\pm$4.40 & 77.80$\pm$2.23 \\
\midrule
Avg.~Rec. & 82.89 & 84.79 & 86.86 & 85.10 & 84.91 & 85.15 & \textbf{87.84} \\
Avg.~Rank & 5.16 & 4.40 & 3.02 & 4.44 & 4.20 & 4.18 & \textbf{2.60} \\
\bottomrule
\end{tabular}
}
\end{table*}

\begin{table*}[!htbp]
\centering
\caption{F1-score (\%) with standard deviations for the STL-10 dataset.}
\label{tab:stl10-f1}
{\scriptsize
\begin{tabular}{cccccccc}
\toprule
Datasets & SVM-2K & MvSVM-2C & MvLSSVC-2C & Wave-MvSVM & MVASY-BX & Mv$SL_{0/1}$-SVM & Mv$L_{\mathrm{ht}}$-SVM \\
\midrule
air.\ vs bird. & 83.76$\pm$2.76 & 89.43$\pm$1.38 & 89.16$\pm$1.48 & 89.74$\pm$2.18 & 89.76$\pm$2.10 & 90.79$\pm$2.40 & \textbf{91.53$\pm$1.06} \\
air.\ vs car. & 88.08$\pm$0.66 & 91.19$\pm$1.23 & 92.20$\pm$1.49 & 92.02$\pm$1.68 & 89.68$\pm$0.63 & 94.63$\pm$0.92 & \textbf{95.29$\pm$0.67} \\
air.\ vs cat. & 90.65$\pm$1.63 & 93.32$\pm$0.99 & \textbf{94.34$\pm$0.87} & 93.46$\pm$1.27 & 93.28$\pm$1.48 & 93.61$\pm$1.50 & 94.06$\pm$1.45 \\
air.\ vs deer. & 92.96$\pm$1.23 & 93.56$\pm$1.45 & \textbf{96.01$\pm$1.01} & 93.60$\pm$0.81 & 95.22$\pm$1.06 & 94.66$\pm$2.76 & 95.74$\pm$0.42 \\
air.\ vs dog. & 92.39$\pm$0.96 & 92.15$\pm$1.48 & \textbf{95.36$\pm$1.60} & 93.42$\pm$0.99 & 93.55$\pm$1.07 & 94.21$\pm$0.93 & 95.13$\pm$1.52 \\
air.\ vs horse. & 92.58$\pm$1.27 & 94.00$\pm$0.47 & 96.13$\pm$0.92 & 94.14$\pm$3.00 & 94.79$\pm$0.73 & 95.68$\pm$1.20 & \textbf{96.65$\pm$0.72} \\
air.\ vs monkey. & 93.06$\pm$0.66 & 97.05$\pm$1.33 & 97.08$\pm$1.57 & 95.35$\pm$1.32 & 95.56$\pm$0.99 & 95.51$\pm$1.33 & \textbf{97.58$\pm$0.51} \\
air.\ vs ship. & \textbf{83.13$\pm$2.66} & 77.87$\pm$3.27 & 78.54$\pm$3.00 & 79.21$\pm$3.47 & 77.90$\pm$1.54 & 78.71$\pm$1.79 & 78.73$\pm$4.43 \\
air.\ vs truck. & 89.58$\pm$1.65 & 91.47$\pm$0.97 & 91.59$\pm$1.41 & 83.80$\pm$2.06 & 90.33$\pm$0.70 & 91.21$\pm$2.01 & \textbf{92.63$\pm$2.16} \\
bird.\ vs car. & 94.64$\pm$1.64 & 95.14$\pm$1.29 & 96.25$\pm$0.97 & 94.75$\pm$1.06 & 94.84$\pm$1.50 & 93.11$\pm$1.42 & \textbf{96.64$\pm$1.06} \\
bird.\ vs cat. & 68.40$\pm$2.99 & 72.10$\pm$4.54 & \textbf{75.79$\pm$2.66} & 70.08$\pm$4.20 & 68.49$\pm$3.27 & 70.34$\pm$4.04 & 75.18$\pm$3.34 \\
bird.\ vs deer. & 61.68$\pm$4.10 & 65.61$\pm$2.99 & 72.47$\pm$3.76 & 72.48$\pm$1.70 & 75.34$\pm$3.79 & \textbf{76.67$\pm$3.51} & 74.23$\pm$3.80 \\
bird.\ vs dog. & 61.46$\pm$3.11 & 63.79$\pm$4.43 & \textbf{69.56$\pm$2.41} & 69.32$\pm$3.60 & 67.40$\pm$4.08 & 68.24$\pm$6.17 & 67.37$\pm$5.40 \\
bird.\ vs horse. & 74.23$\pm$3.55 & 84.75$\pm$1.68 & 80.36$\pm$1.56 & 79.99$\pm$2.87 & 77.19$\pm$3.18 & 81.94$\pm$2.35 & \textbf{86.72$\pm$2.15} \\
bird.\ vs monkey. & 62.38$\pm$1.99 & 68.12$\pm$4.66 & 75.35$\pm$2.98 & 76.68$\pm$2.27 & 71.91$\pm$4.42 & 77.08$\pm$2.15 & \textbf{77.47$\pm$3.73} \\
bird.\ vs ship. & 91.84$\pm$1.95 & 94.06$\pm$2.45 & \textbf{95.49$\pm$1.94} & 91.35$\pm$0.81 & 91.62$\pm$1.31 & 92.89$\pm$1.28 & 93.29$\pm$1.52 \\
bird.\ vs truck. & 94.87$\pm$1.21 & \textbf{96.23$\pm$0.98} & 93.73$\pm$1.41 & 94.31$\pm$2.82 & 94.33$\pm$1.18 & 92.51$\pm$1.48 & 95.42$\pm$1.14 \\
car.\ vs cat. & 93.61$\pm$1.90 & 95.31$\pm$1.48 & 95.07$\pm$1.39 & 94.12$\pm$1.50 & 95.00$\pm$1.11 & 94.25$\pm$1.46 & \textbf{95.82$\pm$1.43} \\
car.\ vs deer. & 92.08$\pm$1.19 & 95.89$\pm$0.96 & 96.62$\pm$0.62 & 95.68$\pm$1.55 & 93.14$\pm$0.97 & 96.76$\pm$0.78 & \textbf{96.89$\pm$1.22} \\
car.\ vs dog. & 86.34$\pm$1.75 & 94.40$\pm$1.28 & 92.78$\pm$1.77 & 94.57$\pm$1.69 & 92.38$\pm$2.18 & \textbf{95.00$\pm$1.18} & 94.86$\pm$1.82 \\
car.\ vs horse. & 93.13$\pm$1.12 & 95.38$\pm$1.83 & 96.55$\pm$0.65 & 94.41$\pm$1.38 & 93.24$\pm$1.50 & 93.68$\pm$1.24 & \textbf{97.66$\pm$1.00} \\
car.\ vs monkey. & 95.89$\pm$1.31 & 96.24$\pm$1.36 & 95.55$\pm$1.32 & 95.92$\pm$1.15 & 95.85$\pm$0.58 & 96.35$\pm$1.36 & \textbf{96.70$\pm$0.92} \\
car.\ vs ship. & \textbf{89.81$\pm$2.07} & 88.34$\pm$1.02 & 87.16$\pm$2.13 & 87.69$\pm$2.19 & 88.96$\pm$1.65 & 86.56$\pm$3.87 & 88.58$\pm$1.90 \\
car.\ vs truck. & 84.35$\pm$2.12 & 86.75$\pm$2.23 & 88.16$\pm$1.56 & 84.69$\pm$1.54 & 83.96$\pm$2.94 & 87.57$\pm$2.08 & \textbf{88.43$\pm$2.32} \\
cat.\ vs deer. & 72.59$\pm$3.46 & 73.14$\pm$3.26 & \textbf{76.04$\pm$2.36} & 71.74$\pm$3.84 & 69.98$\pm$5.14 & 72.13$\pm$4.54 & 75.76$\pm$2.21 \\
cat.\ vs dog. & 63.33$\pm$5.12 & 62.92$\pm$4.82 & 69.22$\pm$2.07 & 62.17$\pm$2.36 & 69.20$\pm$3.65 & 58.40$\pm$2.45 & \textbf{69.29$\pm$2.95} \\
cat.\ vs horse. & 80.75$\pm$2.22 & 77.26$\pm$3.62 & 82.85$\pm$1.85 & 79.13$\pm$3.57 & 80.99$\pm$1.65 & 82.78$\pm$2.81 & \textbf{83.23$\pm$2.03} \\
cat.\ vs monkey. & 71.80$\pm$4.39 & 72.46$\pm$3.88 & 76.12$\pm$3.90 & 73.29$\pm$4.62 & 71.58$\pm$6.97 & 75.53$\pm$4.46 & \textbf{77.89$\pm$4.83} \\
cat.\ vs ship. & 93.00$\pm$1.40 & 93.99$\pm$2.35 & \textbf{94.35$\pm$1.01} & 92.71$\pm$1.21 & 92.40$\pm$1.62 & 91.43$\pm$1.35 & 93.42$\pm$3.00 \\
cat.\ vs truck. & 93.05$\pm$1.25 & 94.58$\pm$0.90 & 93.66$\pm$1.72 & 94.13$\pm$0.78 & 93.11$\pm$1.45 & 90.85$\pm$2.29 & \textbf{95.60$\pm$1.92} \\
deer.\ vs dog. & 76.44$\pm$1.10 & 75.17$\pm$2.12 & 77.63$\pm$2.44 & 73.47$\pm$2.49 & 77.10$\pm$1.54 & 77.31$\pm$1.76 & \textbf{78.04$\pm$2.17} \\
deer.\ vs horse. & 81.52$\pm$2.21 & 80.62$\pm$2.49 & \textbf{84.73$\pm$2.50} & 76.50$\pm$1.57 & 82.55$\pm$3.86 & 82.31$\pm$1.85 & 84.26$\pm$1.67 \\
deer.\ vs monkey. & 78.87$\pm$2.18 & 78.87$\pm$3.02 & 83.02$\pm$1.96 & 82.28$\pm$1.96 & 78.22$\pm$2.02 & 81.58$\pm$1.92 & \textbf{83.64$\pm$2.66} \\
deer.\ vs ship. & 95.04$\pm$1.37 & 96.69$\pm$1.47 & 97.26$\pm$1.54 & 94.65$\pm$1.40 & 94.45$\pm$1.34 & 93.44$\pm$3.02 & \textbf{97.69$\pm$1.76} \\
deer.\ vs truck. & 95.90$\pm$1.10 & 96.74$\pm$1.26 & 97.63$\pm$0.72 & 94.57$\pm$1.16 & 94.92$\pm$1.28 & 92.91$\pm$2.11 & \textbf{97.90$\pm$0.66} \\
dog.\ vs horse. & 71.63$\pm$2.61 & 69.94$\pm$2.12 & 73.61$\pm$3.19 & 73.36$\pm$2.04 & 70.41$\pm$3.76 & 72.78$\pm$2.83 & \textbf{73.67$\pm$3.66} \\
dog.\ vs monkey. & 64.40$\pm$2.18 & 66.21$\pm$4.19 & 68.74$\pm$3.17 & \textbf{69.05$\pm$4.06} & 65.50$\pm$4.88 & 60.05$\pm$5.06 & 67.97$\pm$3.07 \\
dog.\ vs ship. & \textbf{96.11$\pm$1.49} & 94.08$\pm$1.07 & 93.15$\pm$1.09 & 93.96$\pm$1.27 & 92.62$\pm$0.95 & 90.84$\pm$1.00 & 95.09$\pm$1.37 \\
dog.\ vs truck. & 93.32$\pm$0.65 & 94.55$\pm$0.96 & 94.76$\pm$1.48 & 94.20$\pm$2.22 & 93.07$\pm$1.58 & \textbf{96.63$\pm$1.37} & 93.95$\pm$1.30 \\
horse.\ vs monkey. & 70.66$\pm$4.62 & 77.06$\pm$2.83 & \textbf{81.24$\pm$3.44} & 78.64$\pm$1.85 & 77.05$\pm$3.14 & 69.45$\pm$2.55 & 81.21$\pm$3.71 \\
horse.\ vs ship. & 93.40$\pm$1.21 & 94.18$\pm$1.34 & \textbf{95.80$\pm$1.31} & 93.92$\pm$1.37 & 92.63$\pm$0.65 & 92.28$\pm$0.95 & 94.36$\pm$2.21 \\
horse.\ vs truck. & 92.18$\pm$2.14 & 93.70$\pm$0.72 & 91.88$\pm$1.53 & 92.89$\pm$3.00 & 91.71$\pm$1.05 & 90.77$\pm$2.41 & \textbf{96.21$\pm$1.52} \\
monkey.\ vs ship. & 95.50$\pm$1.92 & 95.70$\pm$1.45 & \textbf{96.97$\pm$1.28} & 95.49$\pm$1.11 & 95.20$\pm$0.79 & 92.94$\pm$1.16 & 96.43$\pm$0.79 \\
monkey.\ vs truck. & 95.45$\pm$1.20 & 95.63$\pm$1.22 & 94.95$\pm$1.40 & \textbf{95.86$\pm$0.61} & 94.75$\pm$0.98 & 94.36$\pm$1.90 & 94.78$\pm$0.68 \\
ship.\ vs truck. & 72.85$\pm$3.16 & 79.58$\pm$1.38 & 76.94$\pm$1.68 & 75.36$\pm$3.86 & 75.29$\pm$1.63 & 76.67$\pm$2.31 & \textbf{80.22$\pm$1.35} \\
\midrule
Avg.~F1 & 84.42 & 86.12 & 87.60 & 85.96 & 85.70 & 85.94 & \textbf{88.29} \\
Avg.~Rank & 5.40 & 4.02 & 2.71 & 4.47 & 5.13 & 4.47 & \textbf{1.80} \\
\bottomrule
\end{tabular}
}
\end{table*}

\subsubsection{Parameter sensitivity analysis}
\label{subsubsec:mv-param}

Here, we use the UCI datasets to examine the parameter 
sensitivity of Mv$L_{\mathrm{ht}}$-SVM in order to analyze the impact of 
parameters. The accuracy of Mv$L_{\mathrm{ht}}$-SVM on the Heart dataset is 
shown in 
Figs.~\ref{fig:param-fix-alpha}--\ref{fig:param-fix-sigma}, 
where the color of the dot indicates the accuracy for the 
relevant parameter combination. We fix four parameters separately 
and explore combinations of the remaining three varying 
parameters. These graphs show that the accuracy of Mv$L_{\mathrm{ht}}$-SVM 
is sensitive to all four parameters on all 
datasets. The number of SVs in our 
algorithm is mostly controlled by the values of $C$ and $\alpha$. 
This emphasizes the importance of carefully adjusting these 
parameters for optimal performance.

\begin{figure*}[!htbp]
\centering
\includegraphics[width=0.85\textwidth]{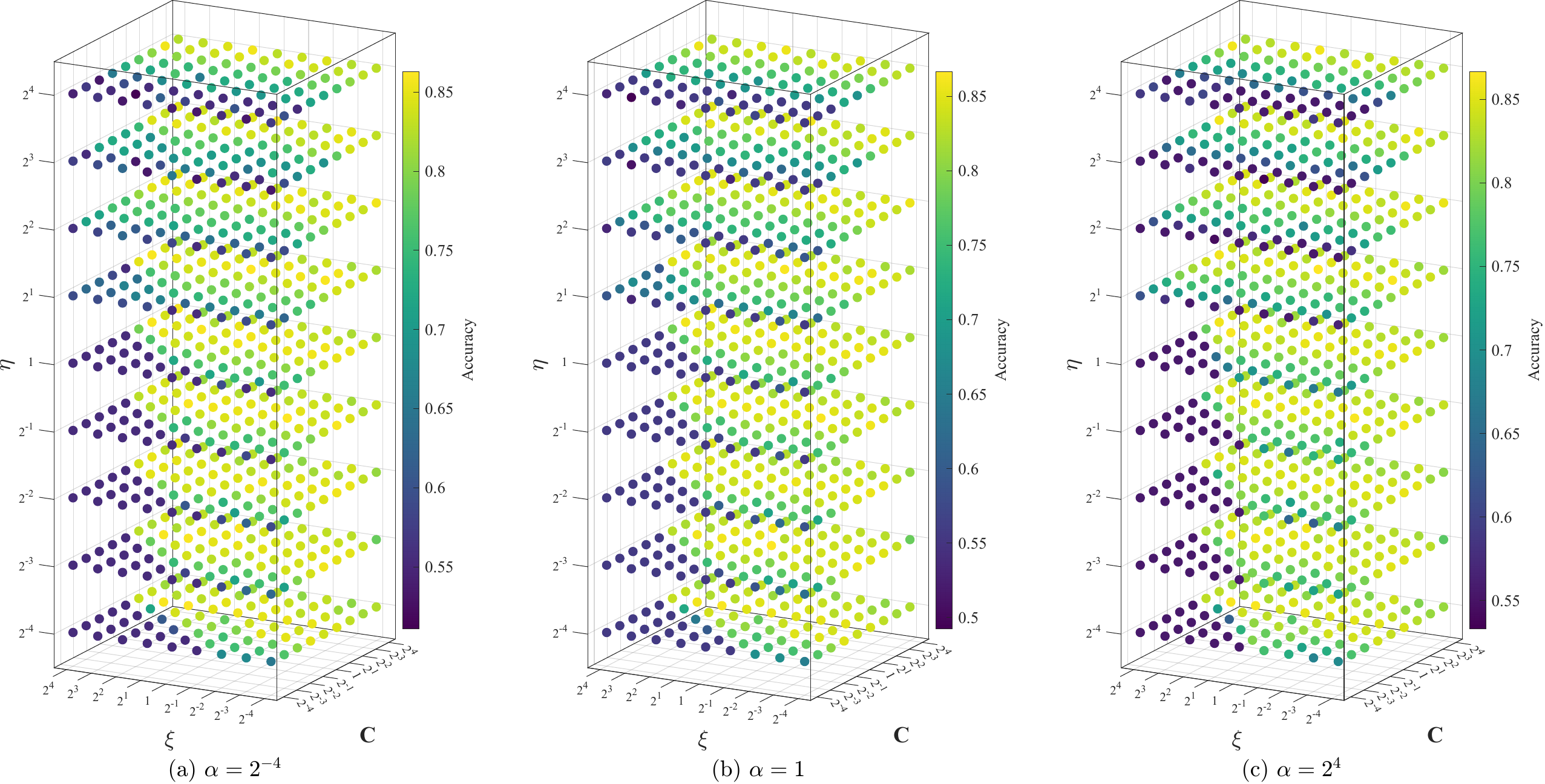}
\caption{Parameter sensitivity on Heart: fix $\alpha$ at 
$2^{-4}$, $1$, $2^4$.}
\label{fig:param-fix-alpha}
\end{figure*}

\begin{figure*}[!htbp]
\centering
\includegraphics[width=0.85\textwidth]{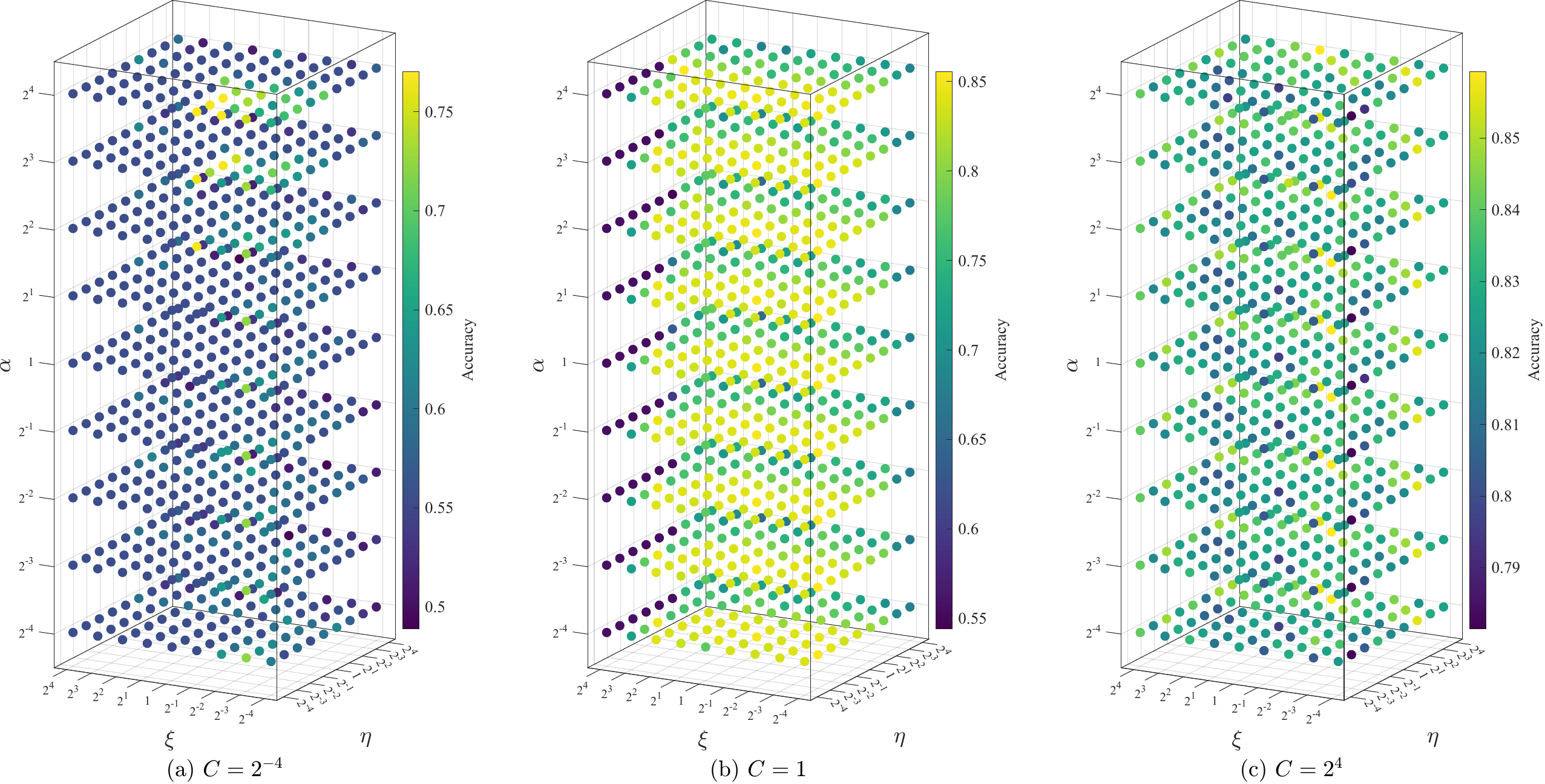}
\caption{Parameter sensitivity on Heart: fix $C$ at 
$2^{-4}$, $1$, $2^4$.}
\label{fig:param-fix-C}
\end{figure*}

\begin{figure*}[!htbp]
\centering
\includegraphics[width=0.85\textwidth]{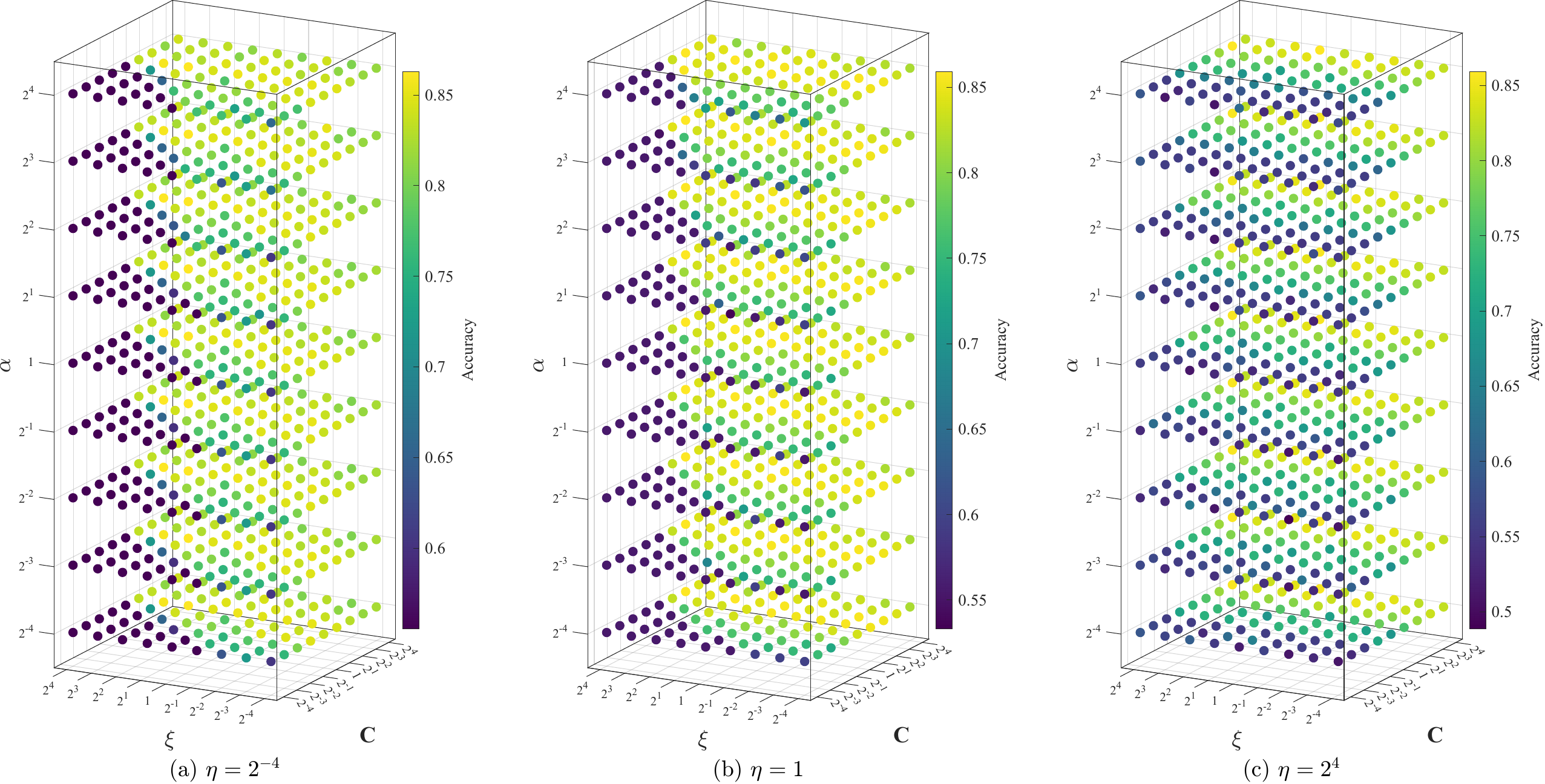}
\caption{Parameter sensitivity on Heart: fix $\eta$ at 
$2^{-4}$, $1$, $2^4$.}
\label{fig:param-fix-eta}
\end{figure*}

\begin{figure*}[!htbp]
\centering
\includegraphics[width=0.85\textwidth]{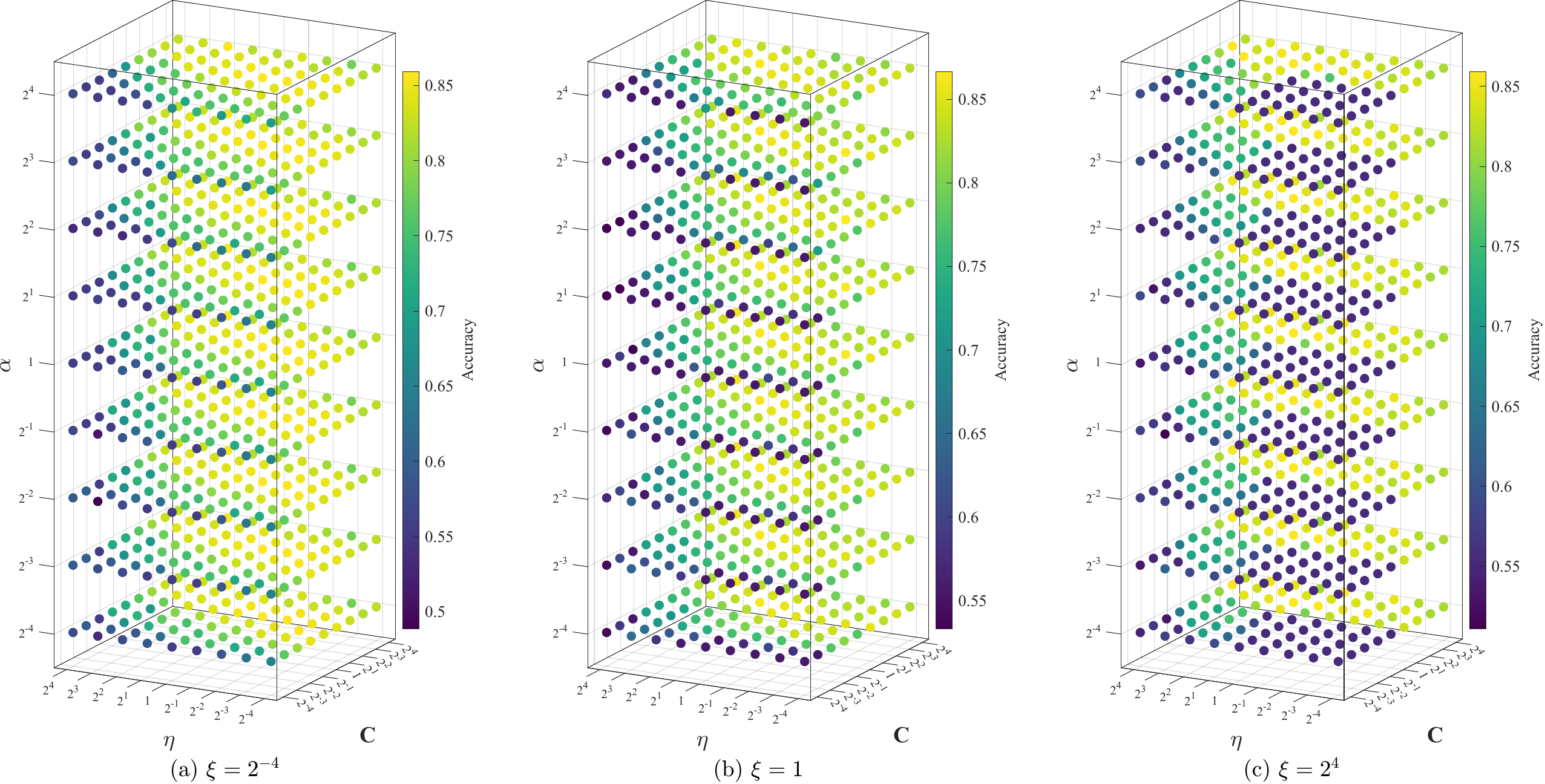}
\caption{Parameter sensitivity on Heart: fix $\xi$ at 
$2^{-4}$, $1$, $2^4$.}
\label{fig:param-fix-sigma}
\end{figure*}

\subsubsection{Ablation study}
\label{subsubsec:mv-ablation}

In order to verify the influence of consistency term and 
complementarity term in Mv$L_{\mathrm{ht}}$-SVM, we compare the performance of 
Mv$L_{\mathrm{ht}}$-SVM in different cases on the first 10 STL-10 datasets: 
(1)~Mv$L_{\mathrm{ht}}$-SVM as in the proposed model; 
(2)~Mv$L_{\mathrm{ht}}$-SVM with the parameter $\eta = 0$, that is, the model 
that only satisfies the principle of complementarity; 
(3)~Mv$L_{\mathrm{ht}}$-SVM with the parameter $\theta_v = 1/V$, that is, the 
model that only satisfies the principle of consistency. In this 
experiment, except for the fixed parameters of Mv$L_{\mathrm{ht}}$-SVM, the rest 
of the parameters adopt the optimal parameters to train and test 
the corresponding model.

The experimental results are shown in 
Fig.~\ref{fig:ablation-bar}, where the horizontal axis is the 
dataset index and the vertical axis is the accuracy of the model 
in the corresponding case on the dataset. It can be clearly seen 
that the Mv$L_{\mathrm{ht}}$-SVM in the first case performs best on all datasets. 
In addition, the Mv$L_{\mathrm{ht}}$-SVM in the third case is usually more 
competitive than the Mv$L_{\mathrm{ht}}$-SVM in the second case. This shows that 
the consistency principle is more instructive in Mv$L_{\mathrm{ht}}$-SVM, and 
satisfying both principles at the same time can further improve 
the learning effect.

\begin{figure}[!htbp]
\centering
\includegraphics[width=0.5\textwidth]{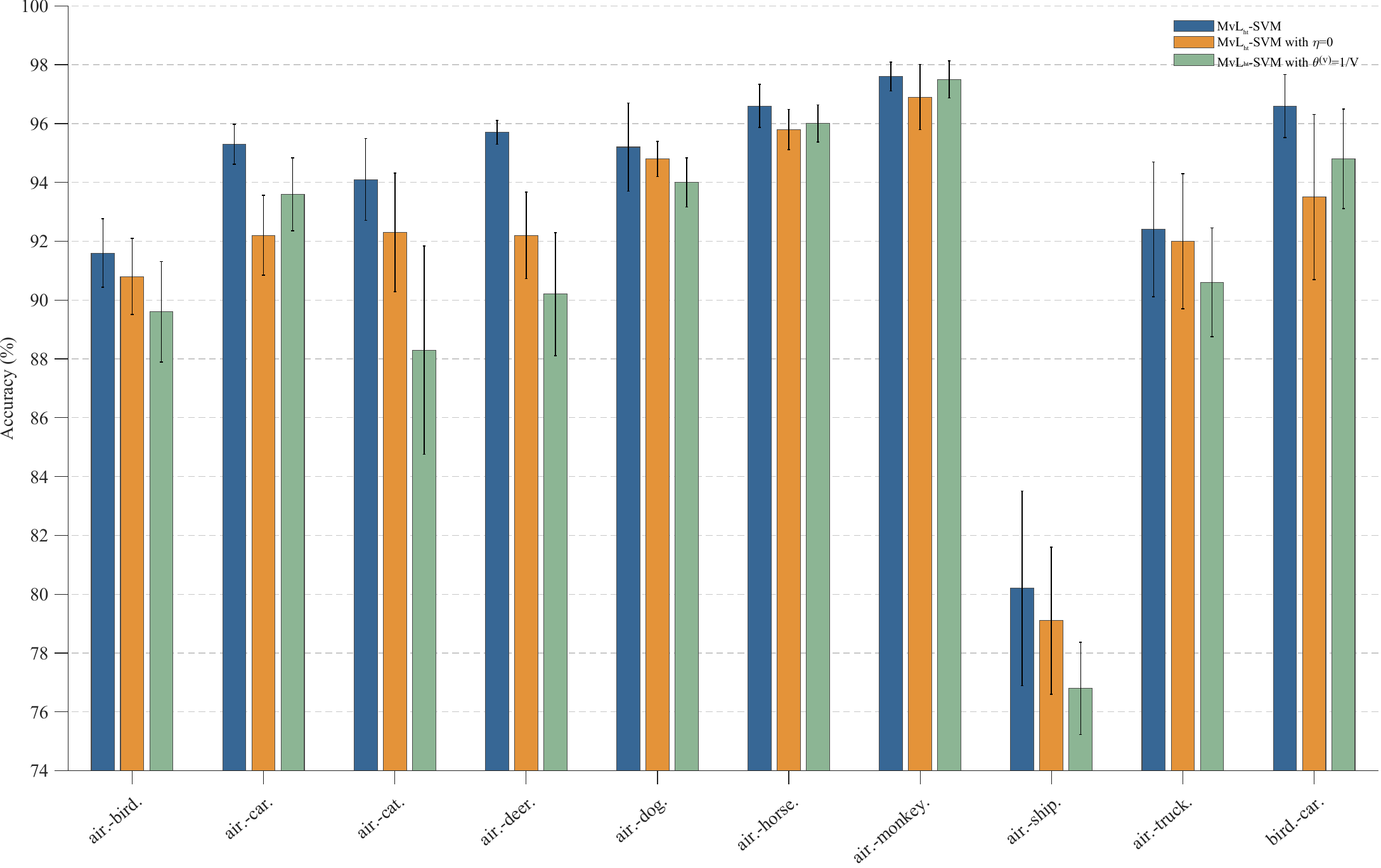}
\caption{Ablation experiments on the first 10 STL-10 datasets.}
\label{fig:ablation-bar}
\end{figure}

\section{Conclusion and future work}
\label{sec:conclusion}

This paper has introduced the hybrid truncated ($L_{\mathrm{ht}}$) 
loss and built two classifiers upon it. For single-view problems, 
$L_{\mathrm{ht}}$-SVM is formulated and solved by an alternating direction method of multipliers (ADMM) 
algorithm whose working-set rule confines each iteration to a 
small active subset, markedly reducing per-iteration cost and 
overall support-vector count. First-order necessary and sufficient 
optimality conditions are established through the P-stationary 
point framework. For multi-view problems, 
Mv$L_{\mathrm{ht}}$-SVM integrates structural covariance 
information and adaptive view weighting, fulfilling both the 
consensus and complementarity principles. Experiments on 
benchmark, synthetic, and image datasets demonstrate that 
$L_{\mathrm{ht}}$-SVM delivers competitive accuracy with 
considerably fewer support vectors and superior noise tolerance, 
while Mv$L_{\mathrm{ht}}$-SVM consistently surpasses six 
multi-view baselines across all evaluation metrics. The 
efficiency advantage grows with dataset scale.

Several directions remain open. 
(i)~Establishing an explicit convergence rate for the 
proposed ADMM scheme. 
(ii)~Extending the $L_{\mathrm{ht}}$ framework to kernel-based 
non-linear SVMs, as well as exploring connections with deep 
learning and logistic regression formulations. 
(iii)~Generalising the approach to multi-class settings where 
the computational and robustness benefits are expected to persist.


\section*{CRediT authorship contribution statement}

\textbf{Yuliang Yang:} Methodology, Software, Validation, Formal analysis,
Investigation, Data curation, Writing -- original draft, Visualization.
\textbf{Chen Chen:} Validation, Writing -- review \& editing.
\textbf{Yuxiang Liu:} Formal analysis, Writing -- review \& editing.
\textbf{Huiru Wang:} Conceptualization, Writing -- review \& editing,
Supervision, Project administration, Funding acquisition.

\section*{Data availability}

Data will be made available on request.

\section*{Declaration of competing interest}

The authors declare that they have no known competing financial interests 
or personal relationships that could have appeared to influence the work 
reported in this paper.

\section*{Acknowledgments}
This work was supported by the National Natural Science Foundation of China 
(No. 12501719) and the Fundamental Research Funds for the Central Universities 
(No. QNTD202510) and Beijing Training Program of Innovation and Entrepreneurship for Undergraduates (No. 202510022420).

\appendix
\section{Proof of Lemma~\ref{lem:prox-1d}}
\label{app:prox-proof}

\noindent\textbf{Proof.} Let $\kappa := \nu C$. According to \eqref{eq:prox-def}, 
we have that $\mathrm{prox}_{\kappa\ell_{\mathrm{ht}}}(z)$ is the minimizer of 
the following function
\begin{equation}
\varphi(s) := 
\begin{cases}
\varphi_1(s) := \frac{(s-z)^2}{2}, & s \le 0,\\[3pt]
\varphi_2(s) := \frac{6\kappa s}{5} + \frac{(s-z)^2}{2}, & s \in (0, \frac{2}{3}],\\[3pt]
\varphi_3(s) := \frac{\kappa(-9s^2+18s-4)}{5} + \frac{(s-z)^2}{2}, & s \in (\frac{2}{3}, 1],\\[3pt]
\varphi_4(s) := \kappa + \frac{(s-z)^2}{2}, & s > 1.
\end{cases}
\label{eq:varphi}
\end{equation}
It can be clearly seen that the minimizers of $\varphi_1(s)$, $\varphi_2(s)$, 
$\varphi_3(s)$, $\varphi_4(s)$ are yielded at $s_1^* = z$, 
$s_2^* = z - \frac{6\kappa}{5}$, 
$s_3^* = \frac{5z-18\kappa}{5-18\kappa}$, $s_4^* = z$ 
and $s_5^* = 0$ respectively.

\noindent (i) For $\kappa \in (0, 5/18)$, by comparing the values of 
$\varphi(s_1^*)$, $\varphi(s_2^*)$, $\varphi(s_3^*)$, $\varphi(s_4^*)$ and 
$\varphi(s_5^*)$, we get conclusion: (a1) As $z > 1$, we have 
$\varphi(s_4^*) < \varphi(s_i^*)$ for $i \ne 4$, which obtains $s^* = s_4^* = z$. 
(a2) As $z \in (\frac{2}{3} + \frac{6\kappa}{5}, 1]$, we have 
$\varphi(s_3^*) < \varphi(s_i^*)$ for $i \ne 3$, which implies 
$s^* = s_3^* = \frac{5z-18\kappa}{5-18\kappa}$. 
(a3) As $z \in (\frac{6\kappa}{5}, \frac{2}{3} + \frac{6\kappa}{5}]$, 
we have $\varphi(s_2^*) < \varphi(s_5^*)$, which means 
$s^* = s_2^* = z - \frac{6\kappa}{5}$. 
(a4) As $z \in [0, \frac{6\kappa}{5}]$, we have 
$\varphi(s_5^*) \le \varphi(s_2^*)$, which means $s^* = s_5^* = 0$. 
(a5) As $z < 0$, we have $\varphi(s_1^*) < \varphi(s_5^*)$, which gets 
$s^* = s_1^* = z$.
Summarizing the above analysis, we obtain \eqref{eq:prox-case1}.

\noindent (ii) For $\kappa \in [5/18, 25/18)$, by comparing the values of 
$\varphi(s_1^*)$, $\varphi(s_2^*)$, $\varphi(s_4^*)$ and $\varphi(s_5^*)$, we get 
conclusion: (b1) As $z > \frac{5}{6} + \frac{3\kappa}{5}$, we have 
$\varphi(s_4^*) < \varphi(s_2^*)$, which obtains $s^* = s_4^* = z$. 
(b2) As $z \in (\frac{6\kappa}{5}, \frac{5}{6} + \frac{3\kappa}{5}]$, 
we have $\varphi(s_2^*) \le \varphi(s_4^*)$, which obtains 
$s^* = s_2^* = z - \frac{6\kappa}{5}$. 
(b3) As $z \in [0, \frac{6\kappa}{5}]$, we have 
$\varphi(s_5^*) \le \varphi(s_2^*)$, which yields $s^* = s_5^* = 0$. 
(b4) As $z < 0$, we have $\varphi(s_1^*) < \varphi(s_5^*)$, which gets 
$s^* = s_1^* = z$.
From the above (b1)--(b4), we obtain \eqref{eq:prox-case2}.

\noindent (iii) For $\kappa \ge 25/18$, by contrasting the values of 
$\varphi(s_1^*)$, $\varphi(s_4^*)$ and $\varphi(s_5^*)$, we obtain conclusion: 
(c1) As $z > \sqrt{2\kappa}$, we have $\varphi(s_4^*) < \varphi(s_5^*)$, 
which implies $s^* = s_4^* = z$. 
(c2) As $z = \sqrt{2\kappa}$, we have $\varphi(s_4^*) = \varphi(s_5^*)$, 
which gets $s^* = z$ or $s^* = 0$. 
(c3) As $z \in [0, \sqrt{2\kappa})$, we have $\varphi(s_5^*) < \varphi(s_4^*)$, 
which means $s^* = s_5^* = 0$. 
(c4) As $z < 0$, we have $\varphi(s_1^*) < \varphi(s_5^*)$, which obtains 
$s^* = s_1^* = z$.
By the (c1)--(c4), we get \eqref{eq:prox-case3}, which completes the proof. 
\hfill$\square$

\section{Proof of Theorem~\ref{thm:necessary}}
\label{app:necessary}

\noindent\textbf{Proof.} Assume $(\mathbf{w}^*; b^*; \mathbf{q}^*)$ to be a 
local optimal solution of \eqref{eq:ht-svm-constrained}. Following variational 
analysis \citep{rockafellar2009variational}, there exists a multiplier 
$\boldsymbol{\psi}^* \in \mathbb{R}^m$ satisfying
\begin{equation}
\begin{cases}
\mathbf{w}^* + N^\top \boldsymbol{\psi}^* = \mathbf{0},\\
\langle \mathbf{y}, \boldsymbol{\psi}^* \rangle = 0,\\
\mathbf{q}^* + N\mathbf{w}^* + b^*\mathbf{y} = \mathbf{1},\\
\boldsymbol{\psi}^* + C\partial L_{\mathrm{ht}}(\mathbf{q}^*) \ni \mathbf{0}.
\end{cases}
\label{eq:kkt-necessary}
\end{equation}
Combining \eqref{eq:p-stationary} with \eqref{eq:kkt-necessary}, it suffices to 
verify that whenever $0 < \nu \le \nu^*$ and $(\boldsymbol{\psi}^*; \mathbf{q}^*)$ 
satisfies $\mathbf{0} \in \boldsymbol{\psi}^* + C\partial L_{\mathrm{ht}}(\mathbf{q}^*)$, 
the inclusion $\mathbf{q}^* \in \mathrm{prox}_{\nu C L_{\mathrm{ht}}}(\mathbf{q}^* - \nu\boldsymbol{\psi}^*)$ holds. 
Invoking Lemma~\ref{lem:subdiff} alongside \eqref{eq:kkt-necessary}, we deduce
\begin{equation}
\psi_k^* \in 
\begin{cases}
0, & q_k^* \ge 1,\\
-18C(1-q_k^*)/5, & q_k^* \in (2/3, 1),\\
-6C/5, & q_k^* \in (0, 2/3],\\
[-6C/5, 0], & q_k^* = 0,\\
0, & q_k^* < 0.
\end{cases}
\label{eq:psi-cases}
\end{equation}
Utilizing Lemma~\ref{lem:prox-1d}, we establish \eqref{eq:p-stationary} by examining 
three regimes: $\nu C \in (0, 5/18)$, $\nu C \in [5/18, 25/18)$, and $\nu C \ge 25/18$.

\textbf{Case (i):} Consider $\nu C \in (0, 5/18)$. Set $\mathbf{p}^* := \mathbf{q}^* - \nu\boldsymbol{\psi}^* \in \mathbb{R}^m$.

(i) For $k \in \mathcal{A}^*$, one has $q_k^* > 1$ with $\psi_k^* = 0$ via \eqref{eq:psi-cases}, 
leading to
\begin{equation}
p_k^* := q_k^* - \nu\psi_k^* = q_k^* > 1.
\label{eq:case1-A}
\end{equation}

(ii) For $k \in \mathcal{B}^*$, one has $q_k^* \in (2/3, 1]$ with 
$\psi_k^* = -18C(1-q_k^*)/5$ via \eqref{eq:psi-cases}. Consequently,
\begin{equation}
p_k^* = q_k^* - \nu\psi_k^* = q_k^* + 18\nu C(1-q_k^*)/5.
\label{eq:case1-B}
\end{equation}
Given $\nu \le \nu_1^* = 5/(18C)$, it follows that $18\nu C \le 5$; combining with 
\eqref{eq:case1-B} yields $p_k^* \in (2/3 + 6\nu C/5, 1]$.

(iii) For $k \in \mathcal{C}^*$, one has $q_k^* \in (0, 2/3]$ with $\psi_k^* = -6C/5$ 
via \eqref{eq:psi-cases}. Hence,
\begin{equation}
p_k^* = q_k^* - \nu\psi_k^* = q_k^* + 6\nu C/5.
\label{eq:case1-C}
\end{equation}
Together with \eqref{eq:case1-C}, this implies $p_k^* \in (6\nu C/5, 2/3 + 6\nu C/5]$.

(iv) For $k \in \mathcal{D}^*$, one has $q_k^* = 0$ with $\psi_k^* \in [-6C/5, 0]$ 
via \eqref{eq:psi-cases}, so that
\begin{equation}
p_k^* = q_k^* - \nu\psi_k^* = -\nu\psi_k^* \in [0, 6\nu C/5].
\label{eq:case1-D}
\end{equation}

(v) For $k \in \mathcal{E}^*$, one has $q_k^* < 0$ with $\psi_k^* = 0$, giving
\begin{equation}
p_k^* = q_k^* - \nu\psi_k^* = q_k^* < 0.
\label{eq:case1-E}
\end{equation}

Summarizing (i)--(v), we arrive at \eqref{eq:prox-case1}, finishing Case (i).

\textbf{Case (ii):} Let $\nu C \in [5/18, 25/18)$, 
$\mathbf{p}^* := \mathbf{q}^* - \nu\boldsymbol{\psi}^*$, 
$\widetilde{\beta} := 6\nu C/5$, and 
$\widehat{\beta} := 5/6 + 3\nu C/5$. We partition $[m]$ 
into $\mathcal{R}^* \cup \mathcal{Q}^* \cup \mathcal{P}^* 
\cup \mathcal{D}^* \cup \mathcal{E}^*$ and verify five sub-cases.

(i) For $k \in \mathcal{R}^*$ (i.e., $q_k^* > 5/6$), the 
restriction $\nu \le \nu_2^* = 5(6q_k^*-5)/(18C)$ yields 
$3\nu C/5 \le (6q_k^*-5)/6$, hence 
$q_k^* \ge 5/6 + 3\nu C/5 = \widehat{\beta}$. 
Noting $\widehat{\beta} \ge 1$ when $\nu C \ge 5/18$, we 
deduce $q_k^* \ge 1$ and $\psi_k^* = 0$ by \eqref{eq:psi-cases}. 
Therefore $p_k^* = q_k^* > \widehat{\beta}$, and 
\eqref{eq:prox-case2} returns $\mathrm{prox}(p_k^*) = p_k^* = q_k^*$.

(ii) For $k \in \mathcal{Q}^*$ (i.e., $q_k^* = 5/6$), observe 
that $5/6 \in [2/3, 1]$, so $k \in \mathcal{B}^*$. 
However, $\mathcal{B}^* \neq \emptyset$ would force 
$\nu \le \nu_1^* = 5/(18C)$, i.e., $\nu C \le 5/18$, 
contradicting the assumption $\nu C \ge 5/18$. 
Consequently, $\mathcal{Q}^*$ is empty in this regime.

(iii) For $k \in \mathcal{P}^*$ (i.e., $q_k^* \in (0, 5/6)$), 
since $\mathcal{B}^* = \emptyset$ (as shown above), all such 
$q_k^*$ satisfy $q_k^* \le 2/3$, giving 
$\psi_k^* = -6C/5$ via \eqref{eq:psi-cases} and 
$p_k^* = q_k^* + \widetilde{\beta}$. The restriction 
$\nu \le \nu_3^* = 5(5-6q_k^*)/(18C)$ ensures 
$q_k^* \le 5/6 - 3\nu C/5$, from which 
$p_k^* \le \widehat{\beta}$. Combined with 
$p_k^* > \widetilde{\beta}$, we obtain 
$p_k^* \in (\widetilde{\beta}, \widehat{\beta}]$, so 
\eqref{eq:prox-case2} delivers 
$\mathrm{prox}(p_k^*) = p_k^* - \widetilde{\beta} = q_k^*$.

(iv) For $k \in \mathcal{D}^*$ (i.e., $q_k^* = 0$), 
\eqref{eq:psi-cases} provides $\psi_k^* \in [-6C/5, 0]$, 
leading to $p_k^* = -\nu\psi_k^* \in [0, \widetilde{\beta}]$. 
Then \eqref{eq:prox-case2} returns $\mathrm{prox}(p_k^*) = 0 = q_k^*$.

(v) For $k \in \mathcal{E}^*$ (i.e., $q_k^* < 0$), we have 
$\psi_k^* = 0$ and $p_k^* = q_k^* < 0$, so 
\eqref{eq:prox-case2} gives $\mathrm{prox}(p_k^*) = p_k^* = q_k^*$.

Collecting (i)--(v) establishes \eqref{eq:prox-case2} for Case (ii).

\textbf{Case (iii):} Let $\nu C \ge 25/18$, 
$\mathbf{p}^* := \mathbf{q}^* - \nu\boldsymbol{\psi}^*$, and 
$\gamma := \sqrt{2\nu C}$. As in Case~(ii), 
$\mathcal{B}^* = \emptyset$ follows from $\nu \le \nu_1^*$. 
Furthermore, $\mathcal{P}^* \neq \emptyset$ would imply 
$\nu \le \nu_3^* \le 5/(18C)$ and thus $\nu C \le 5/18 < 25/18$, 
a contradiction, so $\mathcal{P}^* = \emptyset$. Because 
$\mathcal{Q}^* \subset \mathcal{B}^*$, we also have 
$\mathcal{Q}^* = \emptyset$. Hence every $q_k^* > 0$ 
belongs to $\mathcal{R}^*$. Three sub-cases remain.

(i) For $k \in \mathcal{R}^*$ (i.e., $q_k^* > 0$), the bound 
$\nu \le \nu_5^*$ provides $q_k^* \ge \gamma$. Since 
$\gamma \ge \sqrt{2 \cdot 25/18} = 5/3 > 1$, 
\eqref{eq:psi-cases} yields $\psi_k^* = 0$ and 
$p_k^* = q_k^* \ge \gamma$. When $p_k^* > \gamma$, 
\eqref{eq:prox-case3} gives $\mathrm{prox}(p_k^*) = p_k^* = q_k^*$; 
when $p_k^* = \gamma$, we have 
$q_k^* = \gamma \in \mathrm{prox}(p_k^*) = \{\gamma, 0\}$.

(ii) For $k \in \mathcal{D}^*$ (i.e., $q_k^* = 0$), 
\eqref{eq:psi-cases} supplies $\psi_k^* \in [-6C/5, 0]$ and 
$p_k^* = -\nu\psi_k^* \in [0, 6\nu C/5]$. Provided 
$\mathcal{D}^* \neq \emptyset$, the bound 
$\nu \le \nu_4^* = 25/(18C)$ ensures 
$6\nu C/5 \le 5/3 \le \gamma$, placing $p_k^* \in [0, \gamma]$. 
By \eqref{eq:prox-case3}, $\mathrm{prox}(p_k^*)$ includes 
$0 = q_k^*$.

(iii) For $k \in \mathcal{E}^*$ (i.e., $q_k^* < 0$), 
$\psi_k^* = 0$ and $p_k^* = q_k^* < 0$. By 
\eqref{eq:prox-case3}, $\mathrm{prox}(p_k^*) = p_k^* = q_k^*$.

Collecting (i)--(iii) establishes \eqref{eq:prox-case3} and 
completes the proof. \hfill$\square$

\section{Proof of Theorem~\ref{thm:sufficient}}
\label{app:sufficient}

\noindent\textbf{Proof.} Let $\Sigma^* := (\mathbf{w}^*; b^*; \mathbf{q}^*)$ and define 
$\Pi := \{\Sigma = (\mathbf{w}; b; \mathbf{q}) : \mathbf{q} + N\mathbf{w} + b\mathbf{y} = \mathbf{1}\}$. 
When $\Sigma \in \Pi \cap U(\Sigma^*, \widetilde{\varrho})$ for some $\widetilde{\varrho} > 0$, 
the equality constraint $\mathbf{q} + N\mathbf{w} + b\mathbf{y} = \mathbf{1}$ together with 
\eqref{eq:p-stationary} leads to
\begin{equation}
-N(\mathbf{w} - \mathbf{w}^*) = (b - b^*)\mathbf{y} + \mathbf{q} - \mathbf{q}^*.
\label{eq:suff-A7}
\end{equation}
Since $\|\mathbf{w}\|^2$ is convex, we can lead to
\begin{align}
\|\mathbf{w}\|^2 - \|\mathbf{w}^*\|^2 
&\ge 2\langle \mathbf{w} - \mathbf{w}^*, \mathbf{w}^* \rangle \notag\\
&\overset{\eqref{eq:p-stationary}}{=} 
-2\langle N(\mathbf{w} - \mathbf{w}^*), \boldsymbol{\psi}^* \rangle \notag\\
&\overset{\eqref{eq:suff-A7}}{=} 
2\langle \boldsymbol{\psi}^*, \mathbf{q} - \mathbf{q}^* \rangle 
+ 2(b - b^*)\langle \boldsymbol{\psi}^*, \mathbf{y} \rangle \notag\\
&\overset{\eqref{eq:p-stationary}}{=} 
2\langle \boldsymbol{\psi}^*, \mathbf{q} - \mathbf{q}^* \rangle.
\label{eq:suff-A8}
\end{align}
According to \eqref{eq:suff-A7} and \eqref{eq:suff-A8}, we now show that $\Sigma^*$ is 
a local minimizer of \eqref{eq:ht-svm-constrained}. For any 
$\Sigma \in \Pi \cap U(\Sigma^*, \varrho)$ with $\varrho > 0$, the target inequality is:
\begin{equation}
\frac{1}{2}\|\mathbf{w}\|^2 + CL_{\mathrm{ht}}(\mathbf{q}) \ge 
\frac{1}{2}\|\mathbf{w}^*\|^2 + CL_{\mathrm{ht}}(\mathbf{q}^*).
\label{eq:suff-A9}
\end{equation}
We verify \eqref{eq:suff-A9} by analyzing three regimes of $\nu C$: $(0, 5/18)$, 
$[5/18, 25/18)$, and $[25/18, +\infty)$.

\textbf{Case (i):} Suppose $\nu C \in (0, 5/18)$. Define $\mathbf{p}^* := \mathbf{q}^* - \nu\boldsymbol{\psi}^*$ and partition $[m]$ into five subsets:
\begin{align*}
&\mathbb{H}_1^* := \{i \in [m] : p_i^* < 0\}, \quad 
\mathbb{H}_2^* := \{i \in [m] : p_i^* \in [0, \widetilde{\beta}]\},\\
&\mathbb{H}_3^* := \{i \in [m] : p_i^* \in (\widetilde{\beta}, \widehat{\beta})\},\\
&\mathbb{H}_4^* := \{i \in [m] : p_i^* \in [\widehat{\beta}, 1)\}, \quad 
\mathbb{H}_5^* := \{i \in [m] : p_i^* \ge 1\},
\end{align*}
where $\widetilde{\beta} := \frac{6\nu C}{5}$ and $\widehat{\beta} := \frac{2}{3} + \frac{6\nu C}{5}$. 
From the proximal characterization $\mathbf{q}^* \in \mathrm{prox}_{\nu CL_{\mathrm{ht}}}(\mathbf{p}^*)$ 
combined with \eqref{eq:prox-case1} and \eqref{eq:psi-cases}, we deduce
\begin{align}
&\psi^*_{\mathbb{H}_1^*} = \mathbf{0}_{\mathbb{H}_1^*}, \quad 
\mathbf{q}^*_{\mathbb{H}_2^*} = \mathbf{0}_{\mathbb{H}_2^*}, \quad 
\psi^*_{\mathbb{H}_3^*} = -\frac{6C}{5}\mathbf{1}_{\mathbb{H}_3^*}, \notag\\
&\psi^*_{\mathbb{H}_4^*} = -\frac{18C(1 - \mathbf{q}^*)_{\mathbb{H}_4^*}}{5}, \quad 
\psi^*_{\mathbb{H}_5^*} = \mathbf{0}_{\mathbb{H}_5^*}.
\label{eq:suff-A11}
\end{align}
Together with \eqref{eq:psi-cases}, this implies
\begin{equation}
\begin{aligned}
&\psi_i^* = 0, q_i^* < 0, \; i \in \mathbb{H}_1^*,\\
&\psi_i^* \in [-\frac{6C}{5}, 0], q_i^* = 0, \; i \in \mathbb{H}_2^*,\\
&\psi_i^* = -\frac{6C}{5}, q_i^* \in (0, \frac{2}{3}), \; i \in \mathbb{H}_3^*,\\
&\psi_i^* = -\frac{18C(1 - q_i^*)}{5}, q_i^* \in [\frac{2}{3}, 1), \; i \in \mathbb{H}_4^*,\\
&\psi_i^* = 0, q_i^* \ge 1, \; i \in \mathbb{H}_5^*.
\end{aligned}
\label{eq:suff-A12}
\end{equation}
To confirm \eqref{eq:suff-A9}, introduce 
$\overline{\mathbb{H}}^* := \mathbb{H}_1^* \cup \mathbb{H}_5^*$ and 
$\mathbb{H}^* := \mathbb{H}_2^* \cup \mathbb{H}_3^* \cup \mathbb{H}_4^*$. 
Noting that $\psi_i^* = 0$ for $i \in \overline{\mathbb{H}}^*$, 
it suffices by \eqref{eq:suff-A8} to prove:
\begin{align}
&CL_{\mathrm{ht}}(\mathbf{q}_{\mathbb{H}^*}) - CL_{\mathrm{ht}}(\mathbf{q}^*_{\mathbb{H}^*}) + 
\langle \boldsymbol{\psi}^*_{\mathbb{H}^*}, \mathbf{q}_{\mathbb{H}^*} - \mathbf{q}^*_{\mathbb{H}^*} \rangle \ge 0,
\label{eq:suff-A13}\\
&CL_{\mathrm{ht}}(\mathbf{q}_{\overline{\mathbb{H}}^*}) - CL_{\mathrm{ht}}(\mathbf{q}^*_{\overline{\mathbb{H}}^*}) \ge 0.
\label{eq:suff-A14}
\end{align}

\textit{Verification of \eqref{eq:suff-A14}}: For $i \in \mathbb{H}_1^*$, we have 
$q_i^* < 0$. There exists $\varrho_1 > 0$ such that $q_i < 0$ for all 
$\Sigma \in \Pi \cap U(\Sigma^*, \varrho_1)$, yielding $\ell_{\mathrm{ht}}(q_i) = \ell_{\mathrm{ht}}(q_i^*) = 0$.
For $i \in \mathbb{H}_5^*$, we have $q_i^* \ge 1$. There exists $\varrho_5 > 0$ such that 
$q_i \ge 1$ for all $\Sigma \in \Pi \cap U(\Sigma^*, \varrho_5)$, giving 
$\ell_{\mathrm{ht}}(q_i) = \ell_{\mathrm{ht}}(q_i^*) = 1$. In both cases, 
$\ell_{\mathrm{ht}}(q_i) - \ell_{\mathrm{ht}}(q_i^*) = 0$, confirming \eqref{eq:suff-A14}.

\textit{Verification of \eqref{eq:suff-A13}}: We examine the three subsets in turn.

(i) For $i \in \mathbb{H}_4^*$, we have $q_i^* \in [2/3, 1)$ and 
$\psi_i^* = -18C(1-q_i^*)/5 = -C\nabla\ell_{\mathrm{ht}}(q_i^*)$, since 
$\ell_{\mathrm{ht}}(s) = (-9s^2+18s-4)/5$ yields $\nabla\ell_{\mathrm{ht}}(s) = 18(1-s)/5$ on $(2/3, 1)$. 
There exists $\varrho_4 > 0$ ensuring $q_i \in (2/3, 1)$ for 
$\Sigma \in \Pi \cap U(\Sigma^*, \varrho_4)$. The proximal characterization 
$q_i^* = \mathrm{prox}_{\nu C\ell_{\mathrm{ht}}}(p_i^*)$ gives
\[
\ell_{\mathrm{ht}}(q_i) + \frac{(q_i - p_i^*)^2}{2\nu C} \ge 
\ell_{\mathrm{ht}}(q_i^*) + \frac{(q_i^* - p_i^*)^2}{2\nu C}.
\]
Expanding $(q_i^* - p_i^*)^2 - (q_i - p_i^*)^2 = 2(q_i^* - q_i)(p_i^* - q_i^*) + O(\|q_i - q_i^*\|^2) 
= 2\nu\psi_i^*(q_i - q_i^*) + O(\|q_i - q_i^*\|^2)$ and rearranging, we arrive at 
$C[\ell_{\mathrm{ht}}(q_i) - \ell_{\mathrm{ht}}(q_i^*)] + \psi_i^*(q_i - q_i^*) \ge 0$.

(ii) For $i \in \mathbb{H}_3^*$, we have $q_i^* \in (0, 2/3)$ and $\psi_i^* = -6C/5$. 
On $(0, 2/3]$, $\ell_{\mathrm{ht}}(s) = 6s/5$ is linear with derivative $6/5$, so 
$\psi_i^* = -C\nabla\ell_{\mathrm{ht}}(q_i^*)$. For $\Sigma$ near $\Sigma^*$, 
$q_i$ remains in $(0, 2/3)$, and linearity yields
$C[\ell_{\mathrm{ht}}(q_i) - \ell_{\mathrm{ht}}(q_i^*)] + \psi_i^*(q_i - q_i^*) = 
\frac{6C}{5}(q_i - q_i^*) - \frac{6C}{5}(q_i - q_i^*) = 0$.

(iii) For $i \in \mathbb{H}_2^*$, we have $q_i^* = 0$ and $\psi_i^* \in [-6C/5, 0]$. 
At $s = 0$, the subdifferential is $\partial\ell_{\mathrm{ht}}(0) = [0, 6/5]$ 
(noting $\ell_{\mathrm{ht}}(0) = 0$ and $\ell_{\mathrm{ht}}'(0^+) = 6/5$). 
Since $-\psi_i^*/C \in [0, 6/5]$, the subdifferential inequality 
$\ell_{\mathrm{ht}}(q_i) \ge (-\psi_i^*/C) q_i$ implies 
$C\ell_{\mathrm{ht}}(q_i) + \psi_i^* q_i \ge 0$. With $q_i^* = 0$ and $\ell_{\mathrm{ht}}(0) = 0$, 
this becomes $C[\ell_{\mathrm{ht}}(q_i) - \ell_{\mathrm{ht}}(q_i^*)] + \psi_i^*(q_i - q_i^*) \ge 0$.

Aggregating (i)--(iii), all terms are non-negative, confirming \eqref{eq:suff-A13}. 
This finishes Case (i).

\textbf{Case (ii):} Suppose $\nu C \in [5/18, 25/18)$. Define $\mathbf{p}^* := \mathbf{q}^* - \nu\boldsymbol{\psi}^*$, 
$\widetilde{\beta} := 6\nu C/5$, and $\widehat{\beta} := 5/6 + 3\nu C/5$. From \eqref{eq:prox-case2}, 
the proximal condition partitions $[m]$ into $\mathbb{C}_1^*$ ($q_i^* < 0$, $\psi_i^* = 0$), 
$\mathbb{C}_2^*$ ($q_i^* = 0$, $\psi_i^* \in [-6C/5, 0]$), 
$\mathbb{C}_3^*$ ($q_i^* \in (0, 2/3]$, $\psi_i^* = -6C/5$), and 
$\mathbb{C}_4^*$ ($q_i^* \ge 1$, $\psi_i^* = 0$).

Let $\overline{\mathbb{C}}^* := \mathbb{C}_1^* \cup \mathbb{C}_4^*$ and $\mathbb{C}^* := \mathbb{C}_2^* \cup \mathbb{C}_3^*$.

\textit{For $\overline{\mathbb{C}}^*$}: When $i \in \mathbb{C}_1^*$, $q_i^* < 0$ implies 
$\ell_{\mathrm{ht}}(q_i) = \ell_{\mathrm{ht}}(q_i^*) = 0$ locally. When $i \in \mathbb{C}_4^*$, 
$q_i^* \ge 1$ implies $\ell_{\mathrm{ht}}(q_i) = \ell_{\mathrm{ht}}(q_i^*) = 1$ locally. 
Either way, $\ell_{\mathrm{ht}}(q_i) - \ell_{\mathrm{ht}}(q_i^*) = 0$.

\textit{For $\mathbb{C}^*$}: (i) For $i \in \mathbb{C}_3^*$, linearity of $\ell_{\mathrm{ht}}(s) = 6s/5$ on $(0, 2/3]$ 
with $\psi_i^* = -C\nabla\ell_{\mathrm{ht}}(q_i^*)$ yields 
$C[\ell_{\mathrm{ht}}(q_i) - \ell_{\mathrm{ht}}(q_i^*)] + \psi_i^*(q_i - q_i^*) = 0$.
(ii) For $i \in \mathbb{C}_2^*$, using $-\psi_i^*/C \in \partial\ell_{\mathrm{ht}}(0)$ and $q_i^* = 0$, 
we get $C[\ell_{\mathrm{ht}}(q_i) - \ell_{\mathrm{ht}}(q_i^*)] + \psi_i^*(q_i - q_i^*) \ge 0$.

Both contributions are non-negative, completing Case (ii).

\textbf{Case (iii):} Suppose $\nu C \ge 25/18$. Define $\mathbf{p}^* := \mathbf{q}^* - \nu\boldsymbol{\psi}^*$ 
and $\gamma := \sqrt{2\nu C}$. By \eqref{eq:prox-case3}, the index set splits into 
$\mathbb{E}_1^*$ ($q_i^* < 0$, $\psi_i^* = 0$), 
$\mathbb{E}_2^*$ ($q_i^* = 0$, $\psi_i^* \in [-6C/5, 0]$), and 
$\mathbb{E}_3^*$ ($q_i^* \ge \gamma$, $\psi_i^* = 0$).

Let $\overline{\mathbb{E}}^* := \mathbb{E}_1^* \cup \mathbb{E}_3^*$ and $\underline{\mathbb{E}}^* := \mathbb{E}_2^*$.

\textit{For $\overline{\mathbb{E}}^*$}: When $i \in \mathbb{E}_1^*$, $q_i^* < 0$ yields 
$\ell_{\mathrm{ht}}(q_i) = \ell_{\mathrm{ht}}(q_i^*) = 0$ locally. When $i \in \mathbb{E}_3^*$, 
$q_i^* \ge \gamma \ge 5/3 > 1$ yields $\ell_{\mathrm{ht}}(q_i) = \ell_{\mathrm{ht}}(q_i^*) = 1$ locally. 
Thus $\ell_{\mathrm{ht}}(q_i) - \ell_{\mathrm{ht}}(q_i^*) = 0$.

\textit{For $\underline{\mathbb{E}}^*$}: With $q_i^* = 0$, $\psi_i^* \in [-6C/5, 0]$, and 
$-\psi_i^*/C \in \partial\ell_{\mathrm{ht}}(0)$, the subdifferential property gives 
$C[\ell_{\mathrm{ht}}(q_i) - \ell_{\mathrm{ht}}(q_i^*)] + \psi_i^*(q_i - q_i^*) \ge 0$.

Combining all cases, $(\mathbf{w}^*; b^*; \mathbf{q}^*)$ is a local minimizer of 
\eqref{eq:ht-svm-constrained} in $\Pi \cap U(\Sigma^*, \eta^*)$ for sufficiently small 
$\eta^* > 0$. \hfill$\square$

\section{Proof of Theorem~\ref{thm:sv-case1}}
\label{app:sv-case1}

\noindent\textbf{Proof.} Suppose $(\mathbf{w}^*; b^*; \mathbf{q}^*)$ with 
$\boldsymbol{\psi}^*$ is a P-stationary point of \eqref{eq:ht-svm-constrained}. 
From Definition~\ref{def:p-stationary}, the first equation 
$\mathbf{w}^* + N^\top \boldsymbol{\psi}^* = \mathbf{0}$ gives
\begin{equation}
\mathbf{w}^* = -N^\top \boldsymbol{\psi}^* = -\sum_{i=1}^m \psi_i^* y_i \mathbf{x}_i.
\label{eq:sv-w}
\end{equation}
Set $\mathbf{p}^* := \mathbf{q}^* - \nu\boldsymbol{\psi}^* \in \mathbb{R}^m$. 
For $\nu C \in (0, 5/18)$, invoking Lemma~\ref{lem:prox-1d} and the proximal 
inclusion $\mathbf{q}^* \in \mathrm{prox}_{\nu C L_{\mathrm{ht}}}(\mathbf{p}^*)$, 
we partition the index set $[m]$ into five disjoint subsets 
$\mathbb{G}_1^*, \mathbb{G}_2^*, \mathbb{G}_3^*, \mathbb{G}_4^*, \mathbb{G}_5^*$ 
as defined in Theorem~\ref{thm:sv-case1}. By \eqref{eq:prox-case1}, the proximal 
mapping yields:
\[
q_i^* = \begin{cases}
p_i^*, & i \in \mathbb{G}_1^*,\\
0, & i \in \mathbb{G}_2^*,\\
p_i^* - 6\nu C/5, & i \in \mathbb{G}_3^*,\\
(5p_i^* - 18\nu C)/(5 - 18\nu C), & i \in \mathbb{G}_4^*,\\
p_i^*, & i \in \mathbb{G}_5^*.
\end{cases}
\]
Substituting $p_i^* = q_i^* - \nu\psi_i^*$ into the above expressions, we obtain:

(i) For $i \in \mathbb{G}_1^*$, we have $q_i^* = p_i^* = q_i^* - \nu\psi_i^*$, 
which implies $\psi_i^* = 0$.

(ii) For $i \in \mathbb{G}_2^*$, from \eqref{eq:psi-cases}, when $q_i^* = 0$, 
we have $\psi_i^* \in [-6C/5, 0]$, thus $\psi_i^* \ne 0$ in general.

(iii) For $i \in \mathbb{G}_3^*$, we have $q_i^* = p_i^* - 6\nu C/5 = q_i^* - \nu\psi_i^* - 6\nu C/5$, 
which gives $\psi_i^* = -6C/5 \ne 0$.

(iv) For $i \in \mathbb{G}_4^*$, solving $(5 - 18\nu C)q_i^* = 5p_i^* - 18\nu C$ 
with $p_i^* = q_i^* - \nu\psi_i^*$ yields $\psi_i^* = -18C(1 - q_i^*)/5$. 
Since $q_i^* \in (2/3, 1)$ for $i \in \mathbb{G}_4^*$, we have $\psi_i^* \ne 0$.

(v) For $i \in \mathbb{G}_5^*$, we have $q_i^* = p_i^* = q_i^* - \nu\psi_i^*$, 
which implies $\psi_i^* = 0$.

Summarizing the above analysis, we have $\psi_i^* = 0$ for 
$i \in \mathbb{G}_1^* \cup \mathbb{G}_5^* = \overline{\mathbb{G}}^*$ and 
$\psi_i^* \ne 0$ for $i \in \mathbb{G}_2^* \cup \mathbb{G}_3^* \cup \mathbb{G}_4^* = \mathbb{G}^*$. 
Substituting into \eqref{eq:sv-w} yields \eqref{eq:sv-case1}.

For the support vector conditions \eqref{eq:sv-cond1}, from the feasibility 
constraint $\mathbf{q}^* + N\mathbf{w}^* + b^*\mathbf{y} = \mathbf{1}$, we have 
$q_i^* = 1 - y_i(\langle \mathbf{w}^*, \mathbf{x}_i \rangle + b^*)$. Thus:

(i) For $i \in \mathbb{G}_2^*$, we have $q_i^* = 0$, hence 
$y_i(\langle \mathbf{w}^*, \mathbf{x}_i \rangle + b^*) = 1$.

(ii) For $i \in \mathbb{G}_3^*$, we have $q_i^* \in (0, 2/3)$, hence 
$y_i(\langle \mathbf{w}^*, \mathbf{x}_i \rangle + b^*) = 1 - q_i^* \in (1/3, 1)$.

(iii) For $i \in \mathbb{G}_4^*$, we have $q_i^* \in [2/3, 1)$, hence 
$y_i(\langle \mathbf{w}^*, \mathbf{x}_i \rangle + b^*) = 1 - q_i^* \in (0, 1/3]$.

This completes the proof. \hfill$\square$

\section{Proof of Theorem~\ref{thm:convergence}}
\label{app:convergence}

\noindent\textbf{Proof.} Since the working set $F_k \subseteq [m]$ admits only 
finitely many distinct configurations, there exists a subset 
$\Psi \subseteq \{1, 2, 3, \ldots\}$ such that $F_i \equiv F$ for all $i \in \Psi$.
For convenience, define the aggregate iterate 
$\Phi^k := (\mathbf{w}^k, b^k, \mathbf{q}^k, \boldsymbol{\psi}^k)$ and 
$\Phi^* := (\mathbf{w}^*, b^*, \mathbf{q}^*, \boldsymbol{\psi}^*)$. By hypothesis, 
$\{\Phi^k\} \to \Phi^*$, which implies $\{\Phi^i\}_{i \in \Psi} \to \Phi^*$ 
and $\{\Phi^{i+1}\}_{i \in \Psi} \to \Phi^*$. Taking the limit with respect 
to $\Psi$ in \eqref{eq:psi-update}, i.e., $k \in \Psi$, $k \to \infty$, we get
\begin{equation}
\boldsymbol{\psi}_F^* = \boldsymbol{\psi}_F^* + \tau\xi\boldsymbol{\Lambda}_F^*, 
\quad \boldsymbol{\psi}_{\bar{F}}^* = \mathbf{0},
\label{eq:psi-limit}
\end{equation}
where $\boldsymbol{\Lambda}^{k+1} := \mathbf{q}^{k+1} - \mathbf{1} + N\mathbf{w}^{k+1} 
+ b^{k+1}\mathbf{y}$ denotes the primal residual. From \eqref{eq:psi-limit}, 
we immediately obtain $\boldsymbol{\Lambda}_F^* = \mathbf{0}$.
Taking the limit of $\boldsymbol{\Lambda}^k$ yields
\begin{equation}
\boldsymbol{\Lambda}^* = \mathbf{q}^* - \mathbf{1} + N\mathbf{w}^* + b^*\mathbf{y}.
\label{eq:Lambda-limit}
\end{equation}
Since $\boldsymbol{\Lambda}_F^* = \mathbf{0}$, it remains to show 
$\boldsymbol{\Lambda}_{\bar{F}}^* = \mathbf{0}$. For $i \in \bar{F}$, the working set 
construction ensures $p_i^* \le 0$ or $p_i^* \ge 1$ (in the sense of the respective 
threshold), so $\mathrm{prox}_{\nu C\ell_{\mathrm{ht}}}(p_i^*) = p_i^*$. Combining 
this with $\psi_i^* = 0$ yields $q_i^* = p_i^* = 1 - (N\mathbf{w}^*)_i - b^*y_i$, 
i.e., $\Lambda_i^* = 0$. Therefore $\boldsymbol{\Lambda}^* = \mathbf{0}$, 
establishing the primal feasibility condition
\begin{equation}
\mathbf{q}^* + N\mathbf{w}^* + b^*\mathbf{y} = \mathbf{1}.
\label{eq:feasibility-limit}
\end{equation}
Regarding \eqref{eq:q-update}, we take the limit with respect to $\Psi$ and obtain
\begin{equation}
\mathbf{q}^* = \mathrm{prox}_{\nu C L_{\mathrm{ht}}}(\mathbf{p}^*),
\label{eq:q-limit}
\end{equation}
where $\mathbf{p}^* := \mathbf{1} - N\mathbf{w}^* - b^*\mathbf{y} - \nu\boldsymbol{\psi}^*$. 
Substituting the feasibility condition \eqref{eq:feasibility-limit}, we get
\begin{equation}
\mathbf{p}^* = \mathbf{q}^* - \nu\boldsymbol{\psi}^*.
\label{eq:p-substitute}
\end{equation}
Combining \eqref{eq:q-limit} and \eqref{eq:p-substitute} yields
\begin{equation}
\mathbf{q}^* \in \mathrm{Prox}_{\nu C L_{\mathrm{ht}}}(\mathbf{q}^* - \nu\boldsymbol{\psi}^*).
\label{eq:q-prox-final}
\end{equation}
With respect to \eqref{eq:w-linear-system}, taking the limit gives
$(I + \xi N_F^\top N_F)\mathbf{w}^* = \xi N_F^\top \boldsymbol{\chi}_F^*$,
where $\boldsymbol{\chi}^* := \mathbf{1} - \mathbf{q}^* - b^*\mathbf{y} 
- \boldsymbol{\psi}^*/\xi$. By \eqref{eq:feasibility-limit}, we have 
$\boldsymbol{\chi}_F^* = N_F\mathbf{w}^* - \boldsymbol{\psi}_F^*/\xi$. 
Substituting into the above equation and simplifying yields
\[
\mathbf{w}^* = -N_F^\top\boldsymbol{\psi}_F^* \stackrel{\eqref{eq:psi-limit}}{=} 
-N^\top\boldsymbol{\psi}^*.
\]
Finally, for \eqref{eq:b-update}, taking the limit we obtain
$b^* = \langle \mathbf{y}, \mathbf{r}^* \rangle / m$,
where $\mathbf{r}^* := \mathbf{1} - N\mathbf{w}^* - \mathbf{q}^* 
- \boldsymbol{\psi}^*/\xi$. Using \eqref{eq:feasibility-limit}, we get 
$\mathbf{r}^* = b^*\mathbf{y} - \boldsymbol{\psi}^*/\xi$, so
\[
b^* = \langle \mathbf{y}, b^*\mathbf{y} - \boldsymbol{\psi}^*/\xi \rangle / m 
= b^* - \langle \mathbf{y}, \boldsymbol{\psi}^* \rangle / (m\xi),
\]
which leads to $\langle \mathbf{y}, \boldsymbol{\psi}^* \rangle = 0$.
Summarizing the above, the limit point $(\mathbf{w}^*, b^*, \mathbf{q}^*, \boldsymbol{\psi}^*)$ 
satisfies: (i) $\mathbf{w}^* + N^\top\boldsymbol{\psi}^* = \mathbf{0}$; 
(ii) $\langle \mathbf{y}, \boldsymbol{\psi}^* \rangle = 0$; 
(iii) $\mathbf{q}^* + N\mathbf{w}^* + b^*\mathbf{y} = \mathbf{1}$; 
(iv) $\mathbf{q}^* \in \mathrm{Prox}_{\nu C L_{\mathrm{ht}}}(\mathbf{q}^* - \nu\boldsymbol{\psi}^*)$. 
These are precisely the conditions in Definition~\ref{def:p-stationary}, confirming 
that $(\mathbf{w}^*, b^*, \mathbf{q}^*)$ is a P-stationary point of 
\eqref{eq:ht-svm-constrained}. By Theorem~\ref{thm:sufficient}, it is also a 
local minimizer. \hfill$\square$





\bibliographystyle{elsarticle-harv} 
\bibliography{refs}

\end{document}